\documentclass[final,3p,times]{elsarticle}
\journal{Knowledge-Based Systems}
\usepackage{lineno}
\modulolinenumbers[5]
\usepackage{amsmath}        
\usepackage{amssymb}        
\usepackage{array}
\usepackage{dsfont}         
\usepackage{multirow}       
\usepackage{epsfig}         
\usepackage{graphicx}       
\usepackage{subfigure}      
\usepackage{balance}        
\usepackage{algorithm}      
\usepackage{url}            
\usepackage{color}          
\usepackage{booktabs}       
\usepackage[justification=centering]{caption} 
\usepackage{algpseudocode}  
\usepackage{hyperref}       
\usepackage{url}            
\usepackage{ntheorem}
\usepackage{placeins}
\usepackage{stfloats}
\usepackage{float}
\usepackage{array}
\biboptions{sort&compress} 

\newcommand{\nop}[1]{}

\makeatletter
\newcommand{\thickhline}{%
    \noalign {\ifnum 0=`}\fi \hrule height 0.6pt
    \futurelet \reserved@a \@xhline
}
\makeatother

\hypersetup{
	colorlinks=true,
	linkcolor=cyan,
	filecolor=blue,
	urlcolor=red,
	citecolor=green,
}

\newtheorem{definition}{Definition}
\newtheorem*{proof}{Proof}
\newtheorem{lemma}{Lemma}
\usepackage[misc]{ifsym}
\begin{document}

\title{Unsupervised feature selection via self-paced learning and low-redundant regularization}


\begin{frontmatter}

\author[mymainaddress,mysecondaryaddress]{Weiyi Li}
\ead{weiyili@my.swjtu.edu.cn}

\author[mymainaddress,mysecondaryaddress]{Hongmei Chen\corref{mycorrespondingauthor}}
\ead{hmchen@swjtu.edu.cn}
\cortext[mycorrespondingauthor]{Corresponding author}

\author[mymainaddress,mysecondaryaddress]{Tianrui~Li}
\ead{trli@swjtu.edu.cn}

\author[mymainaddress,mysecondaryaddress]{Jihong Wan}
\ead{jhwan@my.swjtu.edu.cn}

\author[mymainaddress,mysecondaryaddress]{Binbin Sang}
\ead{sangbinbin@my.swjtu.edu.cn}


\address[mymainaddress]{School of Computing and Artificial Intelligence, Southwest Jiaotong University, China}
\address[mysecondaryaddress]{National Engineering Laboratory of Integrated Transportation Big Data Application Technology, Southwest Jiaotong University, China}
\begin{abstract}
Much more attention has been paid to unsupervised feature selection nowadays due to the emergence of massive unlabeled data.
The distribution of samples and the latent effect of training a learning method using samples in more effective order need to be considered so as to improve the robustness of the method. Self-paced learning is an effective method considering the training order of samples. In this study, an unsupervised feature selection is proposed by integrating the framework of self-paced learning and subspace learning. Moreover, the local manifold structure is preserved and the redundancy of features is constrained by two regularization terms. $L_{2,1/2}$-norm is applied to the projection matrix,  which aims to retain discriminative features and further alleviate the effect of noise in the data. Then, an iterative method is presented to solve the optimization problem. The convergence of the method is proved theoretically and experimentally. The proposed method is compared with other state of the art algorithms on nine real-world datasets. The experimental results show that the proposed method can improve the performance of clustering methods and outperform other compared algorithms.
\end{abstract}

\end{frontmatter}

\section{Introduction}
\label{sec:Introduction}  

As information technology develops, high-dimensional data can be obtained easily. Although data with high dimension is able to provide abundant useful information, there are redundant and noisy features which may result in poor performance of the corresponding algorithm \cite{art39,art40}. Feature selection, which aims to select a subset from the original feature set, is a commonly used strategy to reduce the dimension to an extent that only discriminative features are retained \cite{art57}. The selected informative features can keep the semantic information of raw data, which is helpful in the subsequent data analysis. However, due to the high cost and difficulty of acquiring labels, significance has been attached to unsupervised feature selection methods, where features are selected considering the intrinsic properties and structure of the high-dimensional data \cite{art58,art59}. It has been applied in many fields, such as machine learning, pattern recognition, and text mining \cite{art41,art42,art43,art44}.

Unsupervised methods can be bracketed into three kinds by and large, namely, filter, wrapper and embedded methods. Filter methods rank features according to the certain characteristic of data, $e.g.$, information related to labels and topological structure. Laplacian Score, max variance and trace ratio are three traditional approaches of this kind \cite{art12,art45,art46}. Wrapper methods use a learning model to evaluate a subset $e.g.$ methods varFnMS \cite{art47} and UFSACO \cite{art48}. Although the performance of wrapper methods is better than that of filter methods, the model needs to be trained repeatedly, which may result in high computational cost. Therefore, they are unsuitable for large-scale datasets. To address this problem, embedded methods integrate feature selection with model optimization \cite{art60}. Different from the aforementioned two methods, they carry out feature selection automatically in the training process of the learner. Compared with wrapper methods, embedded methods avoid repeated training for the learner to evaluate every feature subset. Hence, the optimal one can be gained rapidly, which indicates the high efficiency of the methods.

In general, there are two kinds of models utilized in embedded methods, that is to say regression-based model and self-representation based model. Regression-based model converts unsupervised feature selection into a supervised one by learning the pseudo labels of data. M. G. Parsa et al. proposed unsupervised feature selection based on adaptive similarity learning and subspace clustering (SCFS) \cite{art49}. Symmetric nonnegative matrix factorization is employed to get the cluster indicator matrix and a regression model is exploited to optimize the coefficient matrix so as to select the most important features. Given the fact that large amounts of data are associated with several clusters instead of a single cluster in real-world applications, selecting features under the guidance of hard labels is very likely to degrade the effectiveness of the algorithms. As a consequence, Wang et al. proposed unsupervised soft-label feature selection (USFS) which combines soft-label learning with the framework of regression \cite{art50}. As for the self-representation based model, it is assumed that each feature can be represented by the linear combination of its relevant features. Lu et al. proposed structure preserving unsupervised feature selection which also maintains the local manifold structure of data \cite{art51}. To further suppose each sample can be reconstructed by the linear combination of its relevant samples and take the advantages of learning the similarity matrix adaptively into account, Tang et al. proposed robust unsupervised feature selection via dual self-representation and manifold regularization (DSRMR) \cite{art52}. As is known to all, negative elements play an insignificant role in practical problems \cite{art53}. Based on self-representation, subspace learning is introduced, which embeds the potential characteristics of the data into the low-dimensional space through a projection matrix \cite{art18}. Through subspace learning, $l$ features are chosen first to delete irrelevant features. Then the original high-dimensional data is reconstructed from the representative features by means of the coefficient matrix, which can prevent the influence of noisy features as much as possible. Wang et al. proposed a subspace learning algorithm for unsupervised feature selection via matrix factorization (MFFS) \cite{art18}. The algorithm imposes subspace learning to select a feature subset that is capable of representing the remaining features. Nevertheless, it doesn’t take the sparsity of the indicator matrix into account. To overcome this problem, Zheng et al. proposed a robust unsupervised feature selection, $i.e.$, nonnegative sparse subspace learning (NSSLFS) \cite{art19}. NSSLFS adds the ${l_1}$-loss and ${l_{2,1}}$-norm minimization to the objective function. Therefore, the sparsity and robustness are achieved.

Moreover, it is widely accepted that redundancy and noises in features and data tend to degrade the performance of a learning method. So, some low-redundant methods have emerged in recent years. Liu et al. proposed a diversity induced self-representation for unsupervised feature selection algorithm (DISR) \cite{art20}. The algorithm takes the diversity of features into consideration to reduce the redundancy. For the reason that DISR gives little care to the manifold structure, Shang et al. proposed sparse and low-redundant subspace learning-based graph regularized feature selection (SLDSR) \cite{art21}. SLDSR focuses on the local geometric structure of feature and data space to promote its performance. Unlike DISR, a novel diversity term is introduced, which utilizes the inner product of feature weight vectors. On the basis of the framework of regularized regression, Lim et al. proposed feature dependency-based unsupervised feature selection (DUFS) \cite{art22}. Mutual information is employed in the algorithm to evaluate the dependency among features.

Usually, ${l_{2,p}}$-norm is imposed on the matrix to enhance the robustness and avoid the problem caused by noises. Zhu et al. proposed co-regularized unsupervised feature selection (CUFS) \cite{CUFS}. The algorithm takes data reconstruction and cluster structure into account simultaneously. By applying ${l_{2,1}}$-norm to the projection matrix and cluster base matrix, sparsity can be guaranteed. Liu et al. proposed robust neighborhood embedding for unsupervised feature selection (RNE) \cite{art31}. Considering that the commonly used ${l_{2,1}}$-norm requires to assume the coefficient distribution, RNE replaces it with ${l_1}$-norm to achieve robust result. Miao et al. proposed unsupervised feature selection by non-convex regularized self-representation (NOVRSR) \cite{NOVRSR}. ${l_{2,1 - 2}}$-norm which is non-convex but Lipschitz continuous is imposed on the representation coefficient matrix. And an efficient iterative algorithm is designed to address its non-convexity. However, the above three algorithms just solve the relaxed problem from the original ${l_{2,0}}$-norm problem, which has a tendency to weaken the performance. Nie et al. proposed unsupervised feature selection with constrained ${l_{2,0}}$-norm and optimized graph (RSOGFS) \cite{RSOGFS}. RSOGFS tackles the ${l_{2,0}}$-norm problem directly so as to choose the needed features at a time instead of one by one. Thus, the optimal combination of features can be obtained. Since the optimization algorithms of ${l_1}$-norm, ${l_{2,1 - 2}}$-norm and ${l_{2,0}}$-norm are much more complex than ${l_{2,1}}$-norm and ${l_{2,{1 \mathord{\left/
 {\vphantom {1 2}} \right.
 \kern-\nulldelimiterspace} 2}}}$-norm has been testified to have more robust and sparser results than ${l_{2,1}}$-norm, utilizing ${l_{2,{1 \mathord{\left/
 {\vphantom {1 2}} \right.
 \kern-\nulldelimiterspace} 2}}}$-norm to ensure the row-sparsity of the matrix may be a better choice \cite{art27}.

Under the assumption that easy samples with smaller loss ought to be selected in the early stage while complex samples with larger loss are supposed to be selected later or not, self-paced learning is raised \cite{art25}. In the iterative process, there are a growing number of complex samples to be involved in the model until the model is “mature”. As a result, it is possible for the relatively complex samples to be excluded from the model or be included with smaller weights, which is another way to strengthen its robustness \cite{art23}. At present, self-paced learning is rarely deployed in unsupervised feature selection. Zheng et al. proposed unsupervised feature selection by self-paced learning regularization which integrates self-paced learning with the framework of self-representation to achieve promising performance \cite{art56}. However, the local geometric structure and the diversity of features which take a vital part in feature selection process are overlooked. Consequently, the effectiveness of the algorithm can be further improved.

From what has been mentioned above, we present unsupervised feature selection via self-paced learning and low-redundant regularization (SPLR). It is devised to settle the following two shortcomings of the existent unsupervised feature selection algorithms: 1) Noisy and redundant features which are likely to depress the performance are not removed during the training process. 2) The distribution of samples and the latent effect of training a learning method using samples in more effective order are rarely taken into consideration. To be specific, self-paced learning and subspace learning are united to not only reduce the influence of noises but also make the reconstruction information more accurate. Through self-paced learning, samples with smaller loss are given larger weights initially and vice versa. In the process of learning, an increasing number of samples are involved until the model is robust \cite{art23,art24,art25,art26}. Through subspace learning, data is first embedded from the original high-dimensional space to the relatively low-dimensional subspace, and then reconstructed to the space with high dimension. Additionally, on account of the fact that local geometric structure with regard to data plays a much more crucial role than the global structure, a local structure preserving term is introduced to the objective function. What’s more, given the fact that features which are closely correlated stand a chance of having a negative effect on the performance of the algorithm, a diversity term is brought up. The primary contributions of the proposed SPLR are as follows.
\begin{itemize}
  \item Diversity from the perspective of both features and data is considered. In other words, a regularization term is leveraged to select low-redundant features. At the same time, self-paced learning is intended for the exclusion of outliers.
  \item Global reconstruction information of data is preserved by subspace learning. In the meantime, local manifold structure with regard to data is retained, which is of great significance in the feature selection process.
  \item ${l_{2,{1 \mathord{\left/
 {\vphantom {1 2}} \right.
 \kern-\nulldelimiterspace} 2}}}$-norm is imposed to constrain the projection matrix, which has been confirmed to achieve sparser result. In addition, to reach the goal of minimizing the corresponding objective function, an iterative algorithm is raised. And convergence is analyzed. Experimental results verify the effectiveness of SPLR.
\end{itemize}

The rest of the paper is organized as follows. Section~\ref{sec:RelatedWork} illustrates some research relevant to the proposed algorithm. Section~\ref{sec:Method} introduces the proposed algorithm in detail. Optimization algorithm and convergence analysis are included as well. Section~\ref{sec:Experiments} demonstrates the experimental results of the proposed algorithm compared with other state-of-the-art algorithms. Finally, conclusions are drawn in Section~\ref{sec:Conclusions}.

\section{Related works}
\label{sec:RelatedWork}  

In this section, we first introduce the notations used in this paper. Then we give a brief introduction of the framework of subspace learning and self-paced learning, which is closely correlated with the proposed method.

\subsection{Notations}
In this paper, matrices and vectors are denoted as bold uppercase letters and bold lowercase letters, respectively. For an arbitrary matrix $A \in {\mathbb{R}^{m \times n}}$, ${a_i}$ represents the $i$th row of the matrix, and ${a^j}$ represents the $j$th column of the matrix. ${A_{ij}}$ denotes the element located in the $i$th row and $j$th column of the matrix $A$. ${A^{\rm T}}$ means the transpose of the matrix $A$. ${\rm{tr}}\left( A \right)$ and ${A^{ - 1}}$ refer to the trace and inverse of $A$ under the condition that $A$ is square. To avoid ambiguity, more details are shown in Table \ref{tab::table1}.

\newcommand{\tabincell}[2]{\begin{tabular}{@{}#1@{}}#2\end{tabular}}
\begin{table}[htbp]
\begin{center}
\caption{Notations used in the paper}\label{tab::table1}
\begin{tabular}[c]{ll}
\toprule
Notation& Description\\
\midrule
$X \in {\mathbb{R}^{n \times d}}$ & Original data matrix\\
$W \in {\mathbb{R}^{d \times K}}$ & Projection matrix\\
$H \in {\mathbb{R}^{K \times d}}$ & Reconstruction matrix\\
$S \in {\mathbb{R}^{n \times n}}$ & Similarity matrix of feature space\\
$Z \in {\mathbb{R}^{d \times d}}$ & Similarity matrix of data space\\
$L \in {\mathbb{R}^{n \times n}}$ & Laplacian matrix of data space\\
${x_i} \in {\mathbb{R}^{1 \times d}}$ & The $i$th row of $X$\\
${x^j} \in {\mathbb{R}^{n \times 1}}$ & The $j$th column of $X$\\
$v \in {\mathbb{R}^{n \times 1}}$ & The weight vector\\
${\rm{tr}}\left( X \right)$ & The trace of $X$\\
${X^{\rm T}}$ & The transpose of $X$\\
${X^{ - 1}}$ & The inverse of square matrix $X$\\
${\left\| X \right\|_F}$ & \tabincell{l}{The Frobenius norm of $X$, $i.e.$, \\${\left\| X \right\|_F} = \sqrt {\sum\limits_{i = 1}^n {\sum\limits_{j = 1}^d {x_{ij}^2} } } $}\\
${\left\| X \right\|_{p,q}}$ & \tabincell{l}{The ${l_{p,q}}$-norm of $X$, $i.e.$, \\${\left\| X \right\|_{p,q}} = {\left( {\sum\limits_{i = 1}^n {{{\left( {\sum\limits_{j = 1}^d {{{\left| {{x_{ij}}} \right|}^p}} } \right)}^{\frac{q}{p}}}} } \right)^{\frac{1}{q}}}$}\\
\bottomrule
\end{tabular}
\end{center}
\end{table}

\subsection{The framework of subspace learning}
Subspace learning is to learn a subspace from the original space with minimal loss. Zhang et al. proposed algorithm MFFS to accomplish this goal \cite{art18}. Precisely speaking, by means of matrix factorization, subspace learning can be achieved by optimizing the following problem.
\begin{equation}
\label{eqn::eq1}
\begin{array}{l}
\arg \mathop {\min }\limits_{W,H} \frac{1}{2}\left\| {X - XWH} \right\|_F^2\\
s.t.\;\;\;\;\;\;\;W \ge 0,\;{W^{\rm T}}W = {E_K}
\end{array}
\end{equation}

where $X = {\left( {x_1^{\rm T},x_2^{\rm T}, \ldots ,x_n^{\rm T}} \right)^T} \in {\mathbb{R}^{n \times d}}$ denotes the original data matrix, in which ${x_i}$ denotes the $i$th data sample. $W \in {\mathbb{R}^{d \times K}}$ is a projection matrix, aiming to select the most discriminative features. $H \in {\mathbb{R}^{K \times d}}$ represents the reconstruction matrix, whose purpose is to reconstruct the high-dimensional space from the low-dimensional one. It should be emphasized that through the non-negative constraint and orthogonal constraint, $W$ is a binary matrix, in which each row and column of it have one non-zero element at most.

It is widely accepted that real-world data is usually nonnegative. For this reason, a constraint is added supposing that the positive linear combination of the selected features has the ability to reconstruct all features. Accordingly, the objective function can be rewritten as
\begin{equation}
\label{eqn::eq2}
\begin{array}{l}
\arg \mathop {\min }\limits_{W,H} \frac{1}{2}\left\| {X - XWH} \right\|_F^2\\
s.t.\;\;\;\;\;\;\;W \ge 0,\;H \ge 0,\;{W^{\rm T}}W = {E_K}
\end{array}
\end{equation}

As can be seen from Eq. \eqref{eqn::eq2}, samples are embedded to a $K$-dimensional subspace through $W$ in the first place. Then, by means of $H$, the low-dimensional samples are reconstructed to the original high-dimensional space. In such a manner, noisy features can be eliminated to some extent. It is worth noting that the dimension of the subspace $K$ is not necessarily identical with the number of selected features $N$ \cite{art38}. In general, we set $K \ge N$ for impressive performance.

\subsection{The framework of self-paced learning}
In 2009, Bengio et al. proposed curriculum learning \cite{curriculumLearning}. And based on this, M. Kumar et al. put forward self-paced learning (SPL) \cite{art25}.

Inspired by the human cognitive mechanism, data is gradually added to the training model from easy to hard in curriculum learning. And during the process, the entropy of the training set is also increasing, which means that more information is included. The essential problem of curriculum learning is the choice of the ranking function, from which each sample is assigned to a learning priority, and samples with higher priority stand a good chance of being selected earlier.

In most cases, the ranking function of a specific problem is determined by the prior knowledge intuitively, which is not elegant. Information that can be obtained during the training process is not taken full advantage of. Self-paced learning includes curriculum learning in a more uniform form. For the sake of considering the training order of samples and eliminating the influence of noises as much as possible, the target of SPL is to minimize the following objective function.

\begin{equation}
\label{eqn::eq3}
\mathop {\min }\limits_{w,v} E\left( {w,v;\lambda } \right) = \sum\limits_{i = 1}^n {\left( {{v_i}L\left( {{y_i},g\left( {{x_i},w} \right)} \right)} \right) + f\left( {{v_i},\lambda } \right)}
\end{equation}

where ${v_i}$ denotes the weight of the $i$th sample ${x_i}$. $f\left( {v,\lambda } \right)$ is a regularization term. $L\left( {{y_i},g\left( {{x_i},w} \right)} \right)$ is the loss function, which characterizes the residual error between the true label ${y_i}$ and the predicted one $g\left( {{x_i},w} \right)$. The smaller the loss is, the larger the value of ${v_i}$ is. That is to say, if the predicted label is similar to the ground truth label, ${v_i}$ will approach 1. And whether to single ${x_i}$ out for model training depends on the value of ${v_i}$. The closer ${v_i}$ is to 1, the more likely ${x_i}$ is to be selected. And the closer ${v_i}$ is to 0, the less likely ${x_i}$ is to be selected. In each iteration, $v$ is fixed when updating $w$ and the learned $w$ is fixed when updating $v$. It stops when the model is “mature”. Therefore, the most representative and informative samples are exploited.

As is illustrated above, a crucial task is to determine the self-paced regularizer. It can be defined as long as the following three conditions are satisfied.
\begin{enumerate}[i.]
\setlength{\itemindent}{1em}
\item $f\left( {v,\lambda } \right)$ is convex with respect to $v \in \left[ {0,1} \right]$;

\item ${v^ * }\left( {\lambda ;l} \right)$ is monotonically decreasing with respect to $l$, and it holds that $\mathop {\lim }\limits_{l \to 0} {v^ * }\left( {\lambda ;l} \right) = 1,\;\mathop {\lim }\limits_{l \to \infty } {v^ * }\left( {\lambda ;l} \right) = 0$;

\item ${v^ * }\left( {\lambda ;l} \right)$ is monotonically increasing with respect to $\lambda $, and it holds that $\mathop {\lim }\limits_{\lambda  \to \infty } {v^ * }\left( {\lambda ;l} \right) \le 1,\;\mathop {\lim }\limits_{\lambda  \to 0} {v^ * }\left( {\lambda ;l} \right) = 0$;
\end{enumerate}

where ${v^ * }\left( {\lambda ;l} \right) = \arg \mathop {\min }\limits_{v \in \left[ {0,1} \right]} vl + f\left( {v,\lambda } \right)$.

Listed below are three commonly used regularizers which are called the hard (${f^H}\left( {v;\lambda } \right)$), linear (${f^L}\left( {v;\lambda } \right)$) and mixture regularizer (${f^M}\left( {v;\lambda ,\gamma } \right)$) respectively and the corresponding solutions.
\begin{enumerate}[i.]
\setlength{\itemindent}{1em}
\item The hard regularizer
\begin{equation}
\label{eqn::eq4}
{f^H}\left( {v;\lambda } \right) =  - \lambda v,
\end{equation}
\begin{equation}
\label{eqn::eq136}
{v^*}\left( {\lambda ;l} \right) = \left\{ \begin{array}{l}
1,\;if\;l < \lambda ;\\
0,\;if\;l \ge \lambda.
\end{array} \right.
\end{equation}
\item The linear regularizer
\begin{equation}
{f^L}\left( {v;\lambda } \right) = \lambda \left( {\frac{1}{2}{v^2} - v} \right),
\end{equation}
\begin{equation}
{v^*}\left( {\lambda ;l} \right) = \left\{ \begin{array}{l}
 - {l \mathord{\left/
 {\vphantom {l \lambda }} \right.
 \kern-\nulldelimiterspace} \lambda } + 1,\;if\;l < \lambda ;\\
0,\;\;\;\;\;\;\;\;\;\;if\;l \ge \lambda.
\end{array}\right.
\end{equation}
\item The mixture regularizer
\begin{equation}
{f^M}\left( {v;\lambda ,\gamma } \right) = \frac{{{\gamma ^2}}}{{v + {\gamma  \mathord{\left/
 {\vphantom {\gamma  \lambda }} \right.
 \kern-\nulldelimiterspace} \lambda }}},\;
 \end{equation}
 \begin{equation}
 {v^*}\left( {\lambda ,\gamma ;l} \right) = \left\{ \begin{array}{l}
1,\;\;\;\;\;\;\;\;\;\;\;\;\;\;\;if\;l \le {\left( {\frac{{\lambda \gamma }}{{\lambda  + \gamma }}} \right)^2};\\
0,\;\;\;\;\;\;\;\;\;\;\;\;\;\;\;if\;l \ge {\lambda ^2};\\
\gamma \left( {\frac{1}{{\sqrt l }} - \frac{1}{\lambda }} \right),\;otherwise.
\end{array} \right.
\end{equation}
\end{enumerate}

\section{The proposed method}
\label{sec:Method}

In this section, a feataure selection method via self-paced learning and low-redundant regularization (SPLR) is developed. Owing to the non-convexity of the objective function when optimizing variables simultaneously, an iterative updating algorithm is exploited to solve the minimization problem. Eventually, we analyze the convergence of SPLR theoretically.

\subsection{Subspace learning with low redundancy}
Given the fact that redundancy between features plays an important role in feature selection, it is necessary to introduce a regularization term to eliminate its negative impact. For this reason, Liu et al. proposed a novel term taking the pairwise similarity of features into consideration \cite{art20}. In order to achieve low redundancy between features, we add the pairwise similarity as regularizer to the framework of subspace learning.

Dot product ${s_{ij}} = \frac{{{{\left( {{f^i}} \right)}^{\rm T}} \cdot {f^j}}}{{\left\| {{f^i}} \right\| \cdot \left\| {{f^j}} \right\|}},\;i,j = 1,2, \ldots ,d$ is applied to calculate the similarity between the $i$th feature and $j$th feature. Apparently, the larger the value is, the less diverse the two features are. Hence, the similarity matrix can be defined as
\begin{equation}
\label{eqn::eq5}
S = {X^{\rm T}}X
\end{equation}

Additionally, the row-sum of the projection matrix is utilized to measure the significance of each feature. As a result, the regularization term is denoted as
\begin{equation}
\label{eqn::eq6}
{\rm{tr}}\left( {{S^{\rm T}}W{\bf{1}}{W^{\rm T}}} \right) = \sum\limits_{i = 1}^d {\sum\limits_{j = 1}^d {\left( {{{\tilde w}_i}{{\tilde w}_j}} \right){s_{ij}}} }.
\end{equation}

where ${\tilde w_i} = \sum\nolimits_{j = 1}^d {{w_{ij}}} $, and ${\bf{1}} \in {\mathbb{R}^{K \times K}}$ is a matrix with all the elements equal to 1.

As is shown in Eq. \eqref{eqn::eq6}, if the $i$th feature and $j$th feature are very similar, then ${s_{ij}}$ is close to 1, so ${\tilde w_i}{\tilde w_j}$ ought to be small. Thus, if ${\tilde w_i}$ is large, ${\tilde w_j}$ should be small. And if ${\tilde w_j}$ is large, ${\tilde w_i}$ should be small. In other words, it is impossible for both ${\tilde w_i}$ and ${\tilde w_j}$ to be large. In this way, redundant features are less likely to be selected simultaneously.

Incorporating Eq. \eqref{eqn::eq6} into Eq. \eqref{eqn::eq2}, the framework of subspace learning with low redundancy is obtained as follows.
\begin{equation}
\label{eqn::eq7}
\begin{array}{l}
\arg \mathop {\min }\limits_{W,H} \frac{1}{2}\left\| {X - XWH} \right\|_F^2 + {\lambda _1}{\rm{tr}}\left( {{S^{\rm T}}W{\bf{1}}{W^{\rm T}}} \right)\\
s.t.\;\;\;\;\;\;\;W \ge 0,\;H \ge 0,\;{W^{\rm T}}W = {E_k}
\end{array}
\end{equation}

By optimizing the above problem, the projection matrix can be learned, and the learned $W$ in turn contributes to the learning of $H$. They constrain each other. And the adverse effect of both noisy features and redundant features is reduced.

\subsection{Local manifold structure preservation}
It is universally acknowledged that the local manifold structure is important for retaining the topological structure of data in feature selection. Based on the fact that similar samples in the original space are supposed to be similar when embedded into the subspace, the aforementioned target can be met by minimizing the following problem.
\begin{equation}
\label{eqn::eq8}
\arg \mathop {\min }\limits_W \frac{1}{2}\sum\limits_{i,j} {\left\| {{x_i}W - {x_j}W} \right\|_2^2{z_{ij}}}  = {\rm{tr}}\left( {{W^{\rm T}}{X^{\rm T}}LXW} \right)
\end{equation}

where ${z_{ij}} = \frac{{{x_i} \cdot x_j^{\rm T}}}{{\left\| {{x_i}} \right\| \cdot \left\| {{x_j}} \right\|}},\;i,j = 1,2, \ldots ,n$ evaluates the similarity between samples. $L = D - Z$ is the Laplacian matrix. And $D$ is a diagonal matrix with ${D_{ii}} = \sum\limits_{j = 1}^n {{z_{ij}}} $. With a view to minimizing the objective function, the distance between embedded samples ${x_i}W$ and ${x_j}W$, $i.e.$ $\left\| {{x_i}W - {x_j}W} \right\|_2^2$, should be small if samples in the high-dimensional space ${x_i}$ and ${x_j}$ are close, which represents a large value of similarity ${z_{ij}}$. And it is consistent with the local manifold structure preservation strategy.

\subsection{The framework of SPLR}
Combining Eq. \eqref{eqn::eq7} and Eq. \eqref{eqn::eq8} together, the expression is transformed as follows.
\begin{equation}
\label{eqn::eq9}
\begin{array}{l}
\arg \mathop {\min }\limits_{W,H} \frac{1}{2}\left\| {X - XWH} \right\|_F^2 + {\lambda _1}{\rm{tr}}\left( {{S^{\rm T}}W{\bf{1}}{W^{\rm T}}} \right)+ {\lambda _2}{\rm{tr}}\left( {{W^{\rm T}}{X^{\rm T}}LXW} \right)\\
s.t.\;\;\;\;\;\;\;W \ge 0,\;H \ge 0,\;{W^{\rm T}}W = {E_k}
\end{array}
\end{equation}

It can be learned from Eq. \eqref{eqn::eq4} and Eq. \eqref{eqn::eq136} that the hard regularizer is unable to distinguish between two samples with different importance for the reason that the weight is set to 1 as long as the loss is less than a given value $\lambda $. If the loss of sample ${x_i}$ is close to 0 and the loss of sample ${x_j}$ approaches $\lambda $, it is apparent that ${x_i}$ plays a much more crucial role than ${x_j}$ and ${x_i}$ should be regarded as an easier sample. But with the hard regularizer, they are given the same weight, resulting in performance deterioration. On the contrary, if the loss of sample ${x_p}$ is a little bit smaller than $\lambda $ and the loss of sample ${x_q}$ is a little bit larger than $\lambda $, they will have totally different weights, which produces less reasonable solution. As for the linear regularizer, the significance of samples can be discriminated but small errors can’t be tolerated. Samples with different loss values are assigned to different weights. However, if both the loss of sample ${x_i}$ and that of sample ${x_j}$ are small enough, it is sensible to set their weights to 1. Thus, the mixture regularizer is employed which can not only enjoy the advantages of the aforementioned regularizers, but also tolerate small errors up to a certain point.

A sparsity regularization term is also added so as to select representative and robust features. One commonly used term is ${l_{2,1}}$-norm. However, Wang et al. have certified the superiority of ${l_{2,{1 \mathord{\left/
 {\vphantom {1 2}} \right.
 \kern-\nulldelimiterspace} 2}}}$-norm \cite{art27}. Several experiments are conducted and the results demonstrate that the regularization term with ${l_{2,{1 \mathord{\left/
 {\vphantom {1 2}} \right.
 \kern-\nulldelimiterspace} 2}}}$-norm outperforms others in terms of classification error.

The contour maps of ${l_{2,{1 \mathord{\left/
 {\vphantom {1 2}} \right.
 \kern-\nulldelimiterspace} 2}}}$-norm, ${l_{2,1}}$-norm, and ${l_{2,2}}$-norm are presented in Fig. \ref{fig::picture8}. It serves to show that ${l_{2,{1 \mathord{\left/
 {\vphantom {1 2}} \right.
 \kern-\nulldelimiterspace} 2}}}$-norm can obtain sparser results than the other two norms during the minimizing process.

\begin{figure*}[htbp]  
\centering
\subfigure[${l_{2,{1 \mathord{\left/
 {\vphantom {1 2}} \right.
 \kern-\nulldelimiterspace} 2}}}$-norm]{
\begin{minipage}{0.3\linewidth}
\centering
  \includegraphics[width=\textwidth]{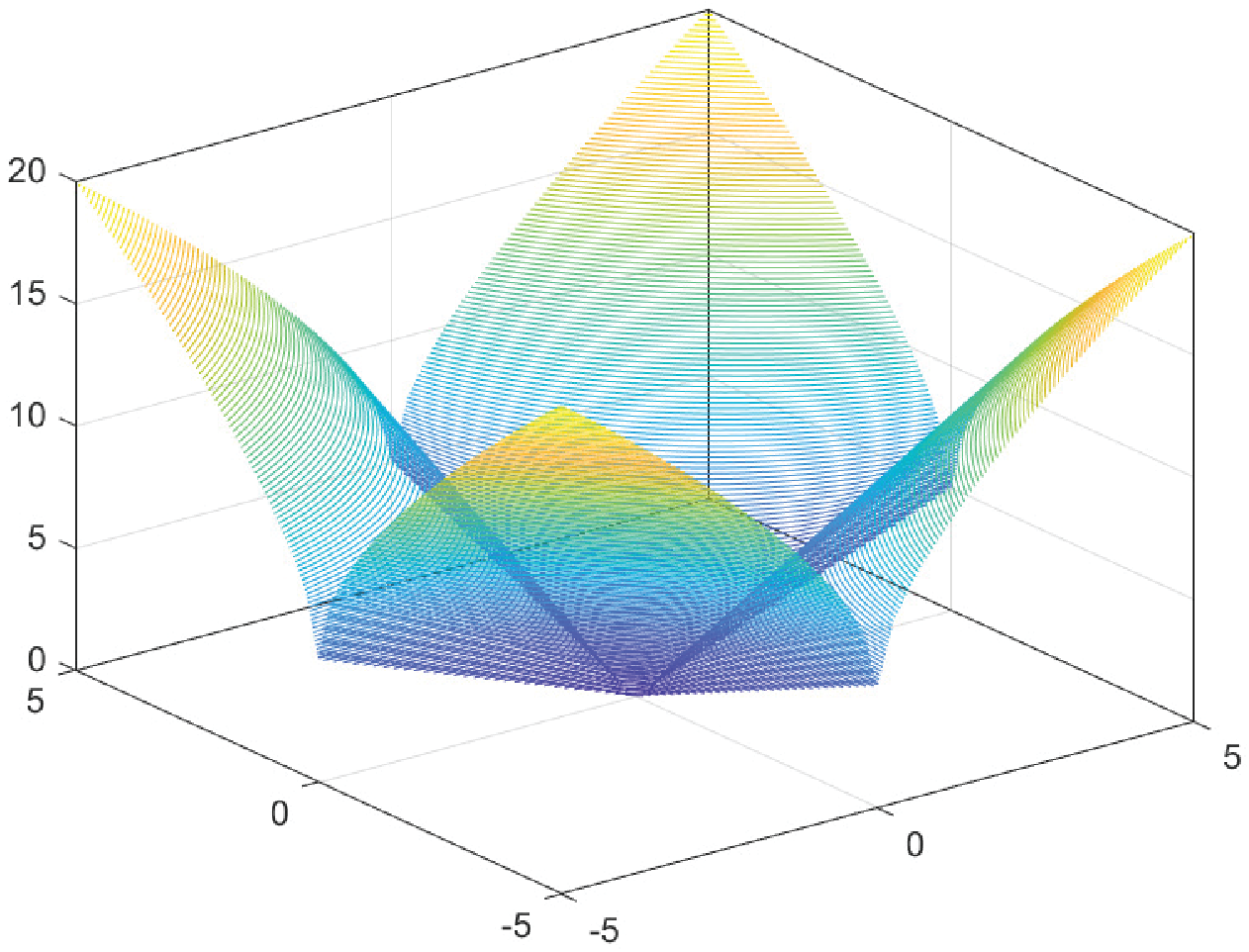}
\end{minipage}
}
\subfigure[${l_{2,1}}$-norm]{
\begin{minipage}{0.3\linewidth}
\centering
  \includegraphics[width=\textwidth]{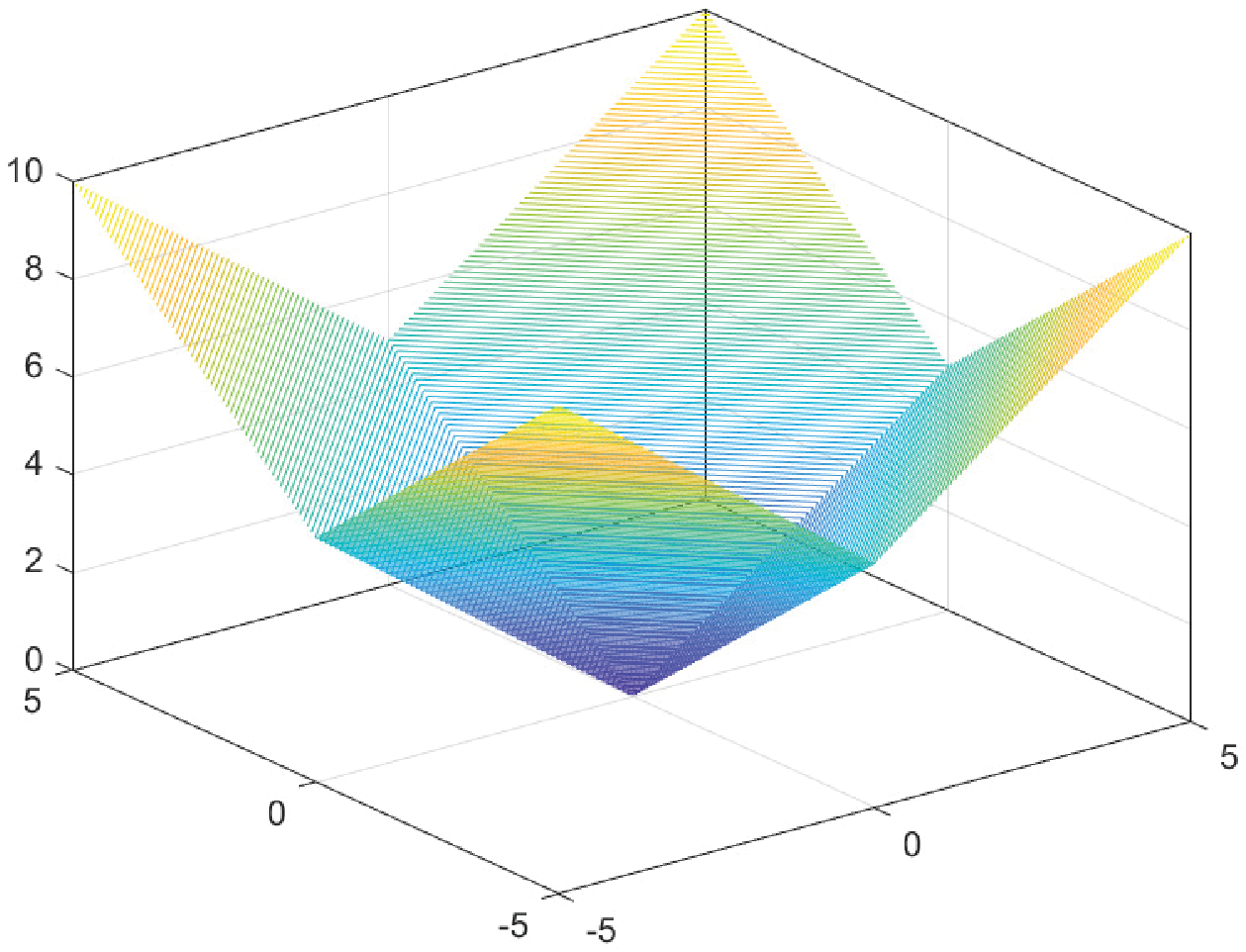}
\end{minipage}
}
\subfigure[${l_{2,2}}$-norm]{
\begin{minipage}{0.3\linewidth}
\centering
  \includegraphics[width=\textwidth]{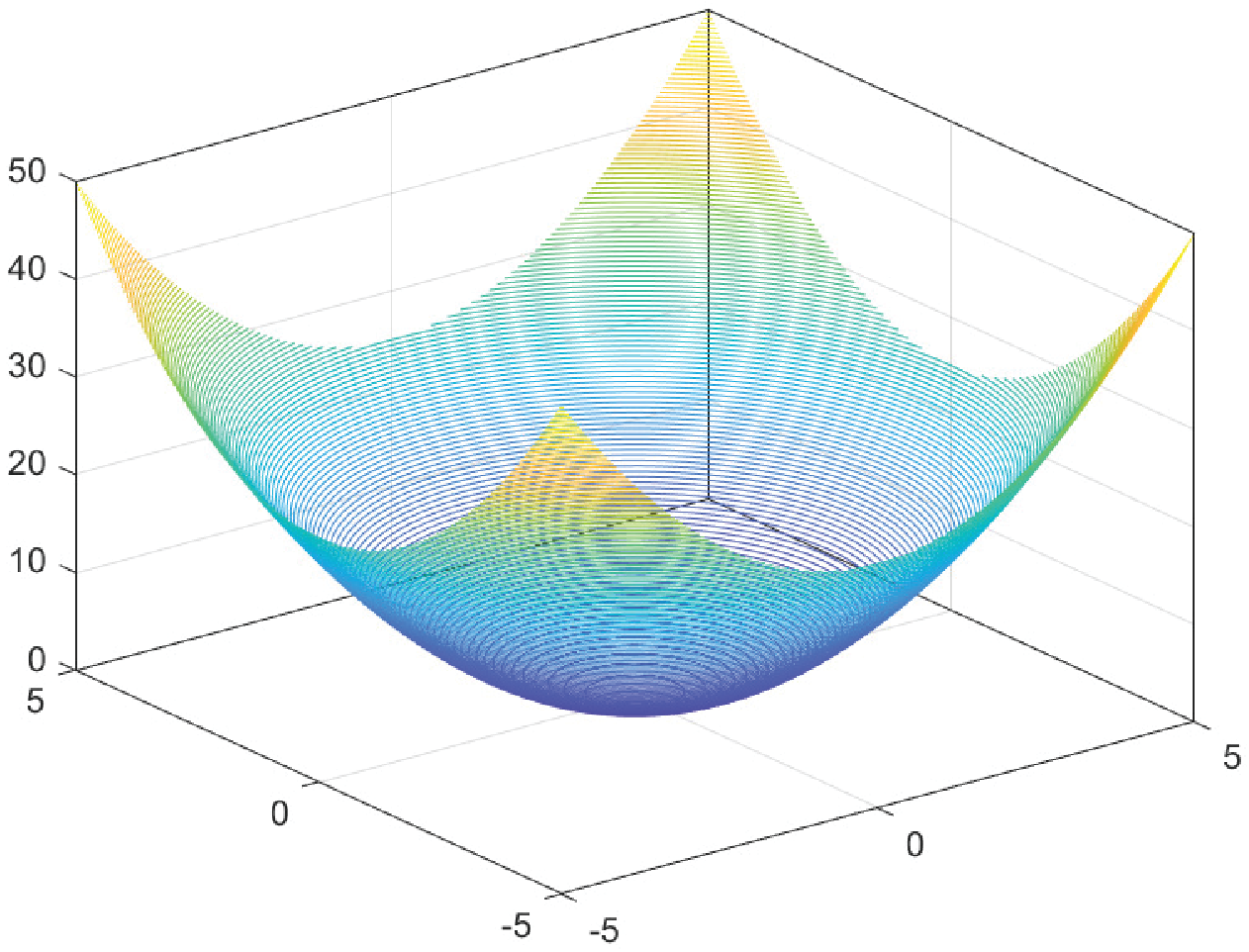}
\end{minipage}
}

\centering
\caption{The contour maps of ${l_{2,{1 \mathord{\left/
 {\vphantom {1 2}} \right.
 \kern-\nulldelimiterspace} 2}}}$-norm, ${l_{2,1}}$-norm, and ${l_{2,2}}$-norm}\label{fig::picture8}
\end{figure*}

Benefiting from the merits of ${l_{2,{1 \mathord{\left/
 {\vphantom {1 2}} \right.
 \kern-\nulldelimiterspace} 2}}}$-norm, the final objective function of the proposed algorithm is formulated as follows.
\begin{equation}
\label{eqn::eq10}
\begin{array}{l}
\mathop {\min }\limits_{W,H,v} \sum\limits_{i = 1}^n {{v_i}\left\| {{x_i} - {x_i}WH} \right\|_2^2}  + \sum\limits_{i = 1}^n {\frac{{{\gamma ^2}}}{{{v_i} + {\gamma  \mathord{\left/
 {\vphantom {\gamma  \eta }} \right.
 \kern-\nulldelimiterspace} \eta }}}}  + {\lambda _1}{\rm{tr}}\left( {{S^{\rm T}}W{\bf{1}}{W^{\rm T}}} \right) + {\lambda _2}{\rm{tr}}\left( {{W^{\rm T}}{X^{\rm T}}LXW} \right) + \alpha \left\| W \right\|_{2,{1 \mathord{\left/
 {\vphantom {1 2}} \right.
 \kern-\nulldelimiterspace} 2}}^{{1 \mathord{\left/
 {\vphantom {1 2}} \right.
 \kern-\nulldelimiterspace} 2}}\\
s.t.\;\;0 \le {v_i} \le 1,i = 1,2, \ldots ,n,W \ge 0,H \ge 0,\;{W^{\rm T}}W = {I_K}
\end{array}
\end{equation}

where $\alpha  > 0$, ${\lambda _1} > 0$ and ${\lambda _2} > 0$ are the trade-off parameters. $\gamma  > 0$ is an interval control parameter, which controls the ``fuzzy interval'' between 0 and 1. The first and second terms in Eq. \eqref{eqn::eq10} aim to maintain the global reconstruction information and relieve the effect of noises to a certain degree. The third term is introduced to reduce the redundancy between features. The fourth term stands for local geometric structure preservation. And the fifth term is designed to promote row-sparsity of the projection matrix.

It is noteworthy that through the projection matrix $W$, important features are selected and uninformative features are excluded. And the $i$th row of $W$ represents the significance of the $i$th feature ${f^i}$. So, after gaining the optimal solution, we sort ${\left\| {{w_i}} \right\|_2}$ in descending order and select the top-ranked $N$ features.

\subsection{Optimization}
Although the objective function is non-convex when optimizing $v$, $W$ and $H$ simultaneously, it is convex when fixing $W$ and $H$ to optimize $v$ and fixing $v$ and $W$ to optimize $H$ \cite{2011Convex}. For variable $W$, there is an efficient algorithm to solve the non-convexity problem of the ${l_{2,{1 \mathord{\left/
 {\vphantom {1 2}} \right.
 \kern-\nulldelimiterspace} 2}}}$-norm \cite{SFSRM}. As a result, the whole optimization problem with three variables can be transformed into three sub-problems with the other two variables fixed, which are easier to tackle. And an iterative updating algorithm is deployed. To be more precise, the corresponding Lagrange function is written as follows.
\begin{equation}
\begin{split}
\label{eqn::eq11}
L\left( {v,\;W,\;H} \right) &= \sum\limits_{i = 1}^n {{v_i}\left\| {{x_i} - {x_i}WH} \right\|_2^2}  + \sum\limits_{i = 1}^n {\frac{{{\gamma ^2}}}{{{v_i} + {\gamma  \mathord{\left/
 {\vphantom {\gamma  \eta }} \right.
 \kern-\nulldelimiterspace} \eta }}}}  + {\lambda _1}{\rm{tr}}\left( {{S^{\rm T}}W{\bf{1}}{W^{\rm T}}} \right) + {\lambda _2}{\rm{tr}}\left( {{W^{\rm T}}{X^{\rm T}}LXW} \right) \\
&+ \alpha \left\| W \right\|_{2,{1 \mathord{\left/
 {\vphantom {1 2}} \right.
 \kern-\nulldelimiterspace} 2}}^{{1 \mathord{\left/
 {\vphantom {1 2}} \right.
 \kern-\nulldelimiterspace} 2}} + \frac{{{\lambda _3}}}{2}\left\| {{W^{\rm T}}W - {I_K}} \right\|_F^2 + {\rm{tr}}\left( {\Phi {H^{\rm T}}} \right) + {\rm{tr}}\left( {\psi {W^{\rm T}}} \right)
\end{split}
\end{equation}

where ${\lambda _3} > 0$ is the balance parameter. $\Phi $ and $\psi $ are the Lagrange multipliers to guarantee the nonnegative constraints. Then, three sub-problems need to be minimized iteratively as follows.
\begin{equation}
\label{eqn::eq12}
{v_{t + 1}} = \arg \mathop {\min }\limits_v {L_v}\left( {v,\;{W_t},\;{H_t}} \right)
\end{equation}
\begin{equation}
{W_{t + 1}} = \arg \mathop {\min }\limits_W {L_W}\left( {{v_{t + 1}},\;W,\;{H_t}} \right)
\end{equation}
\begin{equation}
{H_{t + 1}} = \arg \mathop {\min }\limits_H {L_H}\left( {{v_{t + 1}},\;{W_{t + 1}},\;H} \right)
\end{equation}

\subsubsection{Updating $v$ with fixed $W$ and $H$}
When $W$ and $H$ are fixed, the Lagrange function can be easily converted to Eq. \eqref{eqn::eq133}.
\begin{equation}
\begin{split}
\label{eqn::eq133}
&L\left( v \right) = \sum\limits_{i = 1}^n {{v_i}\left\| {{x_i} - {x_i}WH} \right\|_2^2}  + \sum\limits_{i = 1}^n {\frac{{{\gamma ^2}}}{{{v_i} + {\gamma  \mathord{\left/
 {\vphantom {\gamma  \eta }} \right.
 \kern-\nulldelimiterspace} \eta }}}} ,\;\\
&s.t.\;\;0 \le {v_i} \le 1,\;i = 1,2, \ldots ,n
\end{split}
\end{equation}

Since $\left\| {{x_i} - {x_i}WH} \right\|_2^2$ is irrelevant to variable $v$, we define ${L_i} = \left\| {{x_i} - {x_i}WH} \right\|_2^2$ for simplicity. And Eq. \eqref{eqn::eq133} can be rewritten as
\begin{equation}
\begin{split}
\label{eqn::eq14}
&L\left( v \right) = \sum\limits_{i = 1}^n {{v_i}{L_i}}  + \sum\limits_{i = 1}^n {\frac{{{\gamma ^2}}}{{{v_i} + {\gamma  \mathord{\left/
 {\vphantom {\gamma  \eta }} \right.
 \kern-\nulldelimiterspace} \eta }}}} ,\;\\
&s.t.\;\;0 \le {v_i} \le 1,\;i = 1,2, \ldots ,n.
\end{split}
\end{equation}

Note that Eq. \eqref{eqn::eq14} can be decomposed into $n$ independent sub-problems as follows.
\begin{equation}
\label{eqn::eq15}
L\left( {{v_i}} \right) = {v_i}{L_i} + \frac{{{\gamma ^2}}}{{{v_i} + {\gamma  \mathord{\left/
 {\vphantom {\gamma  \eta }} \right.
 \kern-\nulldelimiterspace} \eta }}},\;s.t.\;\;0 \le {v_i} \le 1
\end{equation}

The closed form solution of ${v_i}$ is
\begin{equation}
\label{eqn::eq16}
{v_i} = \left\{ \begin{array}{l}
1,\;\;\;\;\;\;\;\;\;\;\;\;\;\;\;if\;{L_i} \le {\left( {\frac{{\lambda \gamma }}{{\lambda  + \gamma }}} \right)^2};\\
0,\;\;\;\;\;\;\;\;\;\;\;\;\;\;\;if\;{L_i} \ge {\lambda ^2};\\
\gamma \left( {\frac{1}{{\sqrt {{L_i}} }} - \frac{1}{\lambda }} \right),\;otherwise.
\end{array} \right.
\end{equation}

\subsubsection{Updating $H$ with fixed $v$ and $W$}
When $v$ and $W$ are fixed, the Lagrange function can be easily converted to Eq. \eqref{eqn::eq17}
\begin{equation}
\label{eqn::eq17}
L\left( H \right) = \sum\limits_{i = 1}^n {{v_i}\left\| {{x_i} - {x_i}WH} \right\|_2^2}  + {\rm{tr}}\left( {\Phi {H^{\rm T}}} \right) = \left\| {G - GWH} \right\|_F^2 + {\rm{tr}}\left( {\Phi {H^{\rm T}}} \right).
\end{equation}

where $G = UX$ and $U = diag\left( {\sqrt v } \right)$.

Taking the derivative of Eq. \eqref{eqn::eq17} with respect to $H$, we have
\begin{equation}
\label{eqn::eq18}
\frac{{\partial L\left( H \right)}}{{\partial H}} =  - 2{W^{\rm T}}{G^{\rm T}}G + 2{W^{\rm T}}{G^{\rm T}}GWH + \Phi.
\end{equation}

According to the Karush–Kuhn–Tucker (KKT) condition, namely ${\Phi _{ij}}{H_{ij}} = 0$, the updating rule for $H$ is obtained as follows \cite{art28}.
\begin{equation}
\label{eqn::eq19}
{H_{ij}} = {H_{ij}}\frac{{{{\left( {{W^{\rm T}}{G^{\rm T}}G} \right)}_{ij}}}}{{{{\left( {{W^{\rm T}}{G^{\rm T}}GWH} \right)}_{ij}}}}
\end{equation}

\subsubsection{Updating $W$ with fixed $v$ and $H$}
When $v$ and $H$ are fixed, the Lagrange function can be easily converted to Eq. \eqref{eqn::eq20}.
\begin{equation}
\begin{split}
\label{eqn::eq20}
L\left( W \right) &= \sum\limits_{i = 1}^n {{v_i}\left\| {{x_i} - {x_i}WH} \right\|_2^2}  + {\lambda _1}{\rm{tr}}\left( {{S^{\rm T}}W{\bf{1}}{W^{\rm T}}} \right) + {\lambda _2}{\rm{tr}}\left( {{W^{\rm T}}{X^{\rm T}}LXW} \right) \\
&+ \alpha \left\| W \right\|_{2,{1 \mathord{\left/
 {\vphantom {1 2}} \right.
 \kern-\nulldelimiterspace} 2}}^{{1 \mathord{\left/
 {\vphantom {1 2}} \right.
 \kern-\nulldelimiterspace} 2}} + \frac{{{\lambda _3}}}{2}\left\| {{W^{\rm T}}W - {I_K}} \right\|_F^2 + {\rm{tr}}\left( {\psi {W^{\rm T}}} \right)
\end{split}
\end{equation}

Taking the derivative of Eq. \eqref{eqn::eq20} with respect to $W$, we have
\begin{equation}
\label{eqn::eq21}
\frac{{\partial L\left( W \right)}}{{\partial W}} =  - 2{G^{\rm T}}G{H^{\rm T}} + 2{G^{\rm T}}GWH{H^{\rm T}} + 2\alpha MW + 2{\lambda _1}SW{\bf{1}} + 2{\lambda _2}{X^{\rm T}}LXW + 2{\lambda _3}W{W^{\rm T}}W - 2{\lambda _3}W + \Psi .
\end{equation}

where $M$ is a diagonal matrix with ${M_{ii}} = \frac{1}{{\max \left( {\left\| {{w_i}} \right\|_2^{{3 \mathord{\left/
 {\vphantom {3 2}} \right.
 \kern-\nulldelimiterspace} 2}},\varepsilon } \right)}}$. $\varepsilon $ is a small constant preventing the denominator from being zero.

According to the Karush–Kuhn–Tucker (KKT) condition, namely ${\Psi _{ij}}{W_{ij}} = 0$, the updating rule for $W$ is obtained as follows \cite{art28}.
\begin{equation}
\label{eqn::eq22}
{W_{ij}} = {W_{ij}}\frac{{{{\left( {{G^{\rm T}}G{H^{\rm T}} + {\lambda _2}{X^{\rm T}}ZXW + {\lambda _3}W} \right)}_{ij}}}}{{{{\left( {{G^{\rm T}}GWH{H^{\rm T}} + \alpha MW + {\lambda _1}SW{\bf{1}} + {\lambda _2}{X^{\rm T}}DXW + {\lambda _3}W{W^{\rm T}}W} \right)}_{ij}}}}
\end{equation}

The procedure of SPLR is described by Algorithm ~\ref{alg::SPLR}.

\begin{algorithm}[htbp]
\caption{UFS via self-paced learning and low-redundant regularization}
\label{alg::SPLR}
\begin{algorithmic}[1]
\Require
Data matrix $X \in {\mathbb{R}^{n \times d}}$, Parameters $\alpha  > 0$, $\gamma  > 0$, ${\lambda _1} > 0$, ${\lambda _2} > 0$ and $\mu  > 1$, Maximum iteration number $NIter$, Dimension of the subspace $K$;
\Ensure
Index of selected features $index$, Projection matrix $W \in {\mathbb{R}^{d \times K}}$;
\State Initialize $W = ones\left( {d,K} \right)$, $H = rand\left( {K,d} \right)$, $\eta $;
\State Construct the affinity matrices $S \in {\mathbb{R}^{d \times d}}$, $Z \in {\mathbb{R}^{n \times n}}$;
\Repeat
\State Update $v$ via Eq. \eqref{eqn::eq16};
\State Update $H$ via Eq. \eqref{eqn::eq19};
\State Update $W$ vis Eq. \eqref{eqn::eq22};
\State Update $\eta $ via $\eta  = \mu \eta $;
\Until Convergence
\State Sort all features according to $\sum\limits_{j = 1}^d {w_{ij}^2} $ in descending order and select the top $N$ ranked features.
\end{algorithmic}
\end{algorithm}

\subsection{Convergence analysis}
The proof of the convergence of SPLR can be divided into three parts. Firstly, the value of Eq. \eqref{eqn::eq10} is non-increasing under the updating rule Eq. \eqref{eqn::eq16} with fixed $W$ and $H$. Secondly, the value of Eq. \eqref{eqn::eq10} is non-increasing under the updating rule Eq. \eqref{eqn::eq19} with fixed $v$ and $W$. Lastly, the value of Eq. \eqref{eqn::eq10} is non-increasing under the updating rule Eq. \eqref{eqn::eq22} with fixed $v$ and $H$.

Due to the closed form solution of ${v_i}$, there is no doubt that the value of Eq. \eqref{eqn::eq10} will monotonically decrease when optimizing ${v_i}$. Next, the convergence under the updating rule of the variable $H$ is to be testified.

\begin{definition}
$F\left( h \right)$ is non-increasing under the updating rule
\begin{equation}
\label{eqn::eq23}
{h^{t + 1}} = \arg \mathop {\min }\limits_h J\left( {h,{h^t}} \right),
\end{equation}

where $J\left( {h,h'} \right)$ is an auxiliary function of $F\left( h \right)$ subject to the following conditions.
\begin{equation}
\label{eqn::eq24}
J\left( {h,h'} \right) \ge F\left( h \right),\;J\left( {h,h} \right) = F\left( h \right)
\end{equation}
\end{definition}

\begin{proof}
$F\left( {{h^{t + 1}}} \right) \le J\left( {{h^{t + 1}},{h^t}} \right) \le J\left( {{h^t},{h^t}} \right) = F\left( {{h^t}} \right)$.
\end{proof}

Fixing $v$ and $W$, the objective function is switched as follows.
\begin{equation}
\label{eqn::eq25}
F\left( H \right) = \left\| {G - GWH} \right\|_F^2 = {\rm{tr}}\left( {\left( {G - GWH} \right){{\left( {G - GWH} \right)}^{\rm T}}} \right)
\end{equation}

Manifested below are the first-order and the second-order derivatives of Eq. \eqref{eqn::eq25} with respect to $H$, respectively.
\begin{align}
\label{eqn::eq26-7}
&{F^\prime_{ij}}  = {\left( {2{W^{\rm T}}{G^{\rm T}}GWH - 2{W^{\rm T}}{G^{\rm T}}G} \right)_{ij}}\\
&{F^{\prime \prime }_{ij}} = {\left( {2{W^{\rm T}}{G^{\rm T}}GW} \right)_{ii}}
\end{align}

\begin{lemma}
${J_{ij}}\left( {{H_{ij}},H_{ij}^{\left( t \right)}} \right)$ is an auxiliary function of ${F_{ij}}\left( {{H_{ij}}} \right)$ defined as
\begin{equation}
\label{eqn::eq28}
{J_{ij}}\left( {{H_{ij}},H_{ij}^{\left( t \right)}} \right) = {L_{ij}}\left( {H_{ij}^{\left( t \right)}} \right) + {L^\prime_{ij}} \left( {H_{ij}^{\left( t \right)}} \right)\left( {{H_{ij}} - H_{ij}^{\left( t \right)}} \right) + \frac{{{{\left( {{W^{\rm T}}{G^{\rm T}}GW{H^{\left( t \right)}}} \right)}_{ij}}}}{{H_{ij}^{\left( t \right)}}}{\left( {{H_{ij}} - H_{ij}^{\left( t \right)}} \right)^2}
\end{equation}
\end{lemma}

\begin{proof}
The second-order Taylor expansion of ${F_{ij}}\left( {{H_{ij}}} \right)$ can be calculated  as follows.
\begin{equation}
\label{eqn::eq29}
{F_{ij}}\left( {{H_{ij}}} \right) = {F_{ij}}\left( {H_{ij}^{\left( t \right)}} \right) + {F^\prime_{ij}} \left( {H_{ij}^{\left( t \right)}} \right)\left( {{H_{ij}} - H_{ij}^{\left( t \right)}} \right) + \frac{1}{2}{F^{\prime \prime }_{ij}}\left( {H_{ij}^{\left( t \right)}} \right){\left( {{H_{ij}} - H_{ij}^{\left( t \right)}} \right)^2}
\end{equation}

If ${F_{ij}}\left( {{H_{ij}}} \right)$ conforms to the inequality specified in Eq. \eqref{eqn::eq24}, Eq. \eqref{eqn::eq30} ought to be satisfied.
\begin{equation}
\label{eqn::eq30}
\frac{{{{\left( {{W^{\rm T}}{G^{\rm T}}GW{H^{\left( t \right)}}} \right)}_{ij}}}}{{H_{ij}^{\left( t \right)}}} \ge \frac{1}{2}{\left( {2{W^{\rm T}}{G^{\rm T}}GW} \right)_{ii}}
\end{equation}

It is because $H$, $W$, $G$ are all nonnegative that the following formula holds.
\begin{equation}
\label{eqn::eq31}
{\left( {{W^{\rm T}}{G^{\rm T}}GW{H^{\left( t \right)}}} \right)_{ij}} = \sum\limits_{k = 1}^K {{{\left( {{W^{\rm T}}{G^{\rm T}}GW} \right)}_{ik}}H_{kj}^{\left( t \right)}} \ge {\left( {{W^{\rm T}}{G^{\rm T}}GW} \right)_{ii}}H_{ij}^{\left( t \right)}
\end{equation}

In consideration of the special case where $H_{ij}^{\left( t \right)} = {H_{ij}}$, we have
\begin{equation}
\label{eqn::eq32}
J\left( {{H_{ij}},{H_{ij}}} \right) = {\rm{F}}\left( {{H_{ij}}} \right)
\end{equation}

Thereby, it is clear that ${J_{ij}}\left( {{H_{ij}},H_{ij}^{\left( t \right)}} \right)$ is effective to be regarded as an auxiliary function.
\end{proof}

Substituting Eq. \eqref{eqn::eq28} into Eq. \eqref{eqn::eq23} and setting the corresponding derivative with respect to ${H_{ij}}$ as 0, we have
\begin{equation}
\label{eqn::eq33}
H_{ij}^{\left( {t + 1} \right)} = H_{ij}^{\left( t \right)} - H_{ij}^{\left( t \right)}\frac{{{L^\prime_{ij}} \left( {H_{ij}^{\left( t \right)}} \right)}}{{2{{\left( {{W^{\rm T}}{G^{\rm T}}GW{H^{\left( t \right)}}} \right)}_{ij}}}} = H_{ij}^{\left( t \right)}\frac{{{{\left( {{W^{\rm T}}{G^{\rm T}}G} \right)}_{ij}}}}{{{{\left( {{W^{\rm T}}{G^{\rm T}}GW{H^{\left( t \right)}}} \right)}_{ij}}}}.
\end{equation}

Evidently, Eq. \eqref{eqn::eq33} is in accordance with the updating rule displayed in Eq. \eqref{eqn::eq19}, which completes the proof.

Similarly, the convergence under the updating rule of the variable $W$ is able to be proved. To sum up, the objective function of SPLR decreases monotonically in the process of optimization.

\section{Experiments and analysis}
\label{sec:Experiments}

In this section, the effectiveness of SPLR is compared with seven state-of-art algorithms on nine benchmark datasets. K-means and PAM are used for clustering on features ranking by SPLR and the experimental results are recorded. The convergence of SPLR is further verified empirically and the influence of different parameter settings on the performance of SPLR is explored.

\subsection{Datasets}
The experiments are carried out on nine real-world datasets consisting of one digit image dataset (USPS\footnote[1]{\href{https://jundongl.github.io/scikit-feature/datasets.html}{https://jundongl.github.io/scikit-feature/datasets.html\label{web1}}}), one artificial dataset (Madelon\textsuperscript{\ref {web1}}), one speech signal dataset (Isolet\textsuperscript{\ref {web1}}), three face image datasets (Umist\footnote[2]{\href{https://gitee.com/csliangdu/LGRUFS/tree/master/data}{https://gitee.com/csliangdu/LGRUFS/tree/master/data\label{web2}}}, ORL\textsuperscript{\ref {web1}} and warpPIE10P\textsuperscript{\ref {web1}}), one object image dataset (COIL20\textsuperscript{\ref {web1}}), and two biological microarray datasets (Colon\textsuperscript{\ref {web1}} and GLIOMA\textsuperscript{\ref {web1}}). Details are shown in Table \ref{tab::table2}.
\begin{table}[htbp]
\renewcommand\arraystretch{0.9}
\begin{center}
\caption{The summary of experimental datasets}\label{tab::table2}
\begin{tabular}[c]{lllll}
\toprule
Dataset& \#Instance& \#Feature& \#Class& Type  \\
\midrule
USPS& 9258& 256& 10& Digit images  \\
Madelon& 2600& 500& 2& Artificial  \\
Isolet& 1560& 617& 26& Speech Signal  \\
Umist& 575& 644& 20& Face images  \\
COIL20& 1440& 1024& 20& Object images  \\
ORL& 400& 1024& 40& Face images  \\
Colon& 62& 2000& 2& Biological microarray  \\
warpPIE10P& 210& 2420& 10& Face images  \\
GLIOMA& 50& 4434& 4& Biological microarray  \\
\bottomrule
\end{tabular}
\end{center}
\end{table}

\subsection{Comparison methods}
For the purpose of validating the effectiveness of SPLR\footnote[3]{\href{https://github.com/lllwy/SPLR}{https://github.com/lllwy/SPLR\label{web3}}}, seven state-of-art algorithms are applied in comparison with the proposed one. The brief introduction of each method is as follows.
\begin{itemize}
  \item Baseline: Baseline adopts original features without feature selection.
  \item LS \cite{art12}: Laplacian Score is inclined to select features with larger variance and less fluctuation within class.
  \item MCFS \cite{art29}: Multi-cluster feature selection algorithm chooses features that are capable of maintaining the multi-cluster structure of the data.
  \item UDFS \cite{art30}: The unsupervised discriminative feature selection algorithm defines local total scatter matrix and between class scatter matrix for different samples so as to select the most representative features.
  \item DISR \cite{art20}: The diversity-induced self-representation for unsupervised feature selection algorithm takes the diversity of features into consideration to reduce the redundancy.
  \item RNE \cite{art31}: The robust neighborhood embedding algorithm minimizes the residual error based on the presumption that each sample can be reconstructed by its neighbors.
  \item SGFS \cite{art32}: The subspace learning-based graph regularized feature selection algorithm brings in a regularization term to keep the local manifold structure of features unchanged.
\end{itemize}

\subsection{Evaluation metrics}
In this paper, Clustering Accuracy (ACC) and Normalized Mutual Information (NMI) are adopted to assess the performance of the aforementioned algorithms \cite{art33,art34}.

NMI is defined as
\begin{equation}
\label{eqn::eq34}
{\rm{NMI}}\left( {P,Q} \right) = \frac{{I\left( {P,Q} \right)}}{{\sqrt {H\left( P \right)H\left( Q \right)} }},
\end{equation}

where $I\left( {P,Q} \right)$ denotes the mutual information between $P$ and $Q$. $H\left( P \right)$ and $H\left( Q \right)$ represent the entropy of $P$ and $Q$, respectively. In practice, $P$ and $Q$ refer to the clustering label and the ground truth label, respectively.

ACC is defined as
\begin{equation}
\label{eqn::eq35}
{\rm{ACC}} = \frac{{\sum\nolimits_{i = 1}^n {\delta \left( {{p_i},\;map\left( {{q_i}} \right)} \right)} }}{n},
\end{equation}

where ${p_i}$ and ${q_i}$ stand for the clustering label and the ground truth label of the sample ${x_i}$, respectively. $\delta \left( {x,\;y} \right) = \left\{ \begin{array}{l}
1,\;\;if\;x = y\\
0,\;otherwise
\end{array} \right.$, and $map\left( \cdot \right)$ is a function which matches ${p_i}$ and ${q_i}$.
It takes ${p_i}$ as the reference label and rearranges ${q_i}$ in the same order as ${p_i}$. It is used to solve the problem of label inconsistency. Kuhn-Munkres or Hungarian Algorithm is often utilized to achieve this goal \cite{match1986}.

As can be perceived, an algorithm with larger NMI and ACC is expected.

\subsection{Experimental settings}
Parameters are tuned according to the referenced papers of the corresponding algorithms. To be more accurate, for LS, MCFS, UDFS, RNE and SGFS, we set the neighborhood size $k$ as 5. For LS and SGFS, the bandwidth parameter of Gaussian kernel $\sigma $ is fixed to 10. For SGFS and SPLR, we set the dimension of the subspace $K$ as 200. For RNE, $\alpha $ is set as ${10^3}$ to ensure the orthogonality constraint. Following \cite{art35}, $\mu $ and $\gamma $ in SPLR are set as 1.05 and 2, respectively. Furthermore, for UDFS, the regularization parameter $\gamma $ is searched from $\left\{ {{{10}^{ - 9}},\;{{10}^{ - 6}},\;{{10}^{ - 3}},\;1,\;{\rm{1}}{{\rm{0}}^{\rm{3}}}{\rm{,}}\;{\rm{1}}{{\rm{0}}^{\rm{6}}}{\rm{,}}\;{\rm{1}}{{\rm{0}}^{\rm{9}}}} \right\}$. For DISR, ${\lambda _1}$ and ${\lambda _{\rm{2}}}$ are adjusted in the range of $\left\{ {0.01,\;0.05,\;0.1,\;0.5,\;1,\;5,\;10,\;50,\;100,\;500} \right\}$. Other parameters are all tested in $\{ {{10}^{ - 3}},\;{{10}^{ - 2}},\;{{10}^{ - 1}},\;1,\;{\rm{1}}{{\rm{0}}^1},\;{\rm{1}}{{\rm{0}}^2},$ $\;{\rm{1}}{{\rm{0}}^3} \}$. Features are selected from 20 to 200 with the interval 20. Since we have no access to the label information of data, the performance of the algorithms is judged by clustering tasks instead of classification tasks. Firstly, different algorithms are exerted to select $N$ features. In this way, data with high dimension is transformed into data with low dimension. Then, clustering methods are made use of to group the low-dimensional data into $c$ classes. Finally, evaluation metrics are applied to assess the performance. K-means and PAM are repeated 20 times with random initializations and the average results are recorded for comparison \cite{art36,art37}. The best results with the optimal parameters and the number of selected features are derived for comparison.

\subsection{Effect of ${l_{2,{1 \mathord{\left/
 {\vphantom {1 2}} \right.
 \kern-\nulldelimiterspace} 2}}}$-norm}
Since this paper focuses on unsupervised learning, and Wang et al. just verified the effectiveness of ${l_{2,{1 \mathord{\left/
 {\vphantom {1 2}} \right.
 \kern-\nulldelimiterspace} 2}}}$ regularization term for classification tasks, the superiority of ${l_{2,{1 \mathord{\left/
 {\vphantom {1 2}} \right.
 \kern-\nulldelimiterspace} 2}}}$ regularization term is further discussed when the label information is unavailable \cite{art27}. Datasets including COIL20, ORL and GLIOMA are used for this purpose and parameters are all fixed to 1. Tables \ref{tab::table9}-\ref{tab::table10} and Tables \ref{tab::table11}-\ref{tab::table12} demonstrate the clustering results with different regularization terms and different number of selected features in terms of ACC and NMI, respectively.

\begin{table*}[!htb]\tiny
\renewcommand\arraystretch{1.2}
\centering
\begin{center}
\caption{The clustering results of SPLR with different regularization terms in terms of ACC}\label{tab::table9}
\resizebox{\textwidth}{!}{
\begin{tabular}[c]{llll|lll|lll}
\thickhline
\multirow{2}{*}{Dataset}& \multicolumn{3}{c|}{Top 40 features}& \multicolumn{3}{c|}{Top 80 features}& \multicolumn{3}{c}{Top 120 features}\\
\cline{2-10}
& ${p = 0.5}$ & ${p = 1}$ & ${p = 2}$ & ${p = 0.5}$ & ${p = 1}$ & ${p = 2}$ & ${p = 0.5}$ & ${p = 1}$ & ${p = 2}$\\
\thickhline
COIL20 & \textbf{54.84} & 54.72 & 54.10 & \textbf{58.09} & 54.22 & 53.99 & 56.49 & 54.90 & \textbf{57.17} \\
ORL & \textbf{58.56} & 57.44 & 57.13 & \textbf{63.00} & 60.94 & 62.56 & 60.00 & \textbf{61.69} & 60.13 \\
GLIOMA & \textbf{56.50} & 54.00 & 56.00 & \textbf{61.00} & 60.50 & 59.50 & \textbf{61.00} & 58.50 & 57.00 \\
\thickhline
\end{tabular}}
\end{center}
\end{table*}

\begin{table*}[!htb]\tiny
\renewcommand\arraystretch{1.2}
\centering
\begin{center}
\caption{The clustering results of SPLR with different regularization terms in terms of ACC}\label{tab::table10}
\resizebox{\textwidth}{!}{
\begin{tabular}[c]{llll|lll|lll}
\thickhline
\multirow{2.5}{*}{Dataset}& \multicolumn{3}{c|}{Top 160 features}& \multicolumn{3}{c|}{Top 200 features}& \multicolumn{3}{c}{Top 240 features}\\
\cline{2-10}
& ${p = 0.5}$ & ${p = 1}$ & ${p = 2}$ & ${p = 0.5}$ & ${p = 1}$ & ${p = 2}$ & ${p = 0.5}$ & ${p = 1}$ & ${p = 2}$\\
\thickhline
COIL20 & \textbf{59.65} & 55.94 & 58.72 & \textbf{58.70} & 57.40 & 58.16 & \textbf{58.37} & 56.56 & 57.38 \\
ORL & \textbf{61.56} & 61.38 & 61.06 & \textbf{63.50} & 61.38 & 61.44 & \textbf{62.19} & 61.69 & 61.50 \\
GLIOMA & 61.50 & \textbf{64.00} & 60.50 & \textbf{64.50} & \textbf{64.50} & 63.00 & \textbf{64.50} & 60.50 & 62.00 \\
\thickhline
\end{tabular}}
\end{center}
\end{table*}

\begin{table*}[!htb]\tiny
\renewcommand\arraystretch{1.2}
\centering
\begin{center}
\caption{The clustering results of SPLR with different regularization terms in terms of NMI}\label{tab::table11}
\resizebox{\textwidth}{!}{
\begin{tabular}[c]{llll|lll|lll}
\thickhline
\multirow{2.5}{*}{Dataset}& \multicolumn{3}{c|}{Top 40 features}& \multicolumn{3}{c|}{Top 80 features}& \multicolumn{3}{c}{Top 120 features}\\
\cline{2-10}
& ${p = 0.5}$ & ${p = 1}$ & ${p = 2}$ & ${p = 0.5}$ & ${p = 1}$ & ${p = 2}$ & ${p = 0.5}$ & ${p = 1}$ & ${p = 2}$\\
\thickhline
COIL20 & 68.32 & 69.16 & \textbf{69.60} & \textbf{71.86} & 71.54 & 70.91 & 72.46 & 72.63 & \textbf{73.21} \\
ORL & \textbf{84.12} & 83.70 & 83.80 & \textbf{85.69} & 85.53 & 85.52 & \textbf{85.79} & 85.65 & 85.06 \\
GLIOMA & \textbf{52.18} & 49.29 & 48.23 & 56.08 & \textbf{56.84} & 52.60 & \textbf{55.52} & 50.22 & 48.76 \\
\thickhline
\end{tabular}}
\end{center}
\end{table*}

\begin{table*}[!htb]\tiny
\renewcommand\arraystretch{1.2}
\centering
\begin{center}
\caption{The clustering results of SPLR with different regularization terms in terms of NMI}\label{tab::table12}
\resizebox{\textwidth}{!}{
\begin{tabular}[c]{llll|lll|lll}
\thickhline
\multirow{2.5}{*}{Dataset}& \multicolumn{3}{c|}{Top 160 features}& \multicolumn{3}{c|}{Top 200 features}& \multicolumn{3}{c}{Top 240 features}\\
\cline{2-10}
& ${p = 0.5}$ & ${p = 1}$ & ${p = 2}$ & ${p = 0.5}$ & ${p = 1}$ & ${p = 2}$ & ${p = 0.5}$ & ${p = 1}$ & ${p = 2}$\\
\thickhline
COIL20 & \textbf{74.30} & 72.21 & 74.03 & \textbf{74.77} & 73.75 & 74.33 & 74.27 & 73.97 & \textbf{74.70} \\
ORL & \textbf{86.09} & 85.53 & 85.93 & \textbf{86.17} & 85.98 & 86.16 & \textbf{86.25} & 85.83 & 86.11 \\
GLIOMA & \textbf{58.40} & 53.50 & 55.81 & \textbf{61.09} & 60.36 & 58.10 & \textbf{61.40} & 57.04 & 58.39 \\
\thickhline
\end{tabular}}
\end{center}
\end{table*}

From Tables \ref{tab::table9} and \ref{tab::table10}, it can be seen that except for selecting 120 features on COIL20 and ORL and selecting 160 features on GLIOMA, SPLR with ${l_{2,{1 \mathord{\left/
 {\vphantom {1 2}} \right.
 \kern-\nulldelimiterspace} 2}}}$ regularization term achieves the highest ACC. From Tables \ref{tab::table11} and \ref{tab::table12}, it can be observed that SPLR with ${l_{2,{1 \mathord{\left/
 {\vphantom {1 2}} \right.
 \kern-\nulldelimiterspace} 2}}}$ regularization term surpasses SPLR with ${l_{2,1}}$ regularization term and ${l_{2,2}}$ regularization term on ORL regardless of the number of selected features in terms of NMI. In addition, it outperforms others on GLIOMA except for selecting 80 features. On COIL20, best NMI can be obtained by SPLR with ${l_{2,{1 \mathord{\left/
 {\vphantom {1 2}} \right.
 \kern-\nulldelimiterspace} 2}}}$ regularization term when 80, 160 and 200 features are chosen. Therefore, it is sensible to utilize ${l_{2,{1 \mathord{\left/
 {\vphantom {1 2}} \right.
 \kern-\nulldelimiterspace} 2}}}$ regularization term for robustness and sparsity.

\subsection{Clustering results and analysis}
The clustering results on nine datasets using K-means and PAM in terms of ACC and NMI are listed in Tables \ref{tab::table3} and \ref{tab::table4}, Tables \ref{tab::table5} and \ref{tab::table6}, respectively.

When using K-means for clustering, SPLR outperforms other algorithms on seven datasets, including USPS, Madelon, Isolet, COIL20, ORL, Colon and GLIOMA. DISR achieves the best result on Umist, and UDFS gains highest ACC and NMI on warpPIE10P. Moreover, it turns out that SPLR behaves better than baseline on all datasets, which confirms the necessity for feature selection. Not only is SPLR superior to RNE and SGFS on all datasets, but also it outperforms DISR on most datasets except Umist. The reasons are  as follows. 1) SGFS ignores the redundancy between features and the negative impact of outliers. 2) RNE only keeps the local manifold structure unchanged. 3) DISR merely takes the diversity of features into account. In contrast, SPLR considers the local manifold structure of data as well as the diversity of both features and data simultaneously, which facilitates the feature selection process.

When using PAM for clustering, SPLR outperforms other algorithms on six datasets, including USPS, Madelon, COIL20, Colon, WarpPIE10P and GLIOMA, which also emphasizes its effectiveness. And the clustering results improve a lot on all datasets excluding Isolet when comparing SPLR with baseline, which uses all features for the task. In addition, baseline, MCFS and DISR obtain the optimal ACC and NMI on Isolet, Umist and ORL, respectively.

\begin{table*}[htbp]
\centering
\begin{center}
\caption{The clustering results on benchmark datasets using K-means in terms of ACC}\label{tab::table3}
\resizebox{\textwidth}{!}{
\begin{tabular}[c]{lllllllll}
\toprule
Dataset& Baseline& LS& MCFS& UDFS& DISR& RNE& SGFS& SPLR  \\
\midrule
USPS& 63.45$\pm $0.54 & 63.20$\pm $1.06  & 64.91$\pm $1.91 & 61.33$\pm $1.01 & 63.53$\pm $0.86 & 64.12$\pm $1.02 & 65.42$\pm $0.54 & \textbf{66.36$\pm $0.90} \\
Madelon& 28.30$\pm $2.11 & 27.80$\pm $1.58  & 27.98$\pm $1.64 & 29.77$\pm $1.95 & 30.88$\pm $1.25 & 29.31$\pm $1.08 & 30.40$\pm $1.45 & \textbf{31.43$\pm $0.80} \\
Isolet& 54.20$\pm $1.82 & 54.00$\pm $1.98  & 55.50$\pm $1.72 & 50.39$\pm $1.70 & 48.27$\pm $1.31 & 39.11$\pm $0.72 & 56.54$\pm $1.60 & \textbf{56.78$\pm $2.06} \\
Umist& 49.46$\pm $0.83 & 48.86$\pm $4.40  & 54.93$\pm $1.28 & 51.05$\pm $2.91 & \textbf{56.71$\pm $1.69} & 49.71$\pm $1.77 & 52.27$\pm $1.19 & 52.69$\pm $1.64 \\
COIL20& 57.51$\pm $0.93 & 56.26$\pm $1.72  & 59.65$\pm $1.59 & 56.23$\pm $0.91 & 58.04$\pm $0.88 & 54.56$\pm $1.95 & 60.22$\pm $0.78 & \textbf{63.12$\pm $0.83} \\
ORL& 61.58$\pm $1.93 & 61.55$\pm $1.81  & 62.91$\pm $1.46 & 61.21$\pm $2.44 & 67.36$\pm $3.77 & 60.09$\pm $2.51 & 61.61$\pm $1.47 & \textbf{68.10$\pm $3.09} \\
Colon& 22.51$\pm $5.13 & 25.24$\pm $3.45  & 24.10$\pm $3.77 & 27.66$\pm $4.17 & 30.64$\pm $3.40 & 26.70$\pm $5.29 & 27.63$\pm $4.96 & \textbf{32.72$\pm $5.79} \\
warpPIE10P& 43.36$\pm $4.96 & 44.63$\pm $2.20  & 46.14$\pm $5.43 & \textbf{55.25$\pm $6.01} & 52.99$\pm $7.79 & 41.74$\pm $2.19 & 53.31$\pm $6.97 & 54.52$\pm $6.78 \\
GLIOMA& 59.78$\pm $6.22 & 60.66$\pm $9.92  & 60.40$\pm $5.95 & 66.30$\pm $10.08 & 65.31$\pm $7.51 & 63.71$\pm $9.00 & 63.19$\pm $6.61 & \textbf{66.66$\pm $7.93} \\
\bottomrule
\end{tabular}}
\end{center}
\end{table*}

\begin{table*}[htbp]
\centering
\begin{center}
\caption{The clustering results on benchmark datasets using K-means in terms of NMI}\label{tab::table4}
\resizebox{\textwidth}{!}{
\begin{tabular}[c]{lllllllll}
\toprule
Dataset& Baseline& LS& MCFS& UDFS& DISR& RNE& SGFS& SPLR  \\
\midrule
USPS& 61.22$\pm $1.12 & 61.10$\pm $1.22  & 62.40$\pm $1.38 & 59.09$\pm $1.30 & 61.19$\pm $1.06 & 61.75$\pm $1.00 & 62.56$\pm $1.07 & \textbf{63.42$\pm $1.36} \\
Madelon& 1.29$\pm $1.29 & 1.51$\pm $1.00  & 0.81$\pm $0.67 & 2.05$\pm $1.70 & 2.70$\pm $2.00 & 2.52$\pm $1.90 & 2.83$\pm $2.65 & \textbf{2.95$\pm $1.30} \\
Isolet& 74.26$\pm $1.29 & 73.68$\pm $1.68  & 74.30$\pm $1.63 & 69.83$\pm $1.39 & 67.68$\pm $1.36 & 59.23$\pm $0.82 & 74.55$\pm $1.35 & \textbf{74.87$\pm $1.66} \\
Umist& 70.26$\pm $0.69 & 68.61$\pm $4.52  & 74.53$\pm $0.91 & 70.45$\pm $1.73 & \textbf{76.00$\pm $1.02} & 69.72$\pm $1.54 & 71.94$\pm $1.06 & 72.45$\pm $1.09 \\
COIL20& 75.15$\pm $0.67 & 73.78$\pm $0.87  & 74.45$\pm $0.85 & 73.43$\pm $0.85 & 73.30$\pm $1.20 & 71.47$\pm $1.24 & 76.07$\pm $0.60 & \textbf{78.12$\pm $0.47} \\
ORL& 85.99$\pm $0.70 & 85.81$\pm $1.14  & 86.59$\pm $0.55 & 85.90$\pm $1.20 & 88.77$\pm $1.09 & 85.44$\pm $1.00 & 85.97$\pm $0.67 & \textbf{88.99$\pm $1.03} \\
Colon& 6.08$\pm $2.12 & 9.60$\pm $8.18  & 8.72$\pm $6.47 & 17.29$\pm $12.77 & 21.27$\pm $13.99 & 11.12$\pm $4.00 & 17.14$\pm $10.87 & \textbf{30.42$\pm $8.77} \\
warpPIE10P& 51.60$\pm $5.81 & 52.42$\pm $3.04  & 55.46$\pm $6.74 & \textbf{65.96$\pm $5.18} & 63.06$\pm $7.56 & 50.99$\pm $3.45 & 62.81$\pm $6.97 & 63.30$\pm $4.88 \\
GLIOMA& 56.04$\pm $9.92 & 56.91$\pm $11.02  & 54.74$\pm $8.07 & 60.40$\pm $8.70 & 56.72$\pm $10.83 & 61.07$\pm $13.07 & 60.21$\pm $7.47 & \textbf{61.82$\pm $7.38} \\
\bottomrule
\end{tabular}}
\end{center}
\end{table*}

\begin{table*}[htbp]
\centering
\begin{center}
\caption{The clustering results on benchmark datasets using PAM in terms of ACC}\label{tab::table5}
\resizebox{\textwidth}{!}{
\begin{tabular}[c]{lllllllll}
\toprule
Dataset& Baseline& LS& MCFS& UDFS& DISR& RNE& SGFS& SPLR  \\
\midrule
USPS& 62.17$\pm $0.97 & 62.23$\pm $0.60  & 64.40$\pm $1.31 & 61.22$\pm $0.70 & 63.53$\pm $0.86 & 64.99$\pm $1.76 & 65.69$\pm $1.17 & \textbf{65.82$\pm $1.64} \\
Madelon& 30.72$\pm $3.99 & 33.54$\pm $4.62  & 32.40$\pm $3.07 & 33.44$\pm $1.87 & 34.53$\pm $3.15 & 32.86$\pm $1.73 & 35.31$\pm $4.69 & \textbf{35.37$\pm $1.63} \\
Isolet& \textbf{64.24$\pm $4.46} & 58.67$\pm $3.65  & 60.90$\pm $3.61 & 55.75$\pm $4.36 & 49.83$\pm $4.55 & 41.50$\pm $0.63 & 62.49$\pm $3.22 & 59.21$\pm $4.10 \\
Umist& 53.98$\pm $1.80 & 52.89$\pm $3.65  & \textbf{60.96$\pm $2.18} & 54.97$\pm $0.99 & 58.89$\pm $1.56 & 52.31$\pm $2.35 & 56.62$\pm $2.45 & 58.08$\pm $2.23 \\
COIL20& 67.73$\pm $2.08 & 64.34$\pm $1.55  & 67.61$\pm $2.42 & 63.51$\pm $1.02 & 62.36$\pm $0.96 & 61.82$\pm $2.82 & 68.81$\pm $0.65 & \textbf{69.91$\pm $1.60} \\
ORL& 63.32$\pm $3.79 & 63.79$\pm $2.40  & 66.75$\pm $4.17 & 65.15$\pm $2.12 & \textbf{67.36$\pm $3.77} & 61.11$\pm $4.85 & 64.31$\pm $4.18 & 63.42$\pm $5.17 \\
Colon& 22.97$\pm $6.17 & 24.23$\pm $1.05  & 25.77$\pm $6.57 & 28.22$\pm $6.20 & 30.47$\pm $3.26 & 24.64$\pm $7.35 & 27.64$\pm $6.03 & \textbf{34.63$\pm $3.81} \\
warpPIE10P& 48.33$\pm $8.36 & 47.91$\pm $5.46  & 50.97$\pm $10.78 & 58.20$\pm $8.41 & 57.85$\pm $9.42 & 44.25$\pm $3.45 & 56.76$\pm $8.57 & \textbf{58.23$\pm $7.47} \\
GLIOMA& 66.70$\pm $10.08 & 63.92$\pm $15.84  & 62.90$\pm $12.03 & 71.29$\pm $12.62 & 70.62$\pm $12.81 & 66.82$\pm $12.49 & 66.37$\pm $12.70 & \textbf{71.63$\pm $16.46} \\
\bottomrule
\end{tabular}}
\end{center}
\end{table*}

\begin{table*}[htbp]
\centering
\begin{center}
\caption{The clustering results on benchmark datasets using PAM in terms of NMI}\label{tab::table6}
\resizebox{\textwidth}{!}{
\begin{tabular}[c]{lllllllll}
\toprule
Dataset& Baseline& LS& MCFS& UDFS& DISR& RNE& SGFS& SPLR  \\
\midrule
USPS& 57.19$\pm $0.63 & 58.77$\pm $1.18  & 58.81$\pm $1.50 & 56.40$\pm $0.62 & 59.16$\pm $0.22 & 58.70$\pm $0.99 & 59.25$\pm $1.33 & \textbf{59.54$\pm $1.12} \\
Madelon& 0.88$\pm $1.11 & 1.43$\pm $1.16  & 1.02$\pm $1.30 & 2.40$\pm $1.11 & 3.31$\pm $2.06 & 2.37$\pm $1.67 & 2.49$\pm $1.53 & \textbf{4.29$\pm $1.90} \\
Isolet& \textbf{76.81$\pm $2.79} & 74.66$\pm $2.73  & 75.46$\pm $1.35 & 70.55$\pm $1.95 & 66.07$\pm $2.59 & 60.64$\pm $1.29 & 76.47$\pm $1.48 & 72.86$\pm $2.11 \\
Umist& 72.21$\pm $1.47 & 70.81$\pm $4.03  & \textbf{77.52$\pm $1.50} & 72.77$\pm $2.18 & 76.98$\pm $1.35 & 69.39$\pm $1.63 & 75.44$\pm $3.14 & 75.69$\pm $0.58 \\
COIL20& 77.36$\pm $1.86 & 76.32$\pm $1.36  & 76.73$\pm $1.16 & 74.64$\pm $2.20 & 72.93$\pm $1.21 & 74.38$\pm $0.78 & 78.66$\pm $1.61 & \textbf{79.41$\pm $0.72} \\
ORL& 86.56$\pm $1.11 & 86.86$\pm $1.05  & 88.05$\pm $1.18 & 86.56$\pm $1.24 & \textbf{88.77$\pm $1.09} & 85.90$\pm $1.58 & 87.22$\pm $1.66 & 86.81$\pm $1.88 \\
Colon& 3.70$\pm $2.63 & 6.68$\pm $4.75  & 8.99$\pm $12.63 & 14.69$\pm $6.00 & 19.62$\pm $12.83 & 9.23$\pm $4.43 & 15.21$\pm $8.15 & \textbf{31.97$\pm $4.17} \\
warpPIE10P& 55.90$\pm $9.14 & 54.70$\pm $3.17  & 57.64$\pm $9.97 & 59.89$\pm $7.91 & 65.58$\pm $9.32 & 52.60$\pm $4.94 & 65.00$\pm $8.60 & \textbf{65.75$\pm $4.59} \\
GLIOMA& 63.15$\pm $10.99 & 62.55$\pm $17.61  & 59.13$\pm $7.66 & 67.89$\pm $5.47 & 66.27$\pm $10.80 & 67.50$\pm $14.12 & 64.35$\pm $2.39 & \textbf{68.66$\pm $12.24} \\
\bottomrule
\end{tabular}}
\end{center}
\end{table*}

In order to further discuss whether the performance of the proposed SPLR and the compared algorithms is significantly different, a statistical test is carried out. Specifically, given the fact that the overall distribution of samples is unknown, Wilcoxon signed-rank test is deployed \cite{1945Individual}. The null hypothesis is set as “there is no significant difference between SPLR and the compared algorithm”, and the alternative hypothesis is set as “SPLR is better than the compared algorithm”. It is worth noting that the alternative hypothesis should be “the compared algorithm is better than SPLR” when SPLR is compared with MCFS and DISR on Umist and SPLR is compared with UDFS on warpPIE10P. The reason is that the ACC and NMI of the corresponding compared algorithm under these circumstances are higher than SPLR. Given the significance level ${\alpha  = 0.05}$, the results are shown in Table \ref{tab::table7} and Table \ref{tab::table8} in terms of ACC and NMI using K-means for clustering, respectively.

\begin{table*}[htbp]
\centering
\begin{center}
\caption{Wilcoxon signed-rank test in terms of ACC using K-means for clustering}\label{tab::table7}
\resizebox{\textwidth}{!}{
\begin{tabular}[c]{lllllllllllllll}
\toprule
\multirow{2.5}{*}{Dataset}& \multicolumn{2}{l}{Baseline}& \multicolumn{2}{l}{LS}& \multicolumn{2}{l}{MCFS}& \multicolumn{2}{l}{UDFS}& \multicolumn{2}{l}{DISR}& \multicolumn{2}{l}{RNE}& \multicolumn{2}{l}{SGFS}\\
\cmidrule{2-15}
& $p$ & $h$ & $p$ & $h$ & $p$ & $h$ & $p$ & $h$ & $p$ & $h$ & $p$ & $h$ & $p$ & $h$\\
\midrule
USPS & 1.4013e-04 & \textbf{1} & 0.0304 & \textbf{1} & 1.8901e-04 & \textbf{1} & 5.1672e-04 & \textbf{1} & 0.0010 & \textbf{1} & 5.9342e-04 & \textbf{1} & 0.0100 & \textbf{1} \\
Madelon & 8.8074e-05 & \textbf{1} & 8.8324e-05 & \textbf{1} & 1.3101e-04 & \textbf{1} & 1.0177e-04 & \textbf{1} & 0.0169 & \textbf{1} & 8.8199e-05 & \textbf{1} & 8.8199e-05 & \textbf{1} \\
Isolet & 1.3920e-04 & \textbf{1} & 8.8575e-05 & \textbf{1} & 1.3995e-04 & \textbf{1} & 0.0032 & \textbf{1} & 8.8449e-05 & \textbf{1} & 8.7949e-05 & \textbf{1} & 8.8324e-05 & \textbf{1} \\
Umist & 3.9023e-04 & \textbf{1} & 3.9023e-04 & \textbf{1} & 8.8575e-05 & 1 & 8.9180e-04 & \textbf{1} & 2.1908e-04 & 1 & 8.8575e-05 & \textbf{1} & 8.8575e-05 & \textbf{1} \\
COIL20 & 0.0072 & \textbf{1} & 0.0251 & \textbf{1} & 2.1908e-04 & \textbf{1} & 0.3703 & 0 & 0.0032 & \textbf{1} & 0.0019 & \textbf{1} & 0.0400 & \textbf{1} \\
ORL & 8.8324e-05 & \textbf{1} & 8.8449e-05 & \textbf{1} & 8.8575e-05 & \textbf{1} & 8.8575e-05 & \textbf{1} & 8.8575e-05 & \textbf{1} & 8.8324e-05 & \textbf{1} & 8.8449e-05 & \textbf{1} \\
Colon & 3.3217e-04 & \textbf{1} & 0.0029 & \textbf{1} & 0.0014 & \textbf{1} & 0.0125 & \textbf{1} & 0.0532 & 0 & 0.0107 & \textbf{1} & 0.0220 & \textbf{1} \\
warpPIE10P & 8.8575e-05 & \textbf{1} & 8.9180e-04 & \textbf{1} & 0.0100 & \textbf{1} & 8.8575e-05 & 1 & 8.8575e-05 & \textbf{1} & 8.8575e-05 & \textbf{1} & 0.0010 & \textbf{1} \\
GLIOMA & 0.0022 & \textbf{1} & 6.7694e-04 & \textbf{1} & 6.2407e-04 & \textbf{1} & 0.0642 & 0 & 1.3995e-04 & \textbf{1} & 0.0965 & 0 & 0.0036 & \textbf{1} \\
\bottomrule
\end{tabular}}
\end{center}
\end{table*}

\begin{table*}[htbp]
\centering
\begin{center}
\caption{Wilcoxon signed-rank test in terms of NMI using K-means for clustering}\label{tab::table8}
\resizebox{\textwidth}{!}{
\begin{tabular}[c]{lllllllllllllll}
\toprule
\multirow{2.5}{*}{Dataset}& \multicolumn{2}{l}{Baseline}& \multicolumn{2}{l}{LS}& \multicolumn{2}{l}{MCFS}& \multicolumn{2}{l}{UDFS}& \multicolumn{2}{l}{DISR}& \multicolumn{2}{l}{RNE}& \multicolumn{2}{l}{SGFS}\\
\cmidrule{2-15}
& $p$ & $h$ & $p$ & $h$ & $p$ & $h$ & $p$ & $h$ & $p$ & $h$ & $p$ & $h$ & $p$ & $h$\\
\midrule
USPS & 0.0013 & \textbf{1} & 6.8061e-04 & \textbf{1} & 8.8575e-05 & \textbf{1} & 8.8575e-05 & \textbf{1} & 5.9342e-04 & \textbf{1} & 1.4013e-04 & \textbf{1} & 0.0276 & \textbf{1} \\
Madelon & 8.8575e-05 & \textbf{1} & 8.8575e-05 & \textbf{1} & 8.8575e-05 & \textbf{1} & 8.8575e-05 & \textbf{1} & 0.0019 & \textbf{1} & 8.8575e-05 & \textbf{1} & 8.8575e-05 & \textbf{1} \\
Isolet & 8.8575e-05 & \textbf{1} & 8.8575e-05 & \textbf{1} & 8.8575e-05 & \textbf{1} & 1.8901e-04 & \textbf{1} & 8.8575e-05 & \textbf{1} & 8.8575e-05 & \textbf{1} & 8.8575e-05 & \textbf{1} \\
Umist & 2.1908e-04 & \textbf{1} & 0.2043 & 0 & 8.8575e-05 & 1 & 0.0072 & \textbf{1} & 8.8575e-05 & 1 & 8.8575e-05 & \textbf{1} & 8.8575e-05 & \textbf{1} \\
COIL20 & 8.8575e-05 & \textbf{1} & 3.3845e-04 & \textbf{1} & 4.4934e-04 & \textbf{1} & 0.9405 & 0 & 8.8575e-05 & \textbf{1} & 8.9180e-04 & \textbf{1} & 0.5016 & 0 \\
ORL & 8.8575e-05 & \textbf{1} & 8.8575e-05 & \textbf{1} & 8.8575e-05 & \textbf{1} & 8.8575e-05 & \textbf{1} & 8.8575e-05 & \textbf{1} & 8.8575e-05 & \textbf{1} & 8.8575e-05 & \textbf{1} \\
Colon & 2.5360e-04 & \textbf{1} & 6.8061e-04 & \textbf{1} & 0.0032 & \textbf{1} & 0.0028 & \textbf{1} & 0.2959 & 0 & 1.0335e-04 & \textbf{1} & 8.9180e-04 & \textbf{1} \\
warpPIE10P & 8.8575e-05 & \textbf{1} & 8.8575e-05 & \textbf{1} & 0.0090 & \textbf{1} & 8.8575e-05 & 1 & 8.8575e-05 & \textbf{1} & 8.8575e-05 & \textbf{1} & 8.8575e-05 & \textbf{1} \\
GLIOMA & 7.7959e-04 & \textbf{1} & 0.0012 & \textbf{1} & 0.0051 & \textbf{1} & 0.3905 & 0 & 0.1354 & 0 & 0.0137 & \textbf{1} & 0.0057 & \textbf{1} \\
\bottomrule
\end{tabular}}
\end{center}
\end{table*}

As can be seen, in most cases, SPLR outperforms the other algorithm. As for ACC, SPLR stands out on datasets including USPS, Madelon, Isolet and ORL. On Umist, SPLR is inferior to MCFS and DISR. On COIL20, the performance is not significantly improved compared with UDFS. On Colon, there is no significant difference between SPLR and DISR. On warpPIE10P, UDFS is better than SPLR. And on GLIOMA, the performance of UDFS and RNE is comparable to SPLR. The result in terms of NMI is similar to that in terms of ACC, which further proves the superiority of SPLR.

Figs. \ref{fig::picture1} and \ref{fig::picture2} illuminate the clustering results with different number of selected features in terms of ACC and NMI.

From Fig. \ref{fig::picture1}, it is clear that SPLR achieves the best result on Madelon no matter how many features are selected. On datasets including USPS, COIL20, ORL, Colon and GLIOMA, SPLR outperforms other methods in most cases. On the remaining three datasets, SPLR is inferior to several approaches.

\begin{figure*}[htbp]  
\centering
\subfigure[USPS]{
\begin{minipage}{0.31\linewidth}
\centering
  \includegraphics[width=\textwidth]{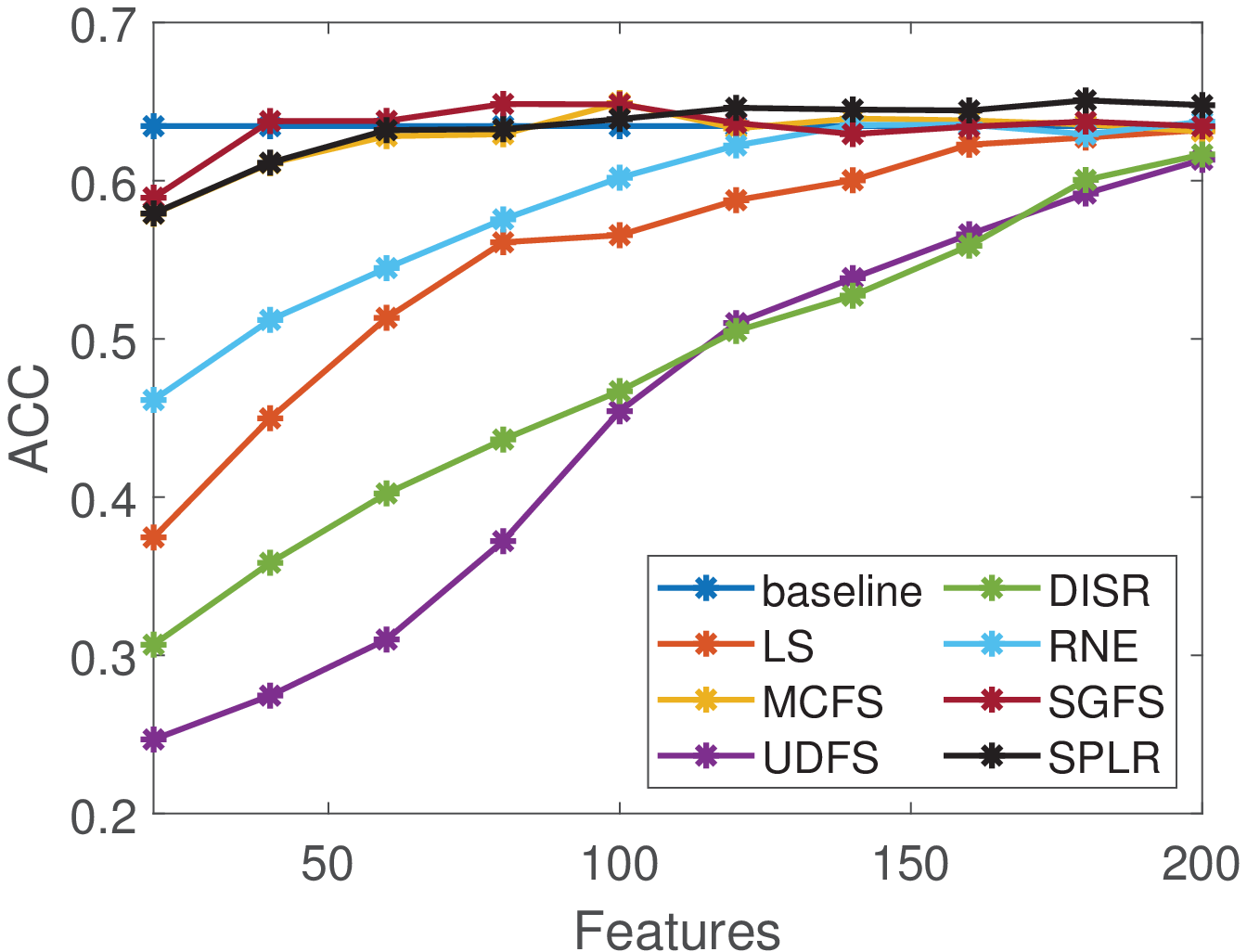}
\end{minipage}
}
\subfigure[Madelon]{
\begin{minipage}{0.31\linewidth}
\centering
  \includegraphics[width=\textwidth]{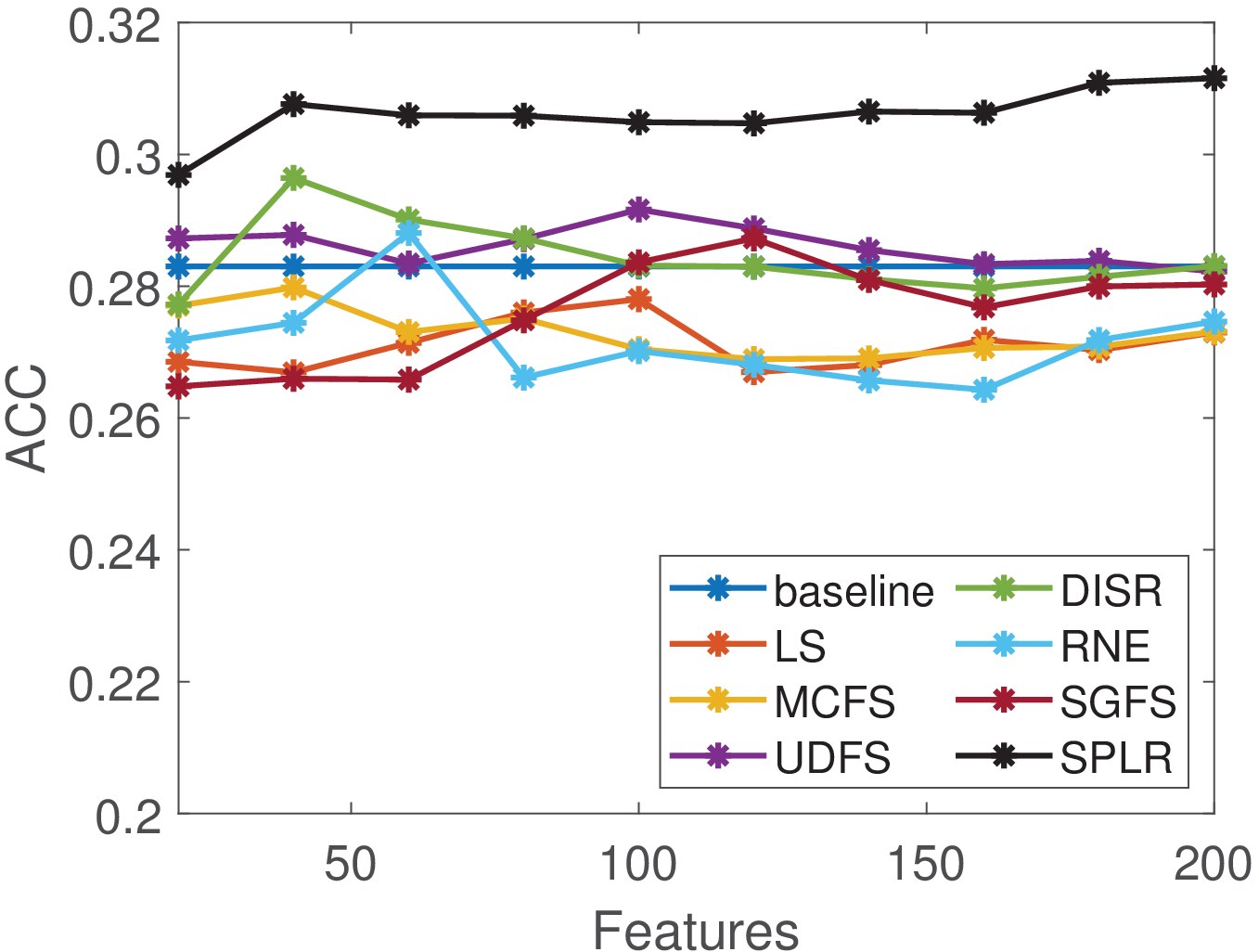}
\end{minipage}
}
\subfigure[Isolet]{
\begin{minipage}{0.31\linewidth}
\centering
  \includegraphics[width=\textwidth]{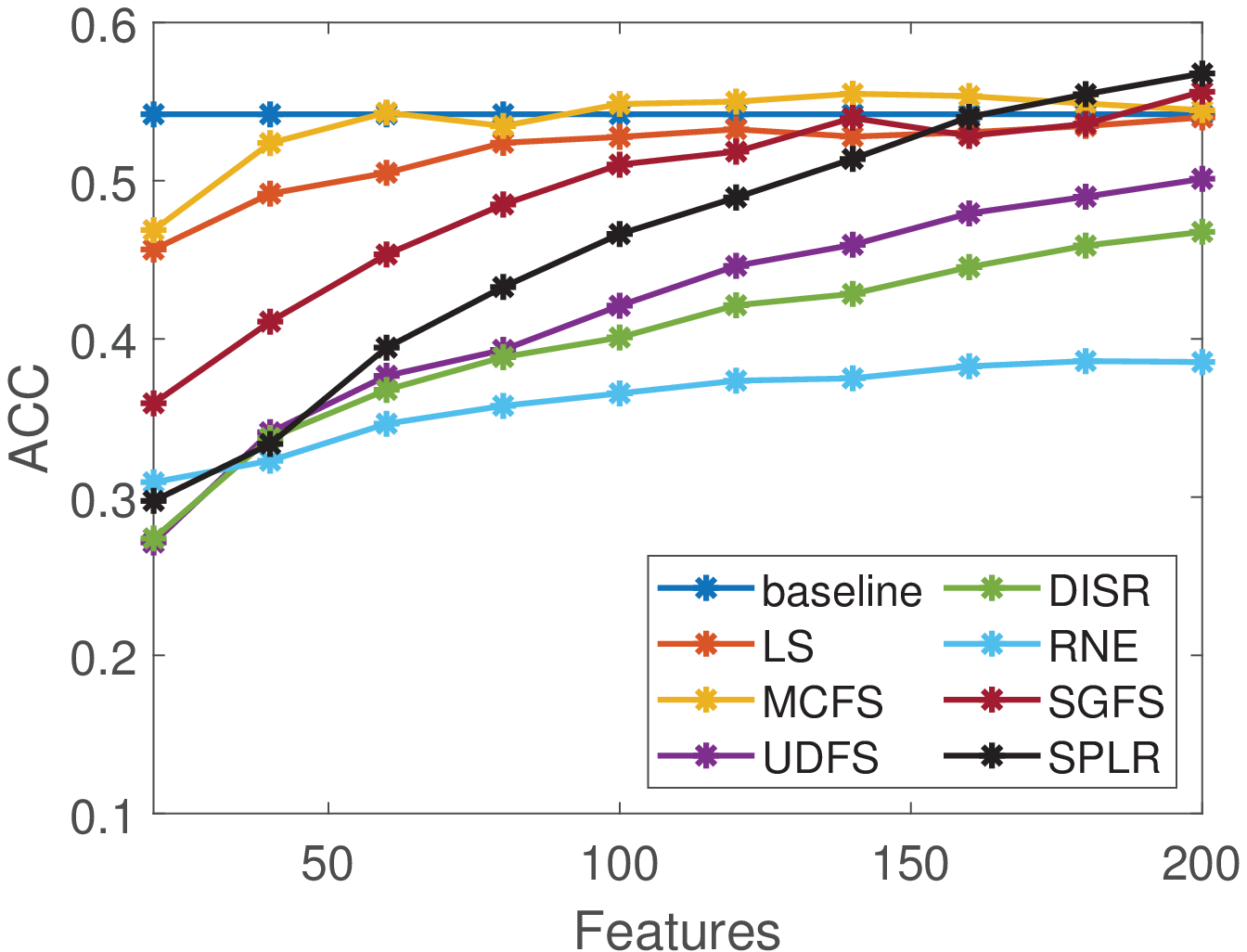}
\end{minipage}
}
%
\subfigure[Umist]{
\begin{minipage}{0.31\linewidth}
\centering
  \includegraphics[width=\textwidth]{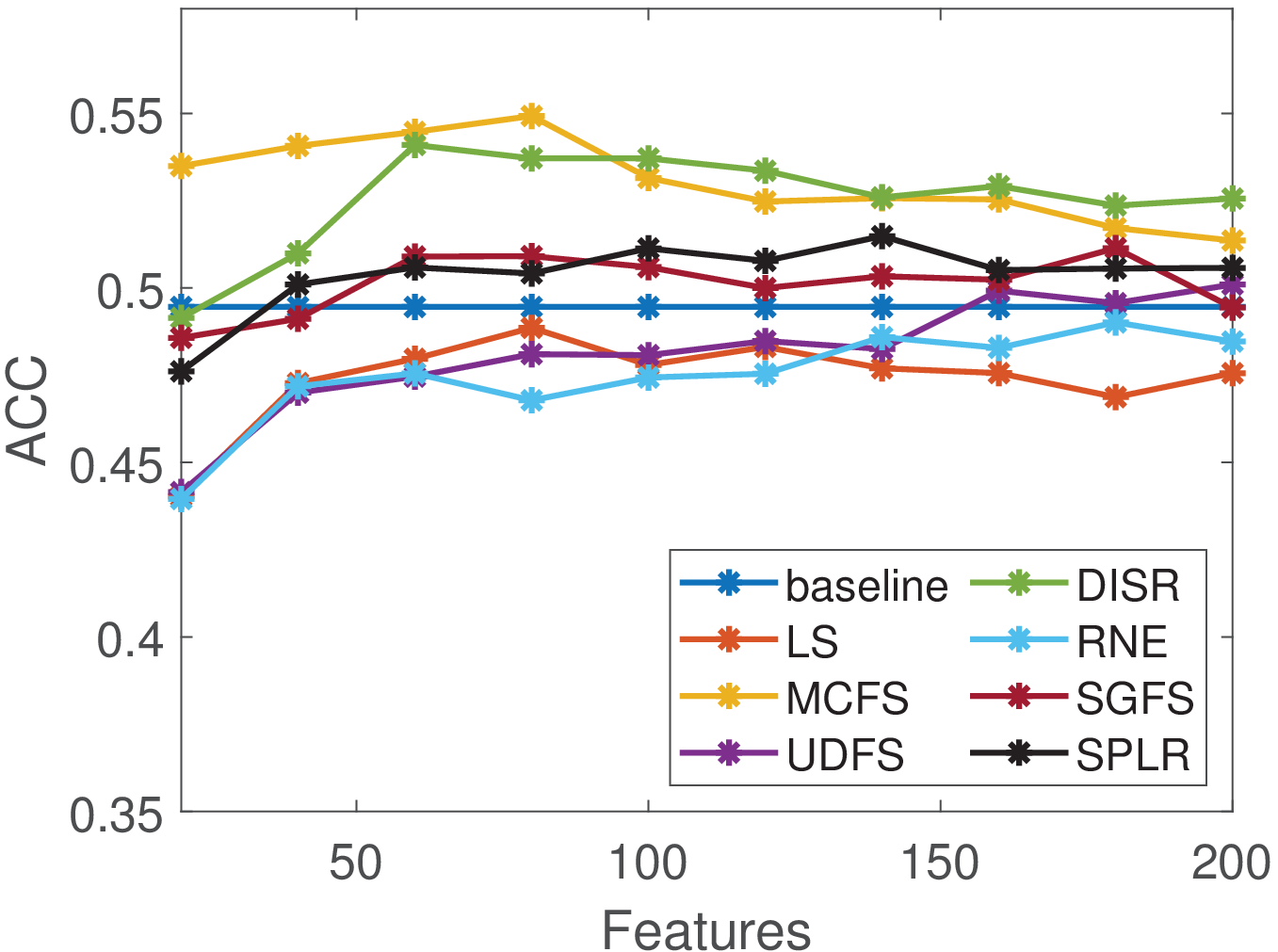}
\end{minipage}
}
\subfigure[COIL20]{
\begin{minipage}{0.31\linewidth}
\centering
  \includegraphics[width=\textwidth]{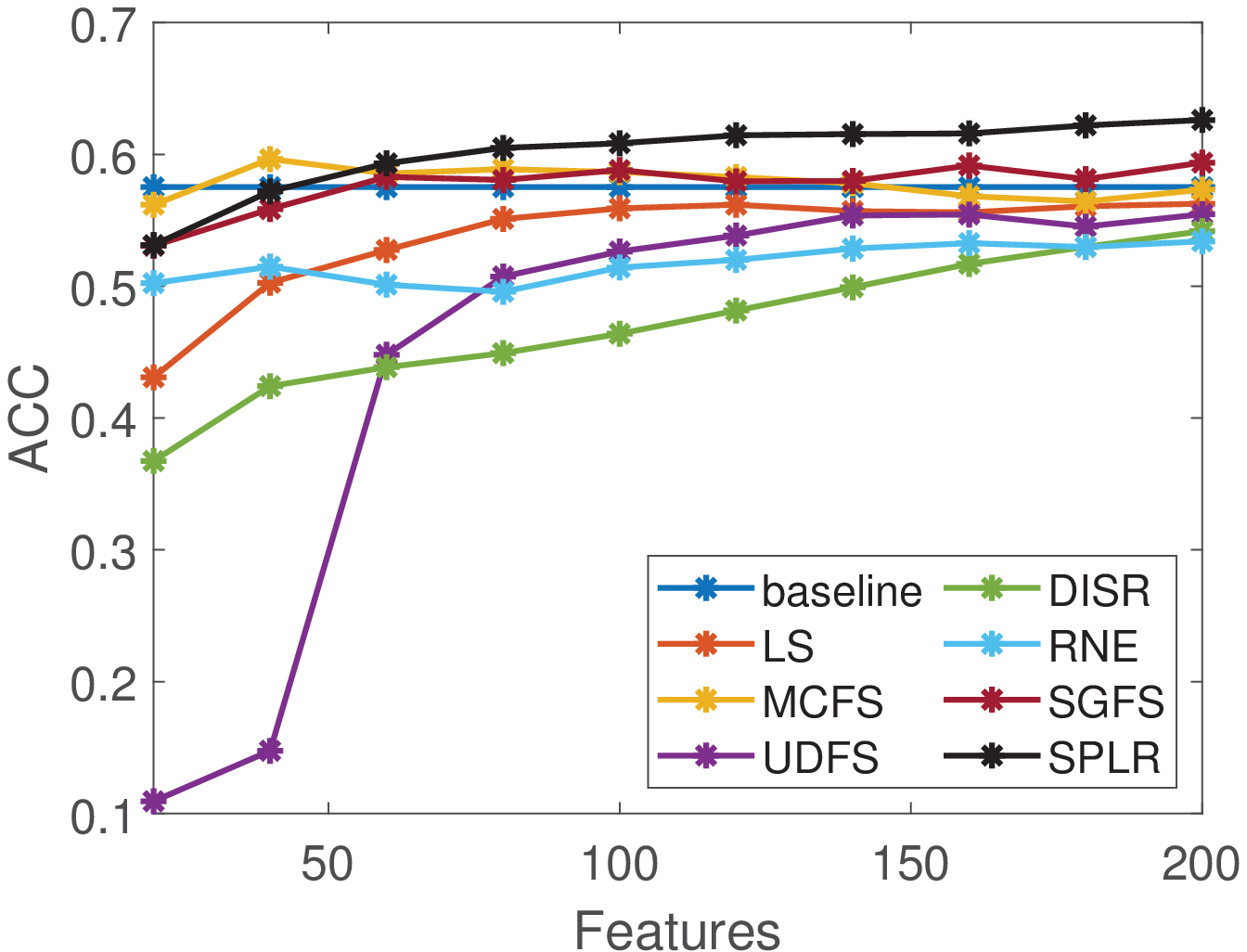}
\end{minipage}
}
\subfigure[ORL]{
\begin{minipage}{0.31\linewidth}
\centering
  \includegraphics[width=\textwidth]{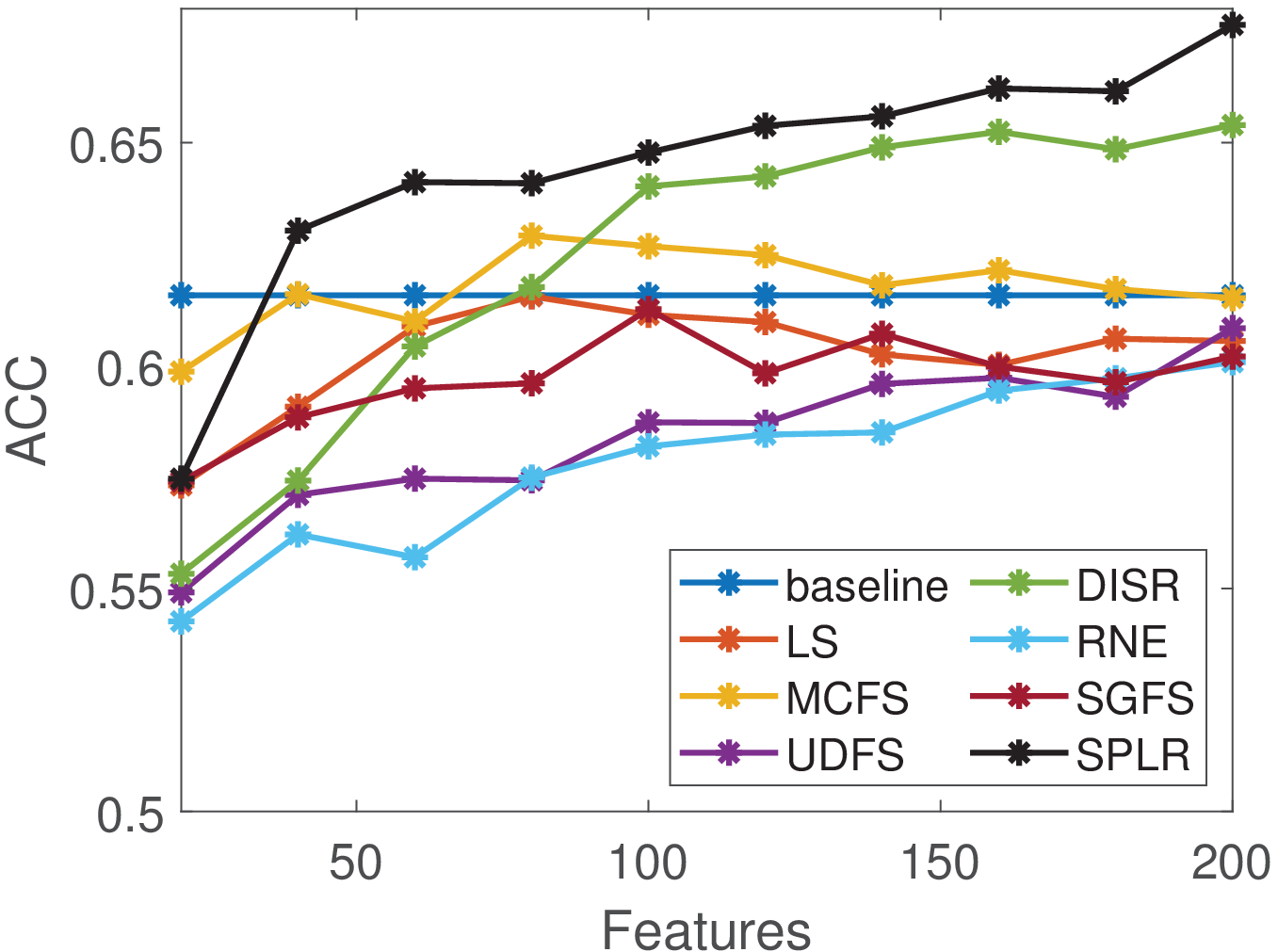}
\end{minipage}
}

\subfigure[Colon]{
\begin{minipage}{0.31\linewidth}
\centering
  \includegraphics[width=\textwidth]{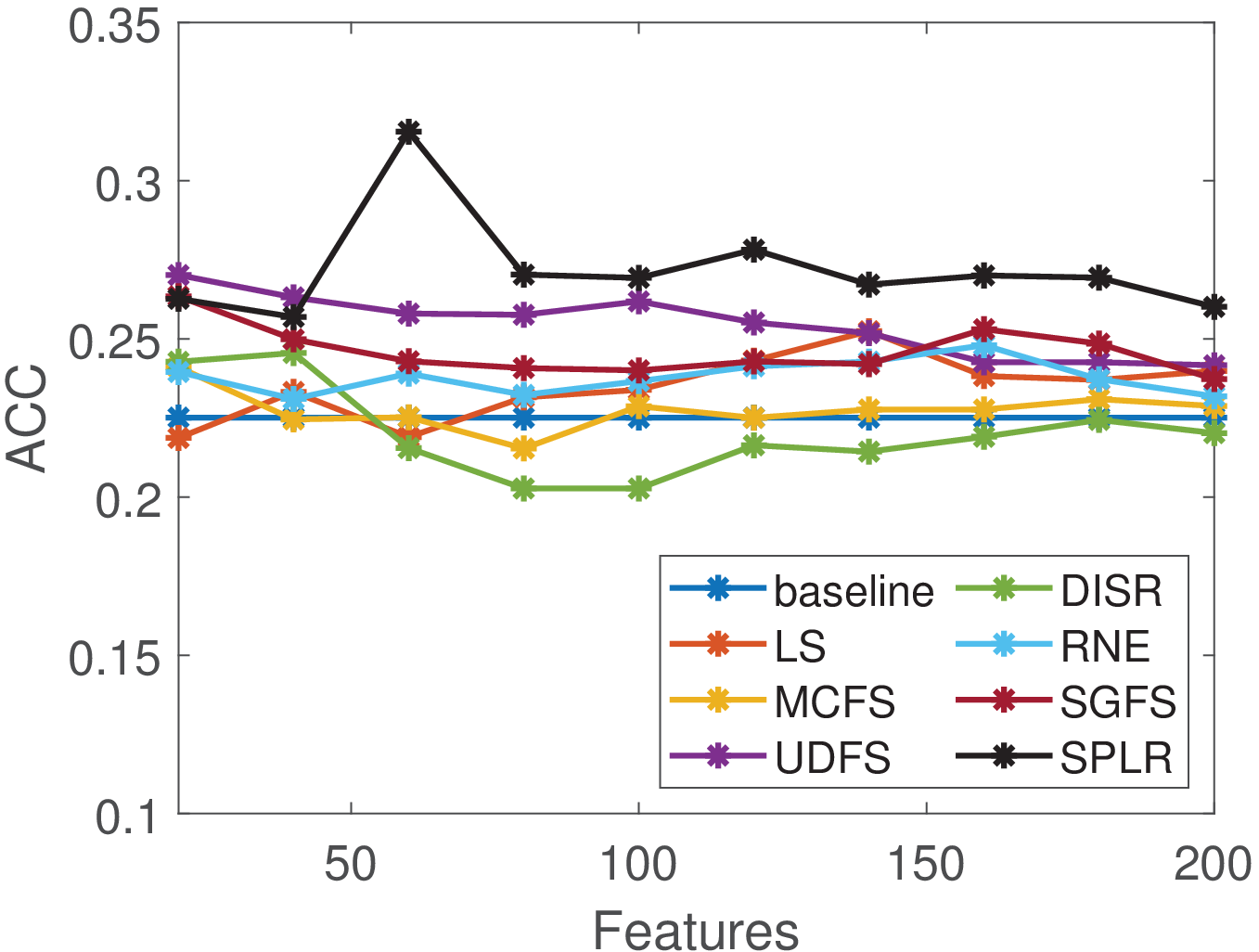}
\end{minipage}
}
\subfigure[warpPIE10P]{
\begin{minipage}{0.31\linewidth}
\centering
  \includegraphics[width=\textwidth]{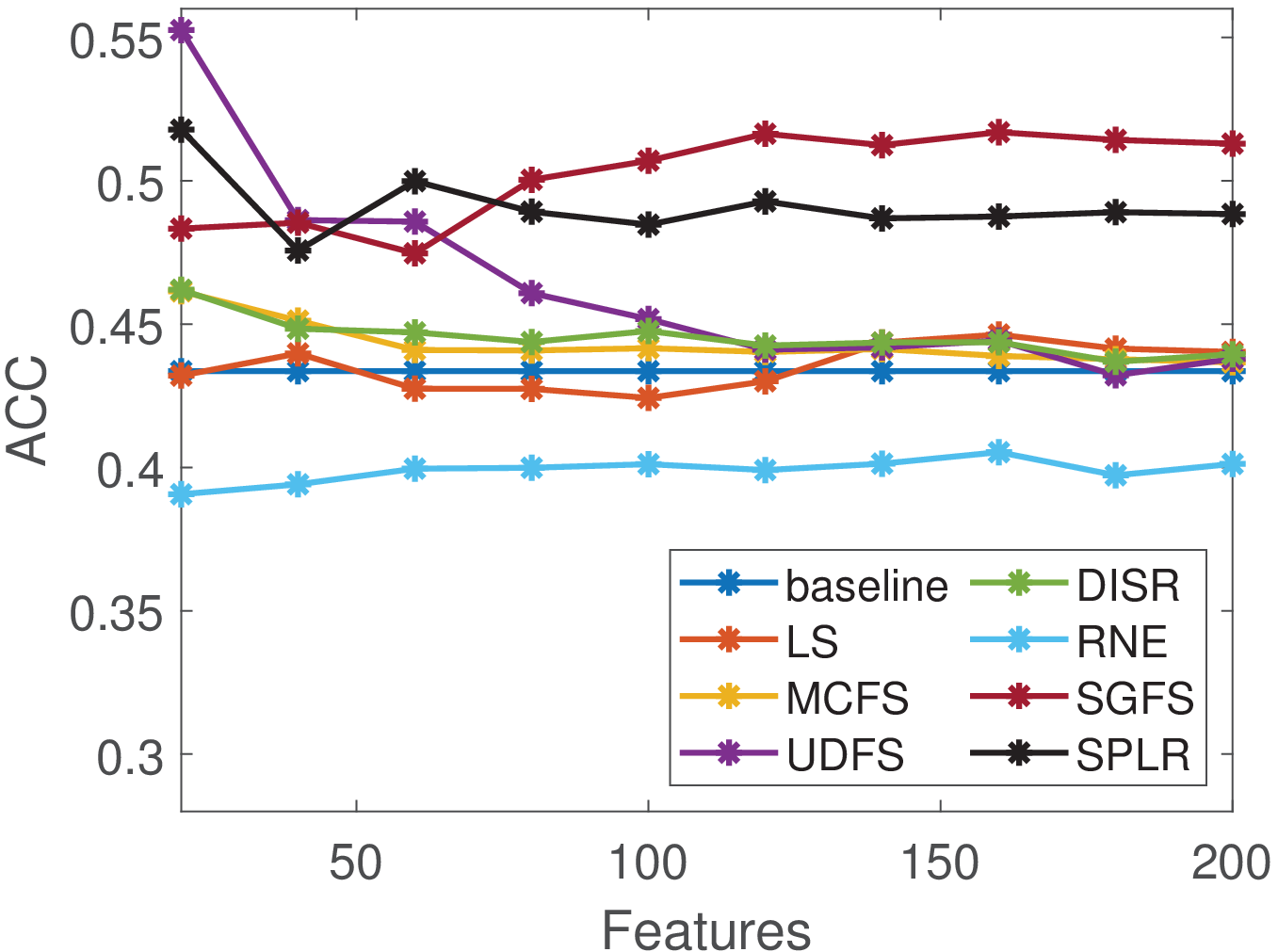}
\end{minipage}
}
\subfigure[CLIOMA]{
\begin{minipage}{0.31\linewidth}
\centering
  \includegraphics[width=\textwidth]{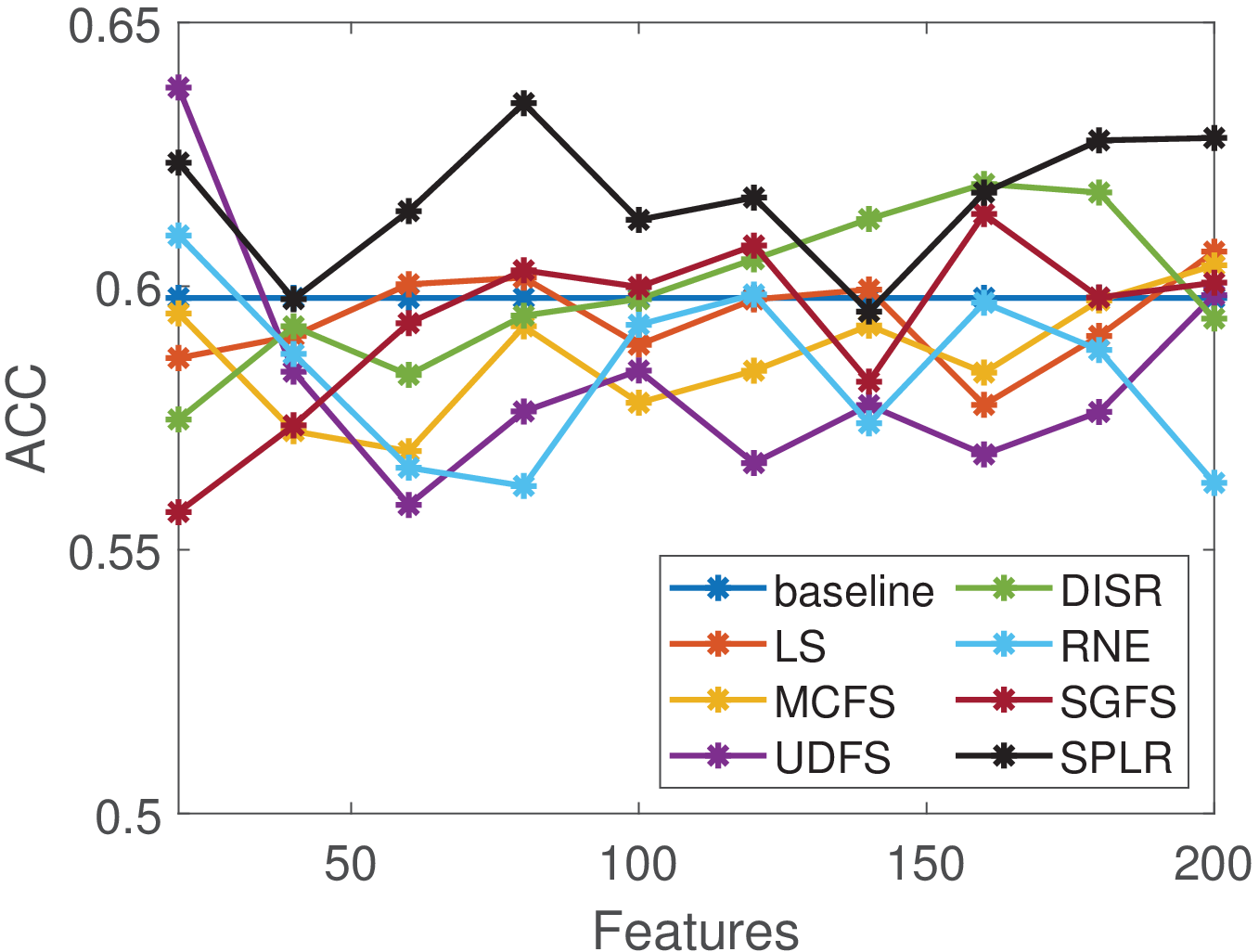}
\end{minipage}
}
\centering
\caption{The clustering results with different number of selected features in terms of ACC}\label{fig::picture1}
\end{figure*}

\begin{figure*}[htbp]  
\centering
\subfigure[USPS]{
\begin{minipage}{0.31\linewidth}
\centering
  \includegraphics[width=\textwidth]{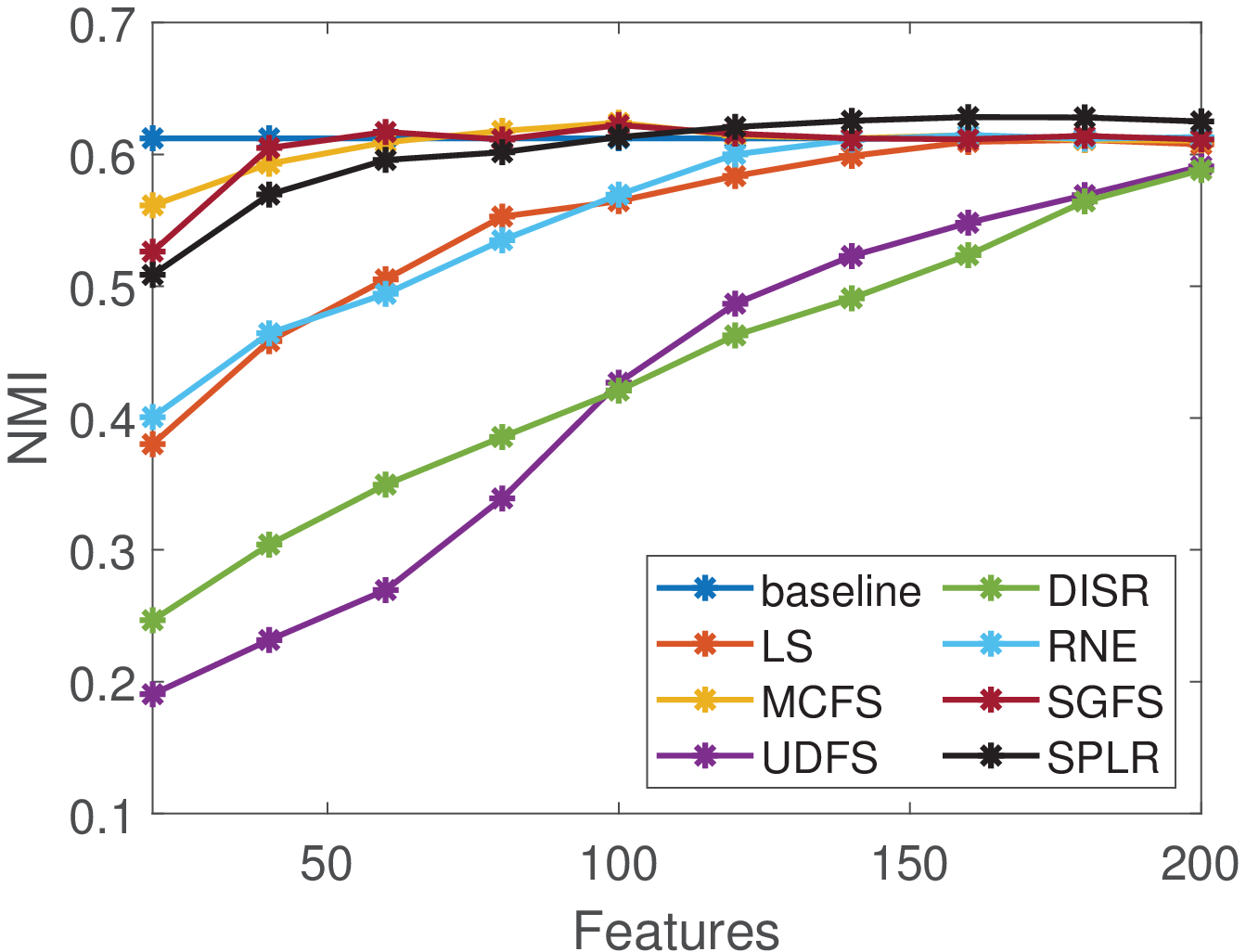}
\end{minipage}
}
\subfigure[Madelon]{
\begin{minipage}{0.31\linewidth}
\centering
  \includegraphics[width=\textwidth]{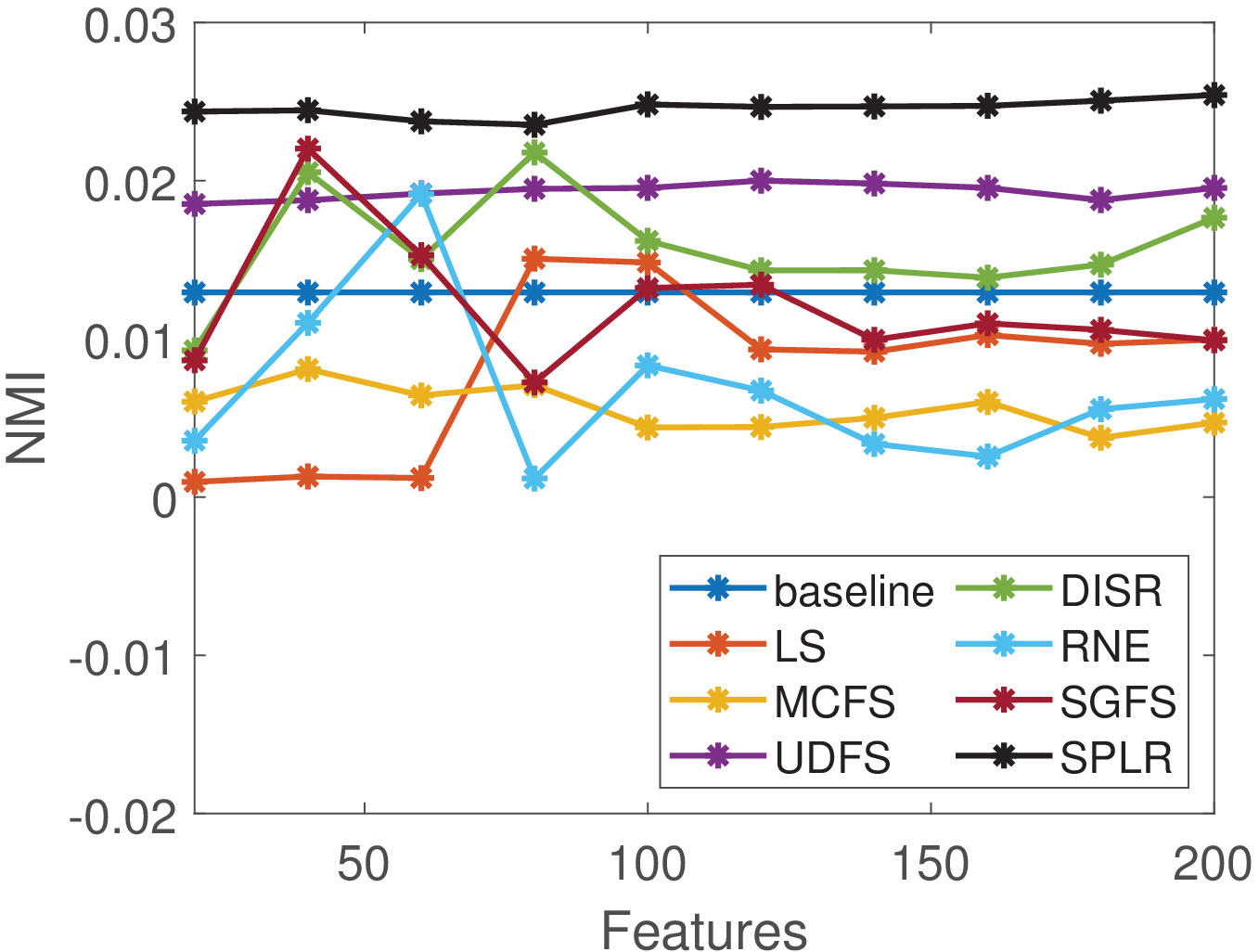}
\end{minipage}
}
\subfigure[Isolet]{
\begin{minipage}{0.31\linewidth}
\centering
  \includegraphics[width=\textwidth]{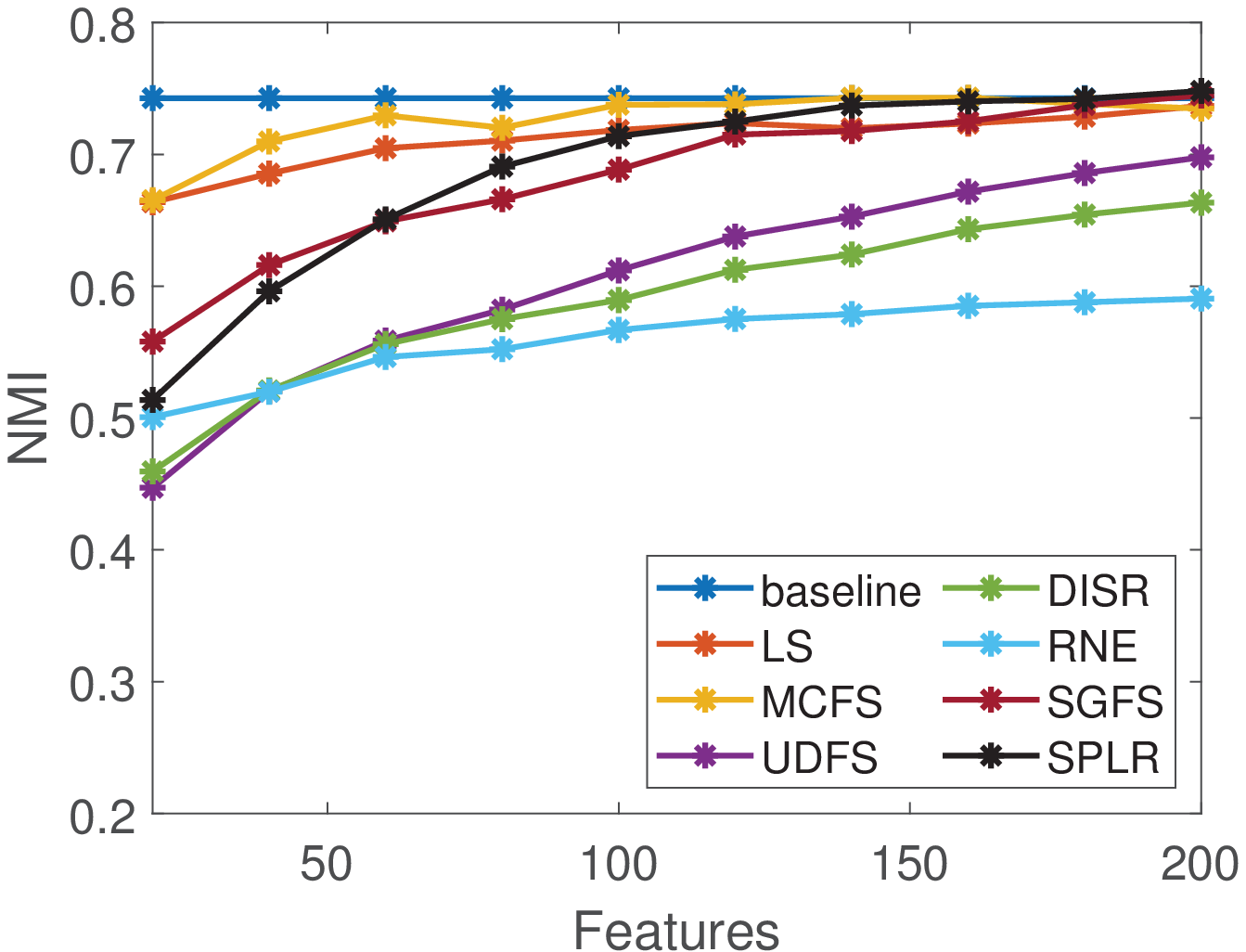}
\end{minipage}
}
%
\subfigure[Umist]{
\begin{minipage}{0.31\linewidth}
\centering
  \includegraphics[width=\textwidth]{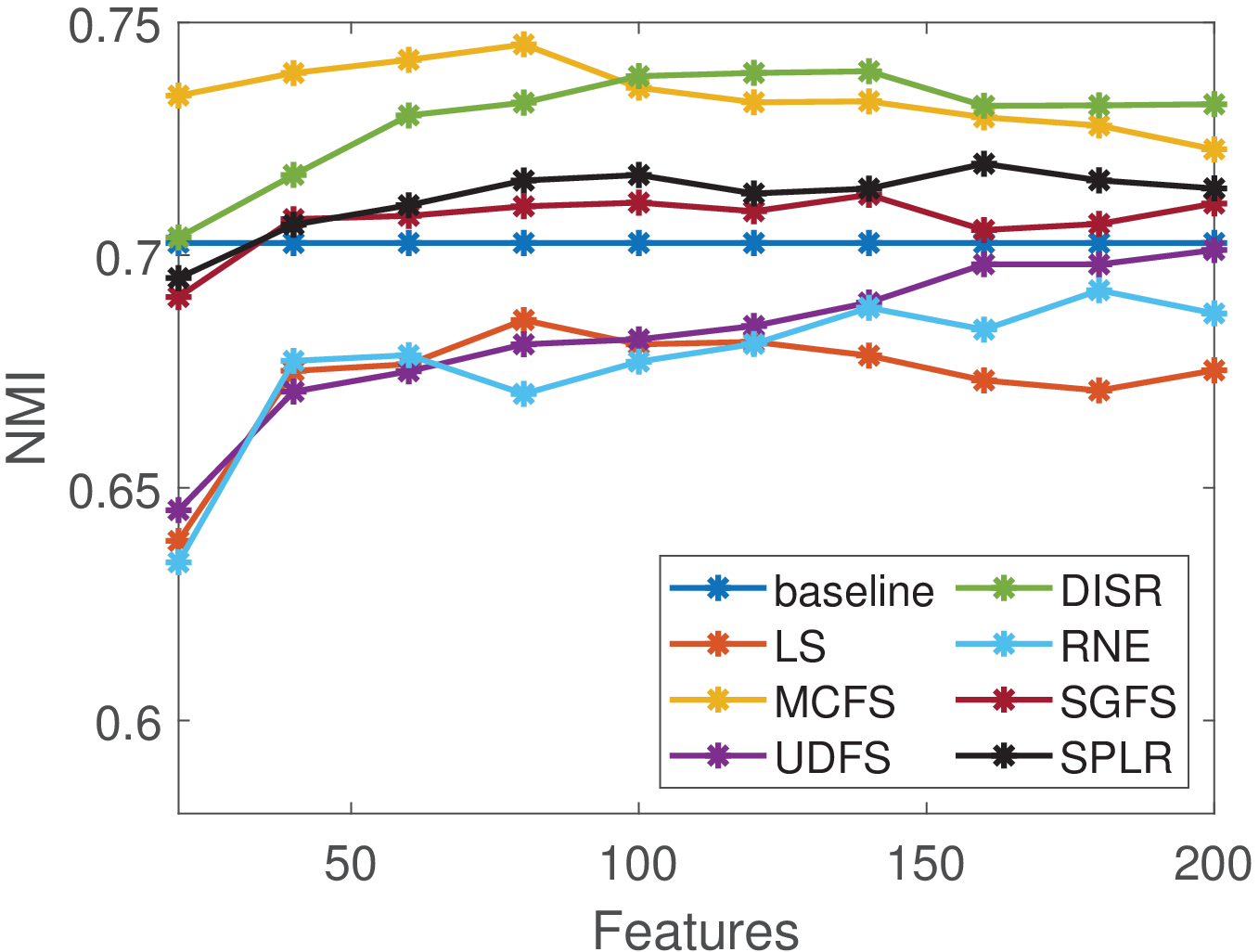}
\end{minipage}
}
\subfigure[COIL20]{
\begin{minipage}{0.31\linewidth}
\centering
  \includegraphics[width=\textwidth]{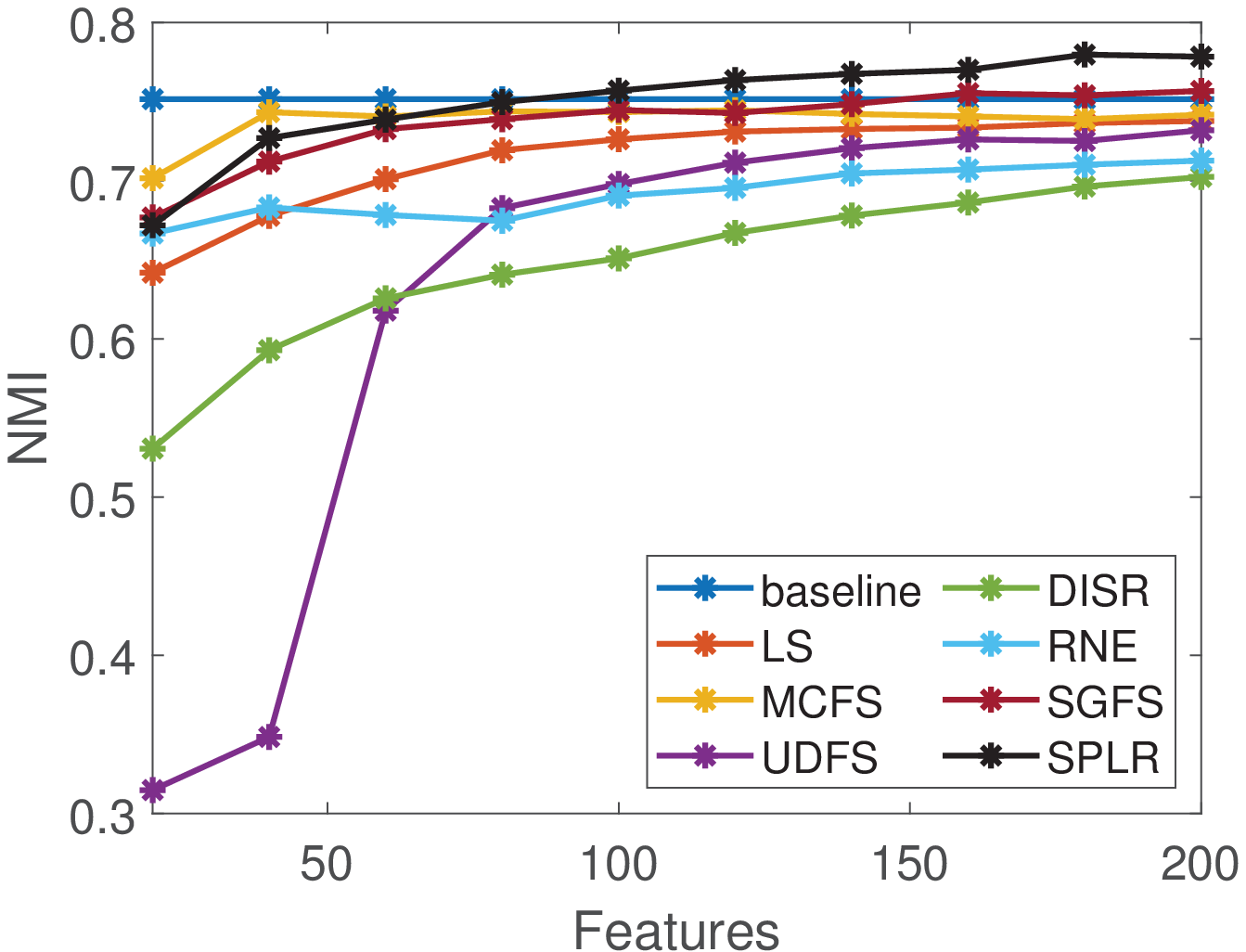}
\end{minipage}
}
\subfigure[ORL]{
\begin{minipage}{0.31\linewidth}
\centering
  \includegraphics[width=\textwidth]{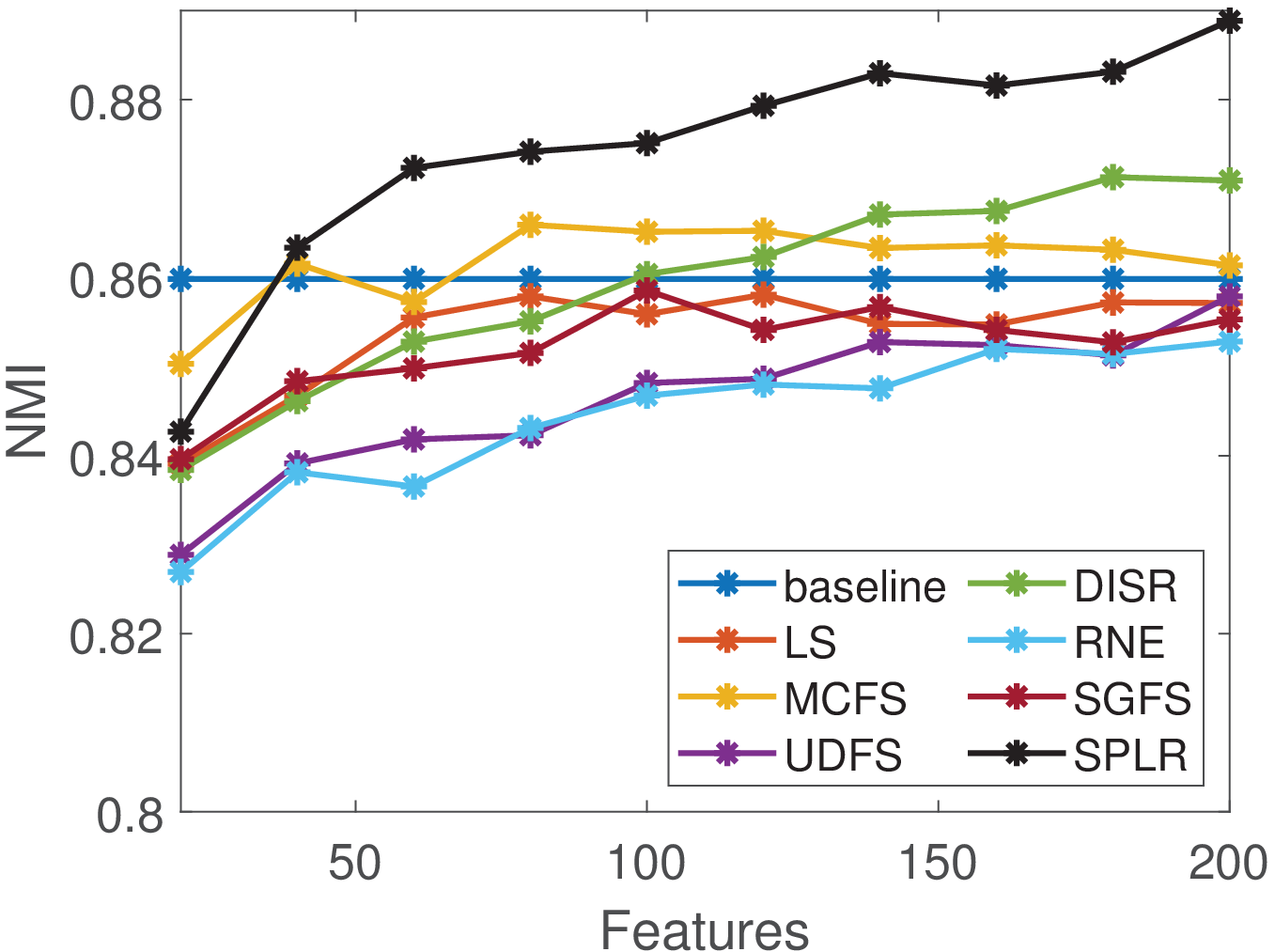}
\end{minipage}
}
%
\subfigure[Colon]{
\begin{minipage}{0.31\linewidth}
\centering
  \includegraphics[width=\textwidth]{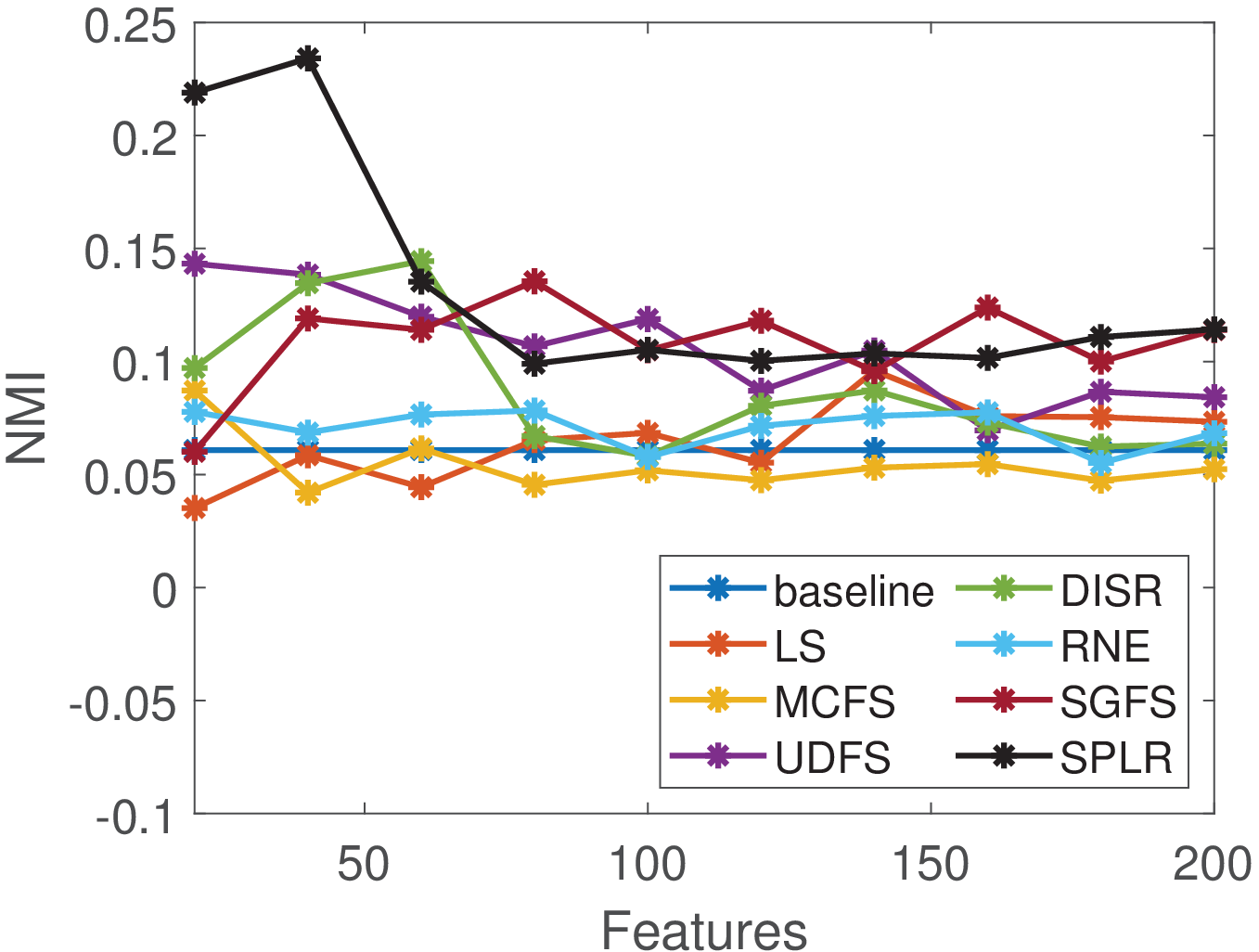}
\end{minipage}
}
\subfigure[warpPIE10P]{
\begin{minipage}{0.31\linewidth}
\centering
  \includegraphics[width=\textwidth]{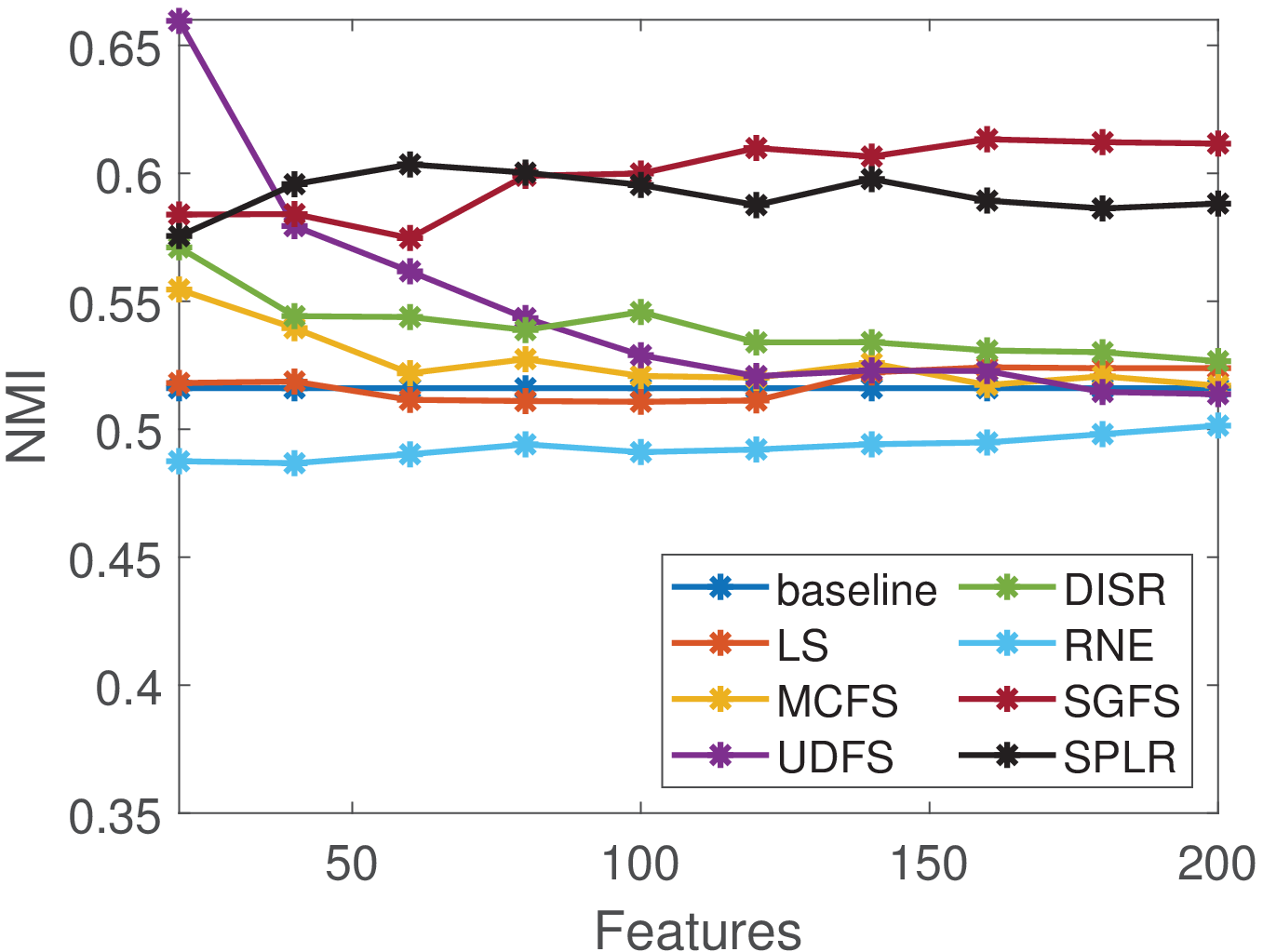}
\end{minipage}
}
\subfigure[CLIOMA]{
\begin{minipage}{0.31\linewidth}
\centering
  \includegraphics[width=\textwidth]{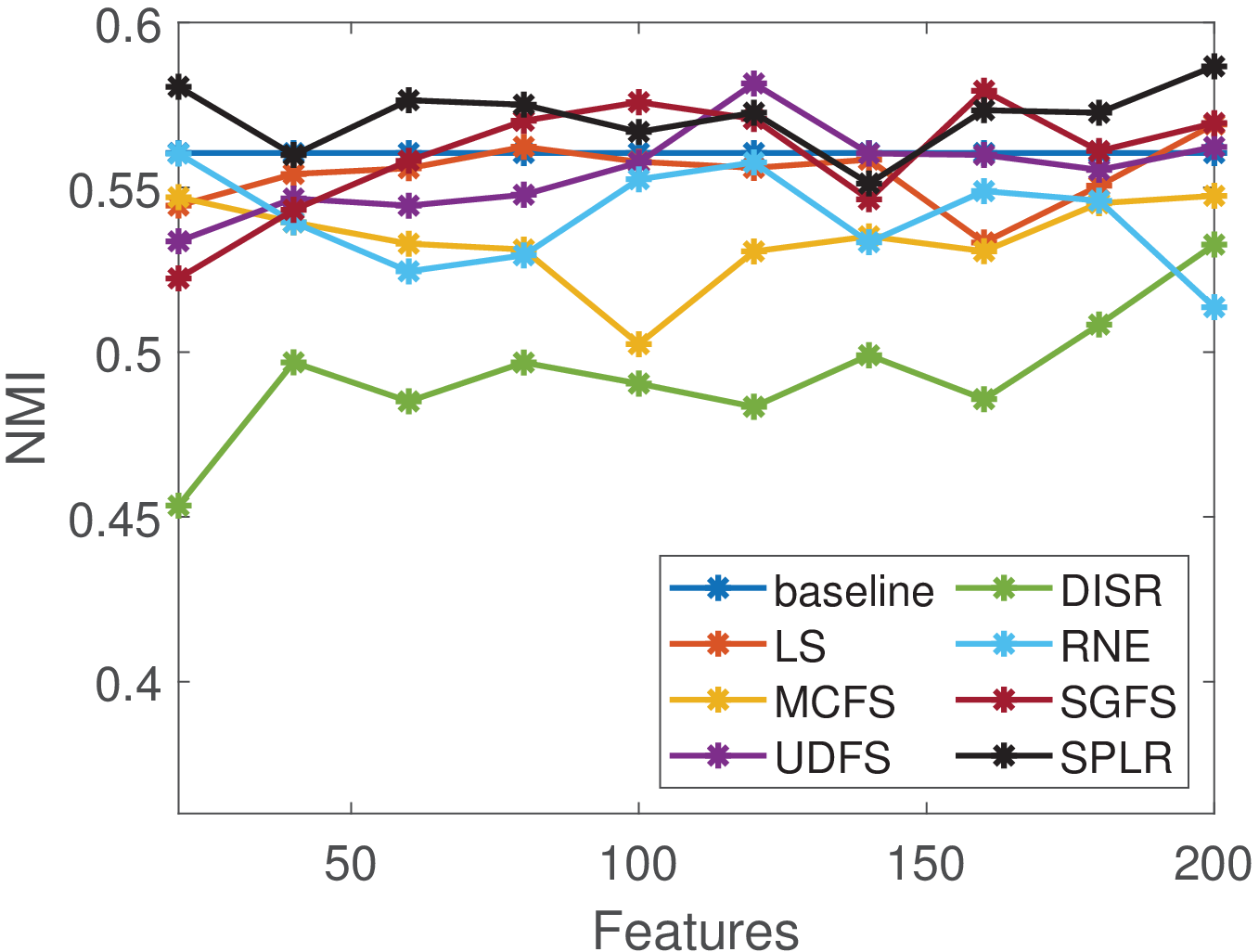}
\end{minipage}
}
\centering
\caption{The clustering results with different number of selected features in terms of NMI}\label{fig::picture2}
\end{figure*}

Similarly, Fig. \ref{fig::picture2} elucidates that SPLR achieves the best result on Madelon in all cases. On datasets including USPS, COIL20, ORL and GLIOMA, SPLR is superior to other algorithms under most circumstances. On the remaining four datasets, there are algorithms that perform better than SPLR in some cases. As a consequence, it can be concluded that SPLR has advantages over other approaches.

\subsection{Convergence analysis}
In this subsection, the convergence of SPLR on different datasets is empirically discussed, which has been proved in theory in Section~\ref{sec:Method}. The objective function is the same as Eq. \eqref{eqn::eq10}, and the variation of the corresponding values with the increase of number of iterations is depicted in Fig. \ref{fig::picture3}.

\begin{figure*}[htbp]  
\centering
\subfigure[USPS]{
\begin{minipage}{0.31\linewidth}
\centering
  \includegraphics[width=\textwidth]{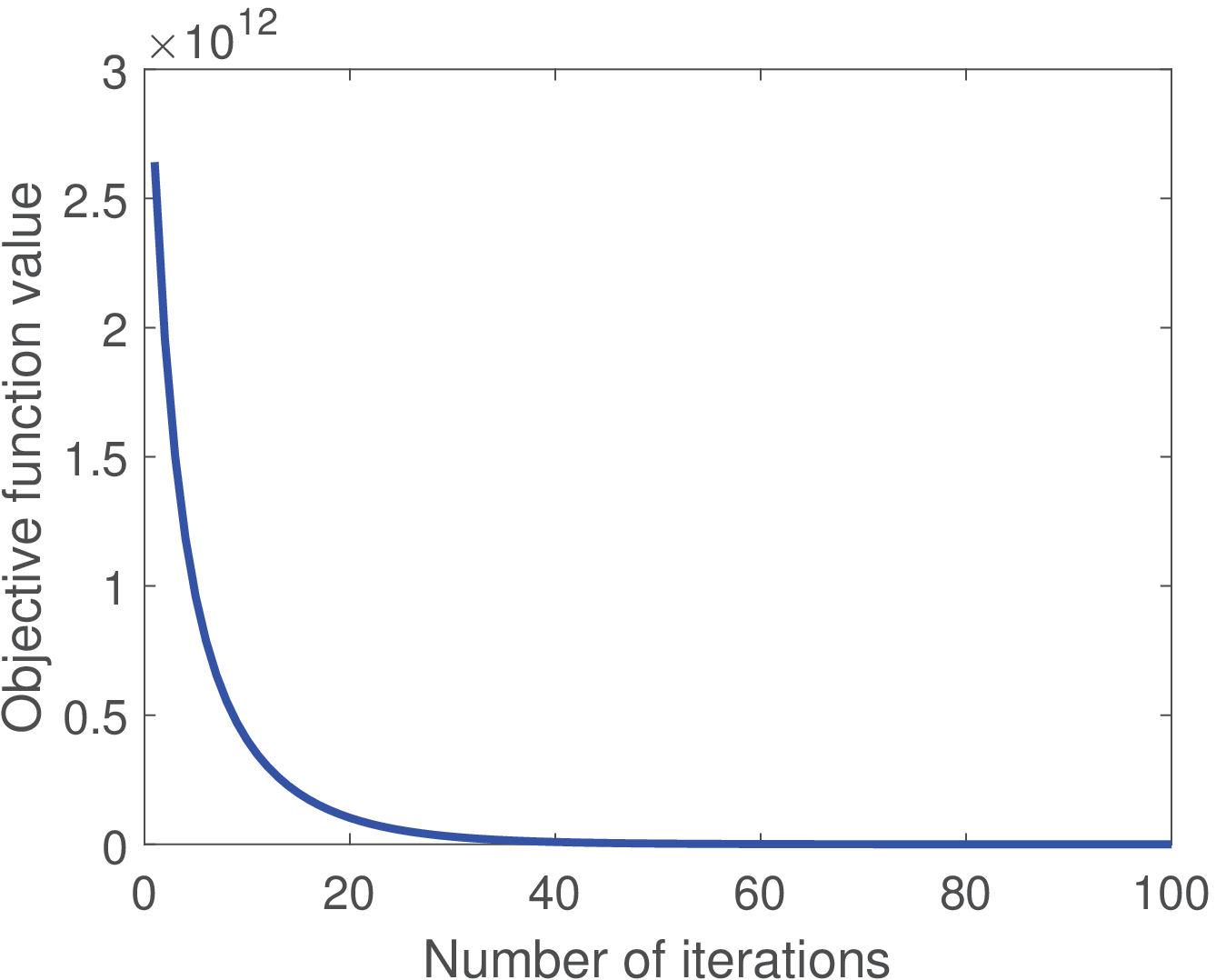}
\end{minipage}
}
\subfigure[Madelon]{
\begin{minipage}{0.31\linewidth}
\centering
  \includegraphics[width=\textwidth]{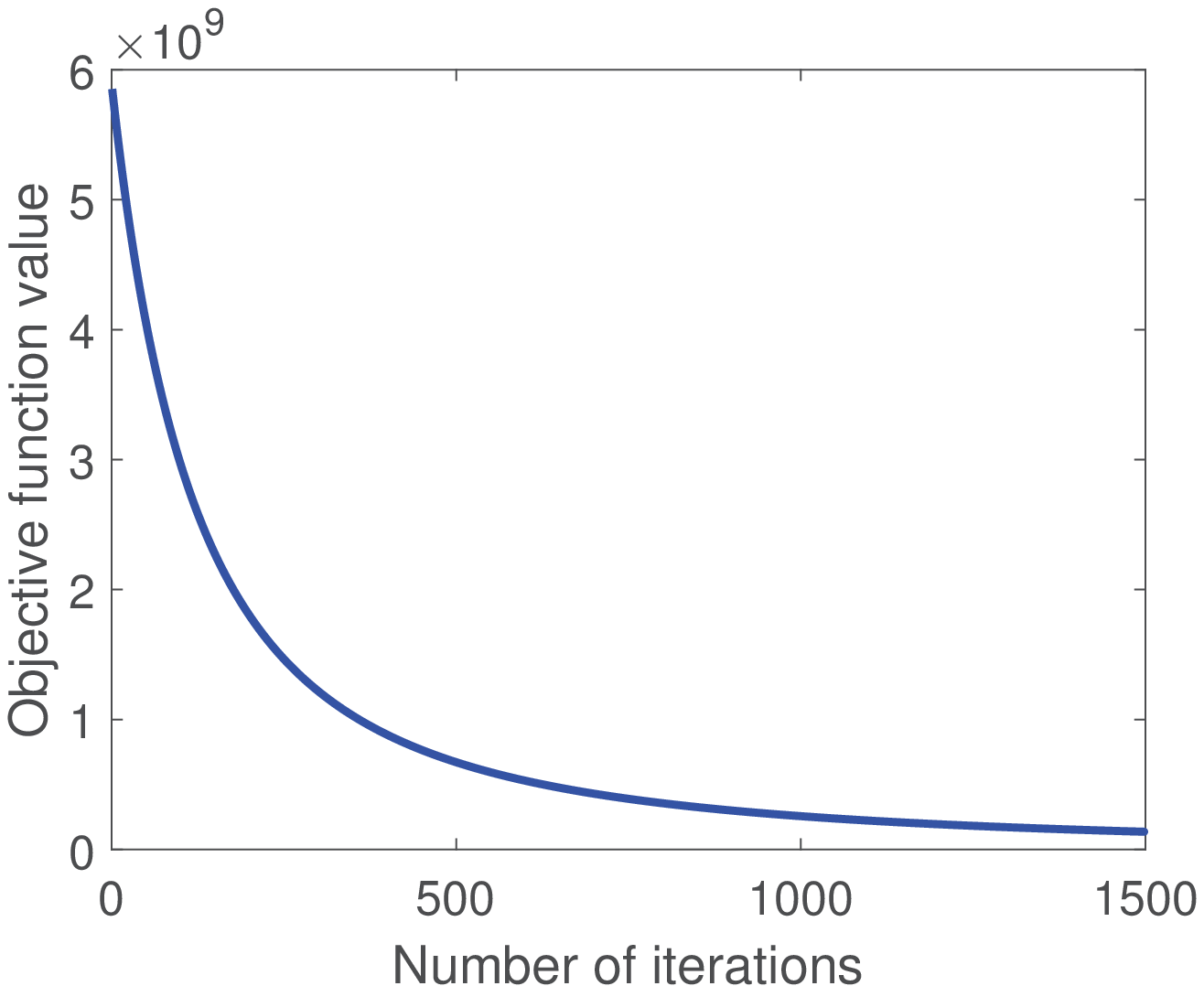}
\end{minipage}
}
\subfigure[Isolet]{
\begin{minipage}{0.31\linewidth}
\centering
  \includegraphics[width=\textwidth]{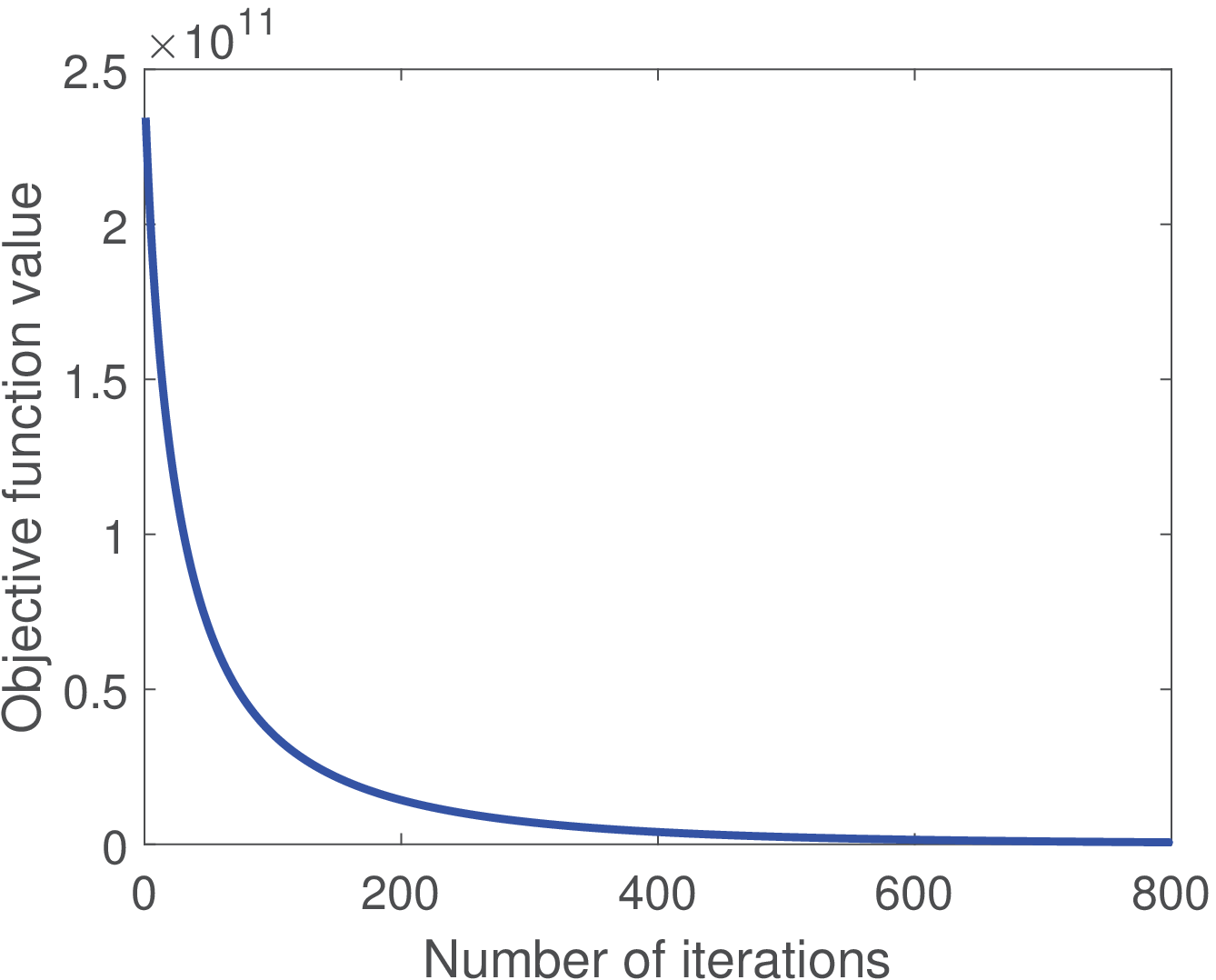}
\end{minipage}
}
%
\subfigure[Umist]{
\begin{minipage}{0.31\linewidth}
\centering
  \includegraphics[width=\textwidth]{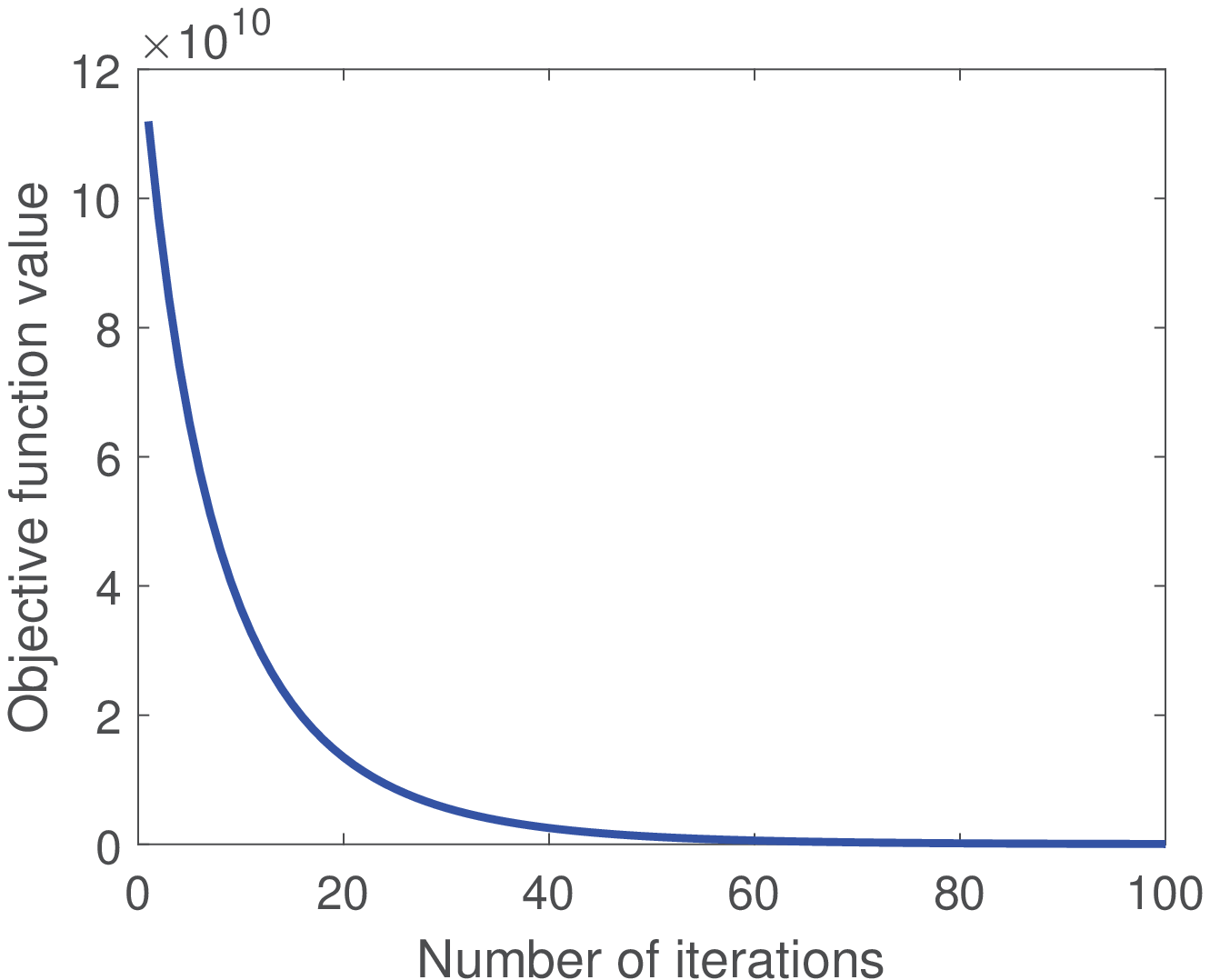}
\end{minipage}
}
\subfigure[COIL20]{
\begin{minipage}{0.31\linewidth}
\centering
  \includegraphics[width=\textwidth]{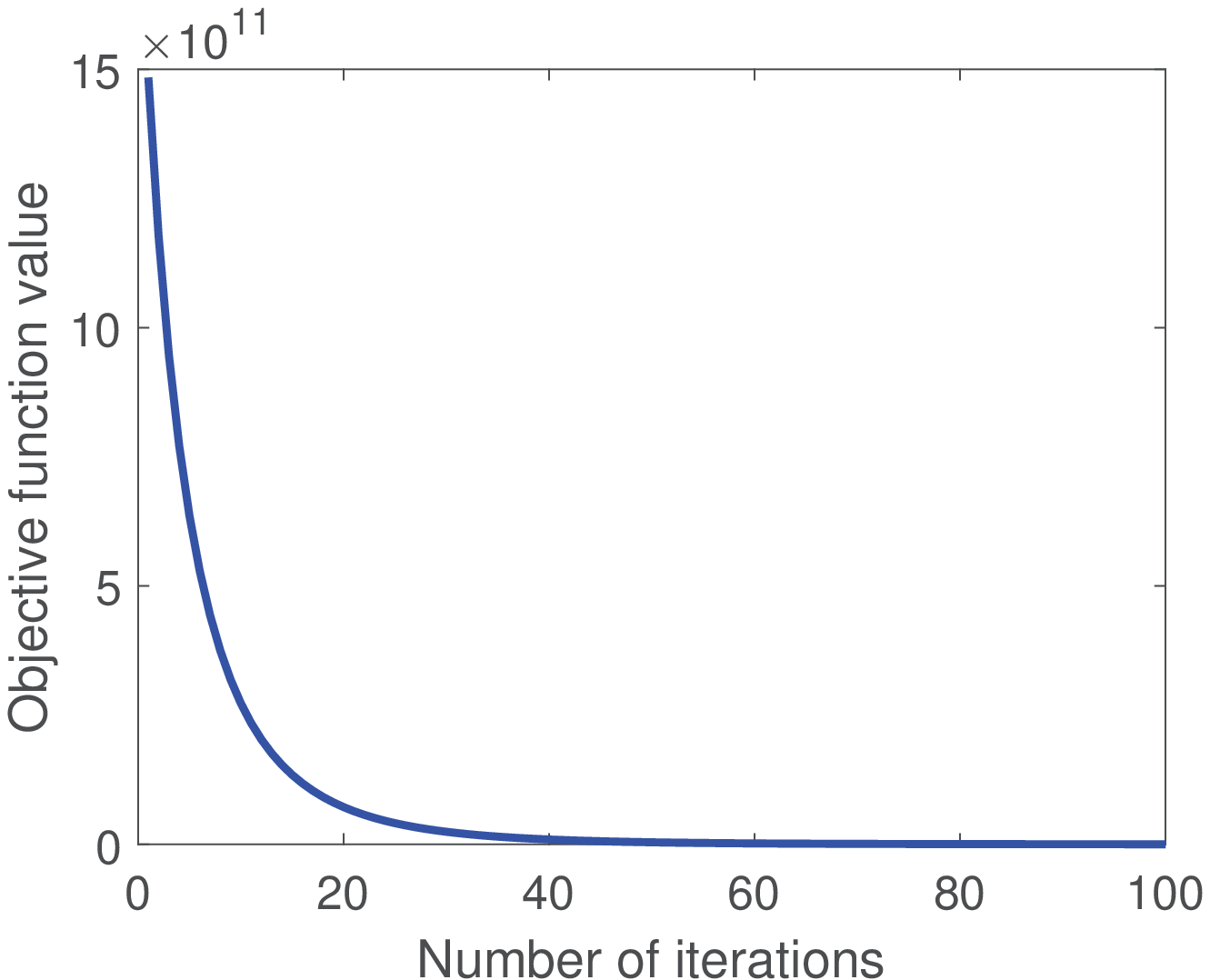}
\end{minipage}
}
\subfigure[ORL]{
\begin{minipage}{0.31\linewidth}
\centering
  \includegraphics[width=\textwidth]{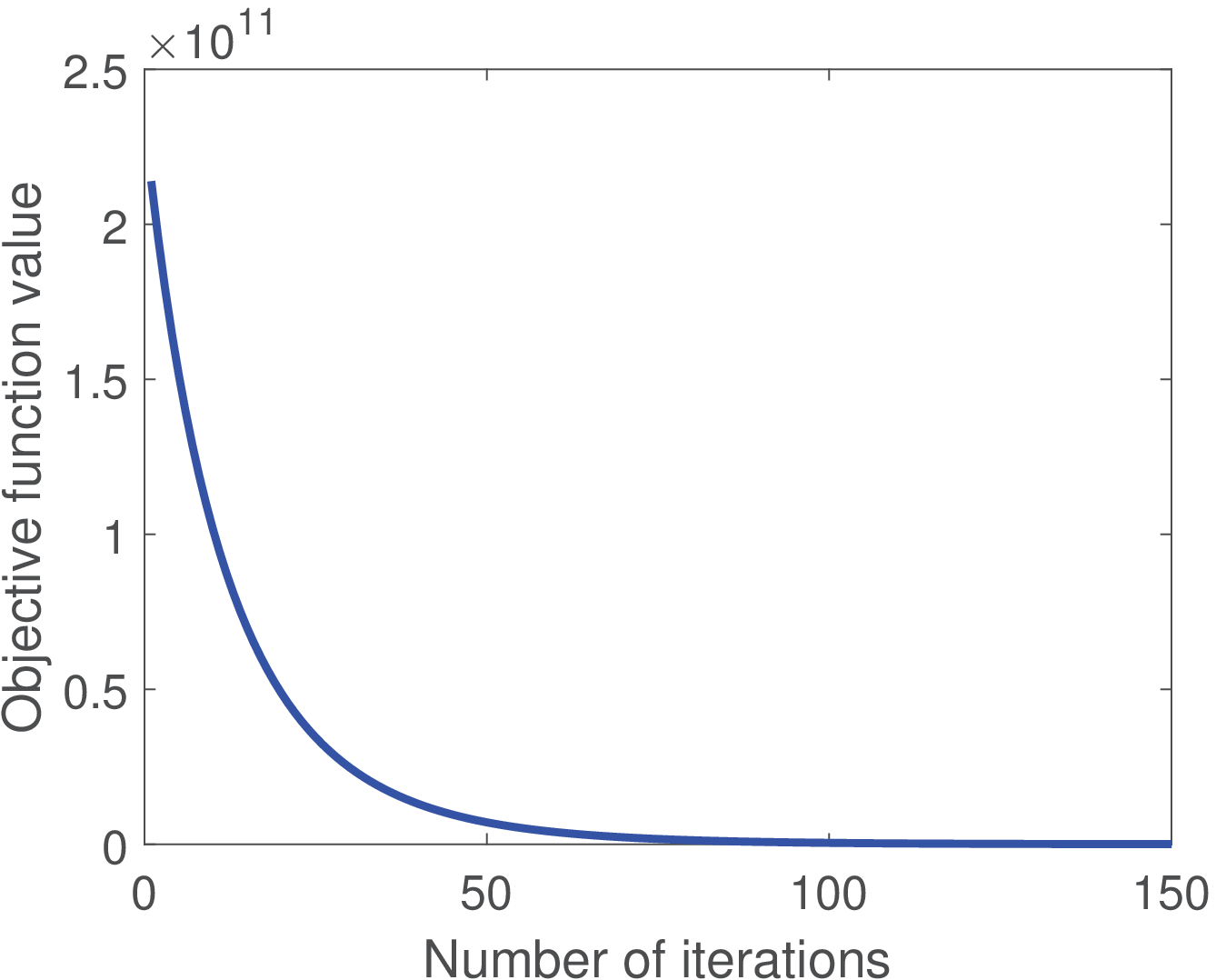}
\end{minipage}
}
%
\subfigure[Colon]{
\begin{minipage}{0.31\linewidth}
\centering
  \includegraphics[width=\textwidth]{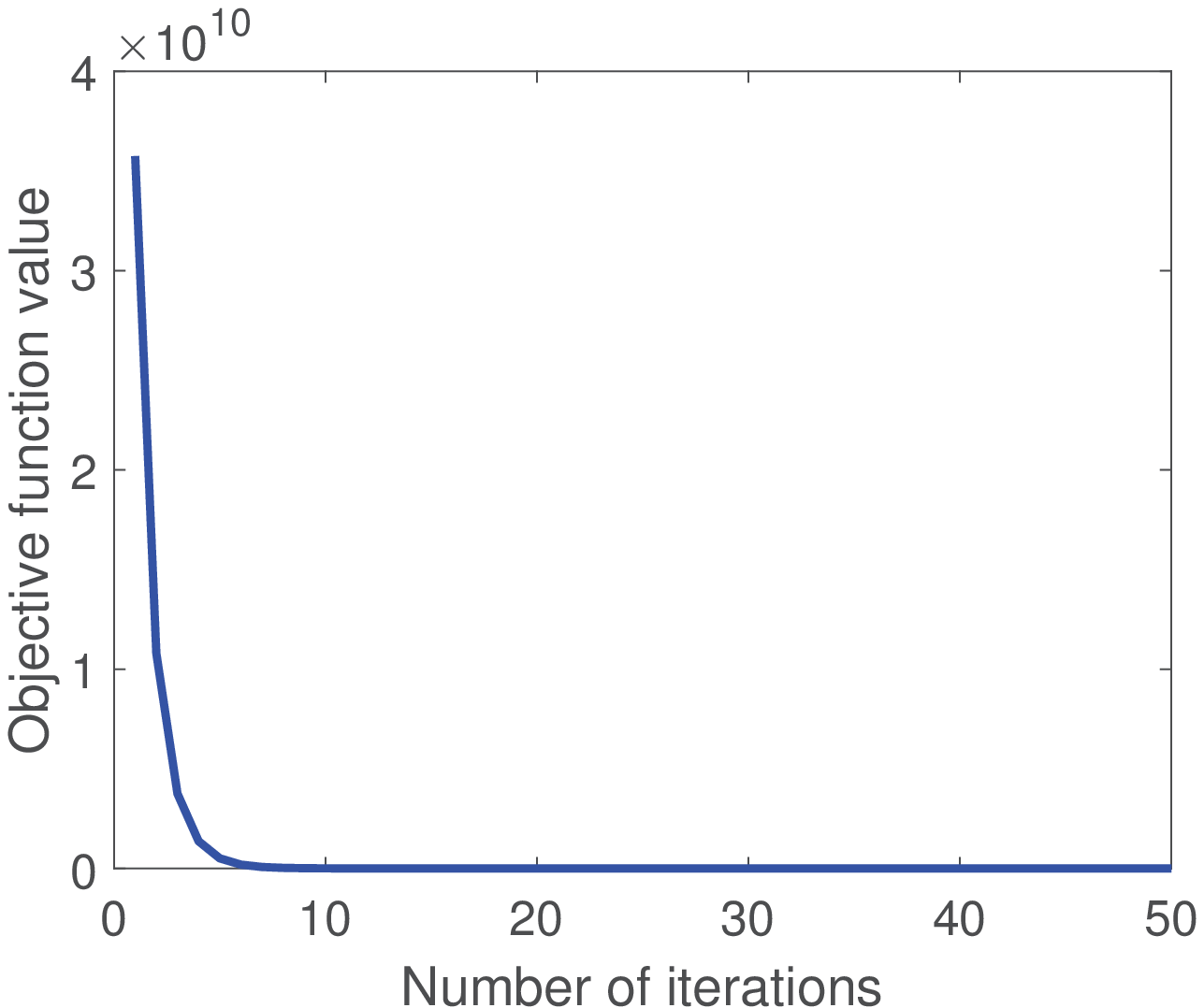}
\end{minipage}
}
\subfigure[warpPIE10P]{
\begin{minipage}{0.31\linewidth}
\centering
  \includegraphics[width=\textwidth]{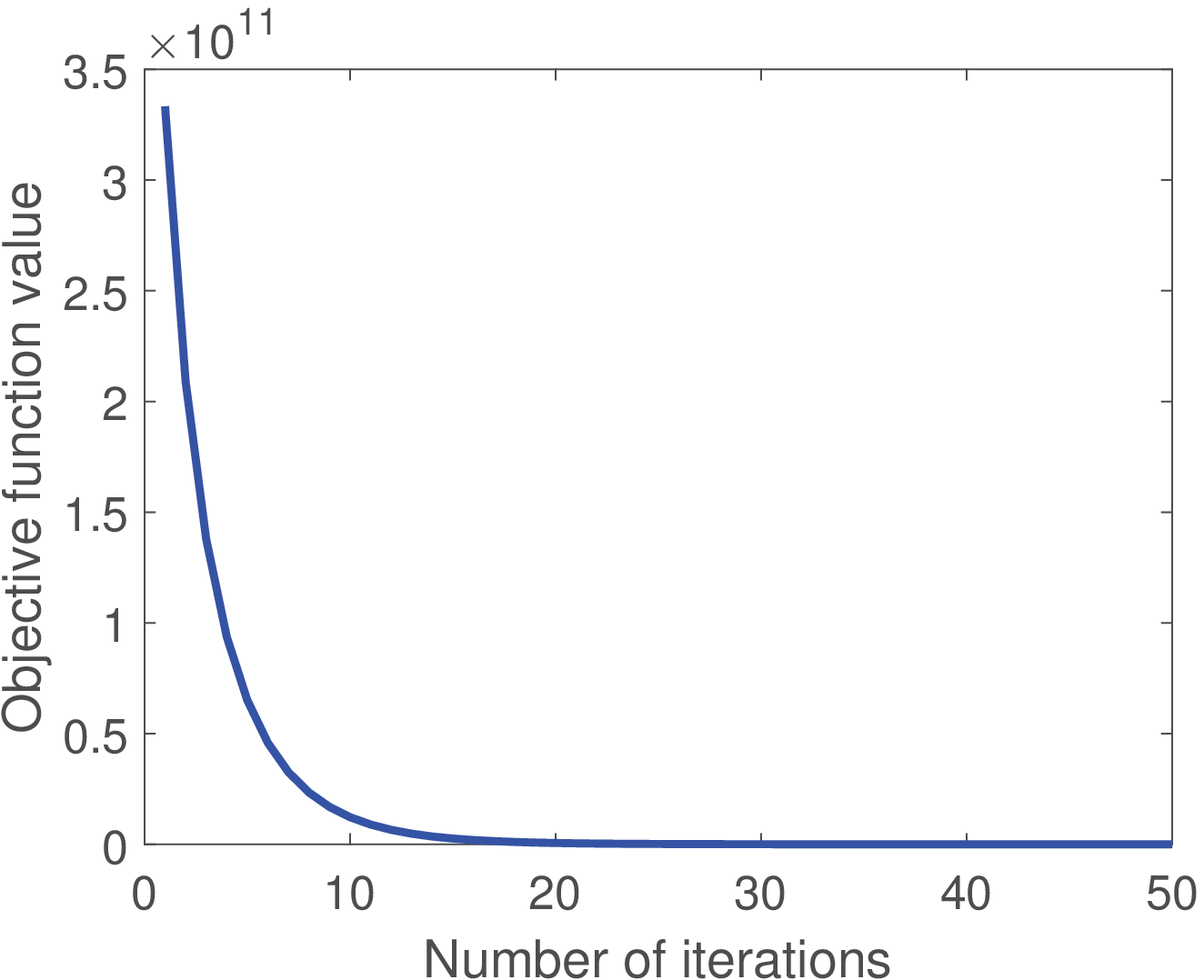}
\end{minipage}
}
\subfigure[CLIOMA]{
\begin{minipage}{0.31\linewidth}
\centering
  \includegraphics[width=\textwidth]{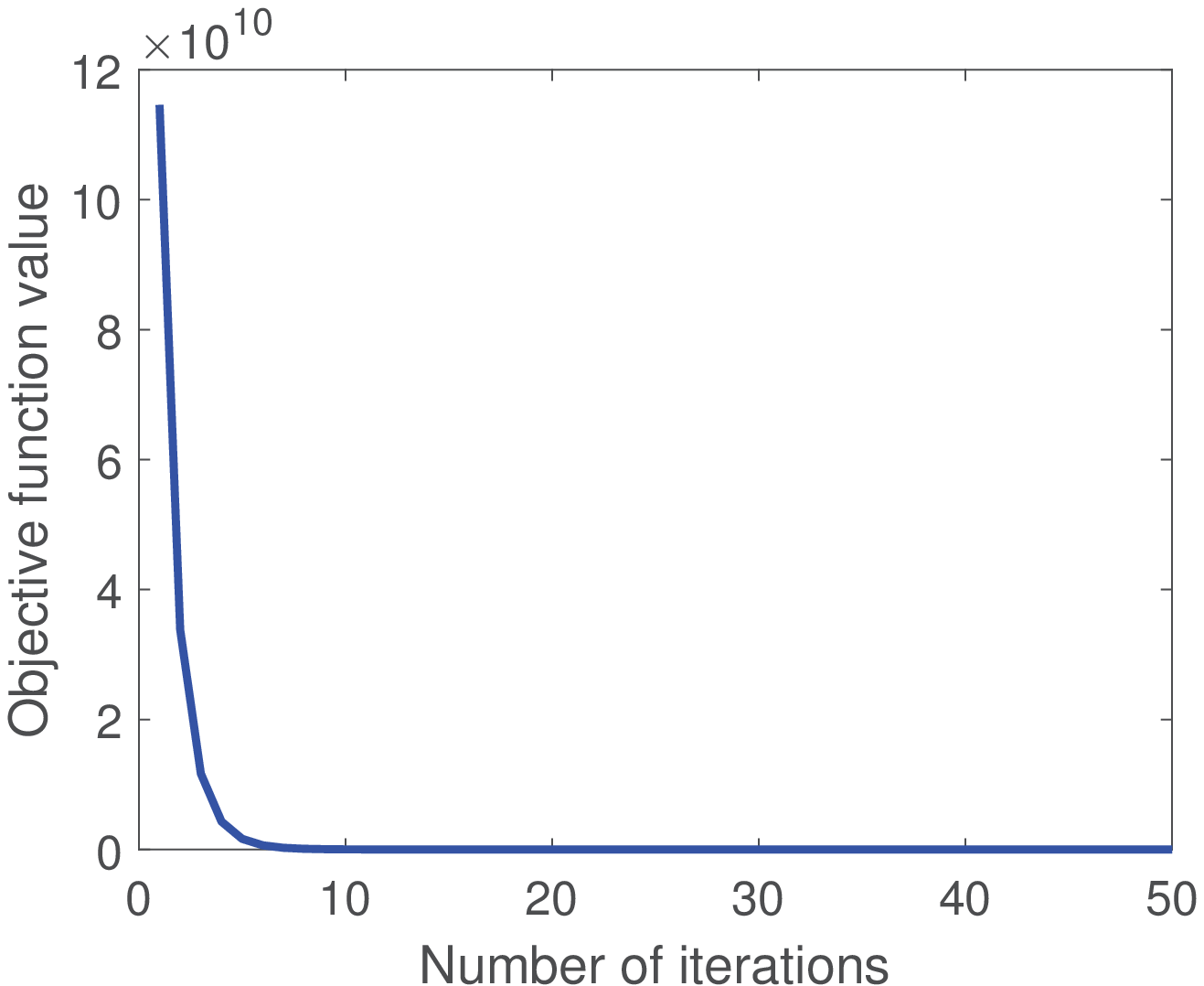}
\end{minipage}
}
\centering
\caption{The convergence curves of SPLR ($\alpha  = 1$, ${\lambda _1} = 1$, ${\lambda _{\rm{2}}} = 1$, ${\lambda _{\rm{3}}} = 1$)}\label{fig::picture3}
\end{figure*}
\FloatBarrier

 As depicted in Fig. \ref{fig::picture3}, the objective function value decreases monotonically under different circumstances. Precisely speaking, most of the datasets including USPS, Umist, COIL20, ORL, Colon, warpPIE10 and GLIOMA converge within 100 iterations. Isolet converges within 600 iterations and Madelon converges within 1500 iterations. Then, the maximum iteration number is set as 1500 in experiments.
\subsection{Parameter sensitivity analysis}
For SPLR, there are five parameters that need to be investigated, namely, $\alpha $, ${\lambda _1}$, ${\lambda _{\rm{2}}}$, ${\lambda _{\rm{3}}}$ and $\gamma $. Following \cite{art35}, $\gamma $ is fixed to 2. The remaining parameters are searched from $\left\{ {{{10}^{ - 3}},\;{{10}^{ - 2}},\;{{10}^{ - 1}},\;1,\;{\rm{1}}{{\rm{0}}^1},\;{\rm{1}}{{\rm{0}}^2},\;{\rm{1}}{{\rm{0}}^3}} \right\}$. Seeing that the space is limited, only the performance on six datasets including USPS, Madelon, Isolet, Colon, warpPIE10P and GLIOMA is reported in terms of ACC in Figs. \ref{fig::picture4}-\ref{fig::picture7}.
\begin{figure*}[htbp]  
\centering
\subfigure[USPS]{
\begin{minipage}{0.31\linewidth}
\centering
  \includegraphics[width=\textwidth]{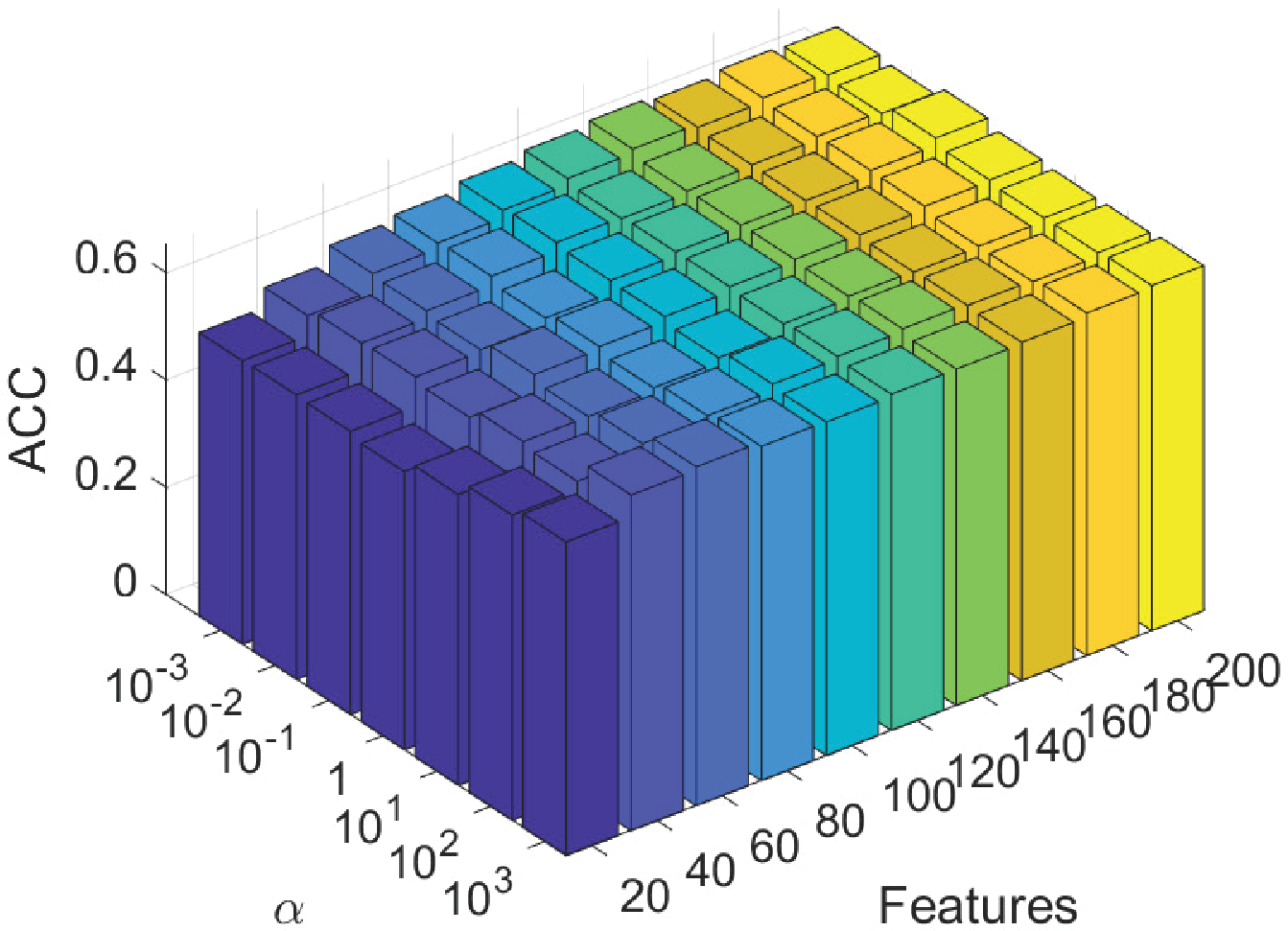}
\end{minipage}
}
\subfigure[Madelon]{
\begin{minipage}{0.31\linewidth}
\centering
  \includegraphics[width=\textwidth]{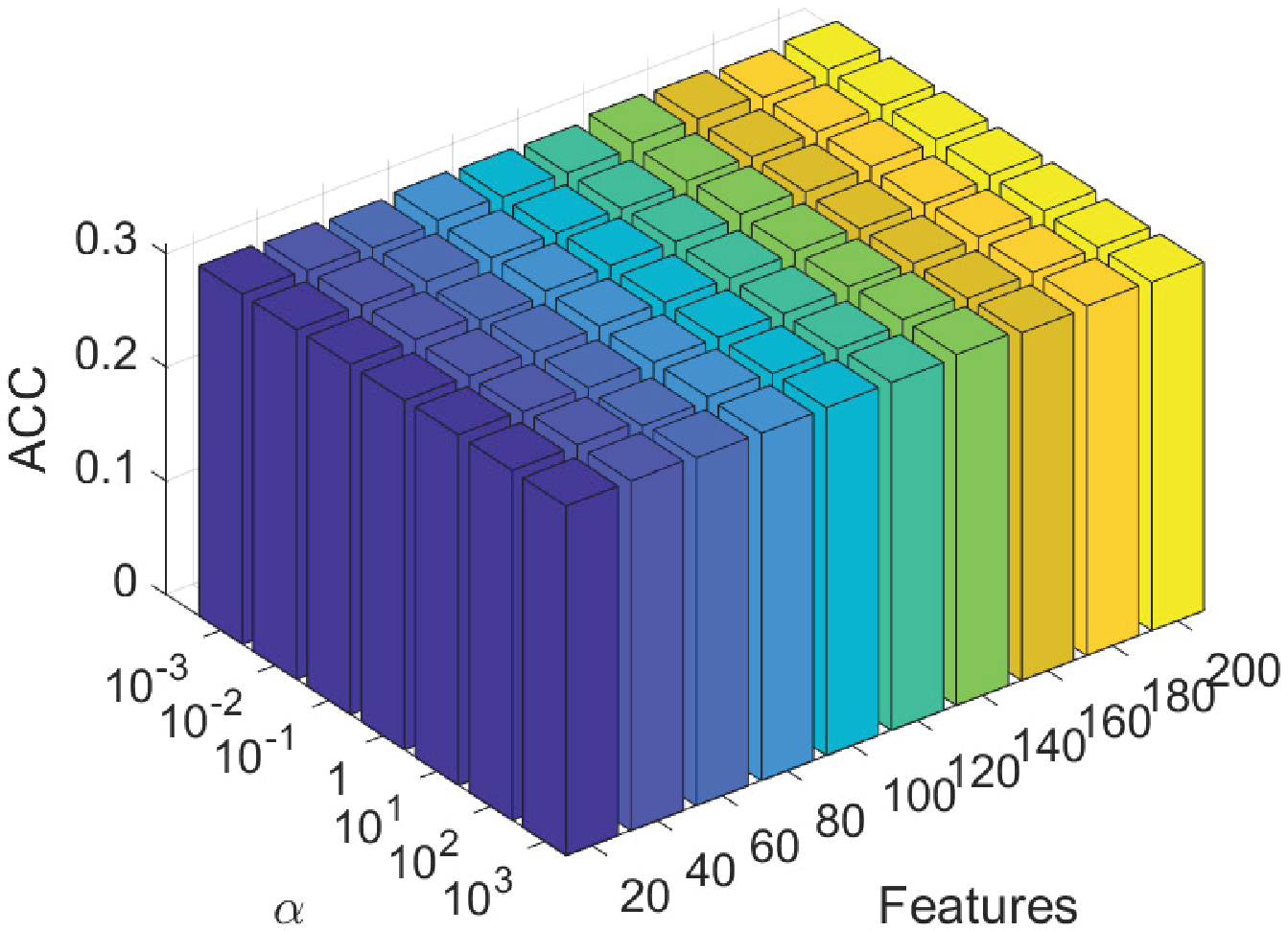}
\end{minipage}
}
\subfigure[Isolet]{
\begin{minipage}{0.31\linewidth}
\centering
  \includegraphics[width=\textwidth]{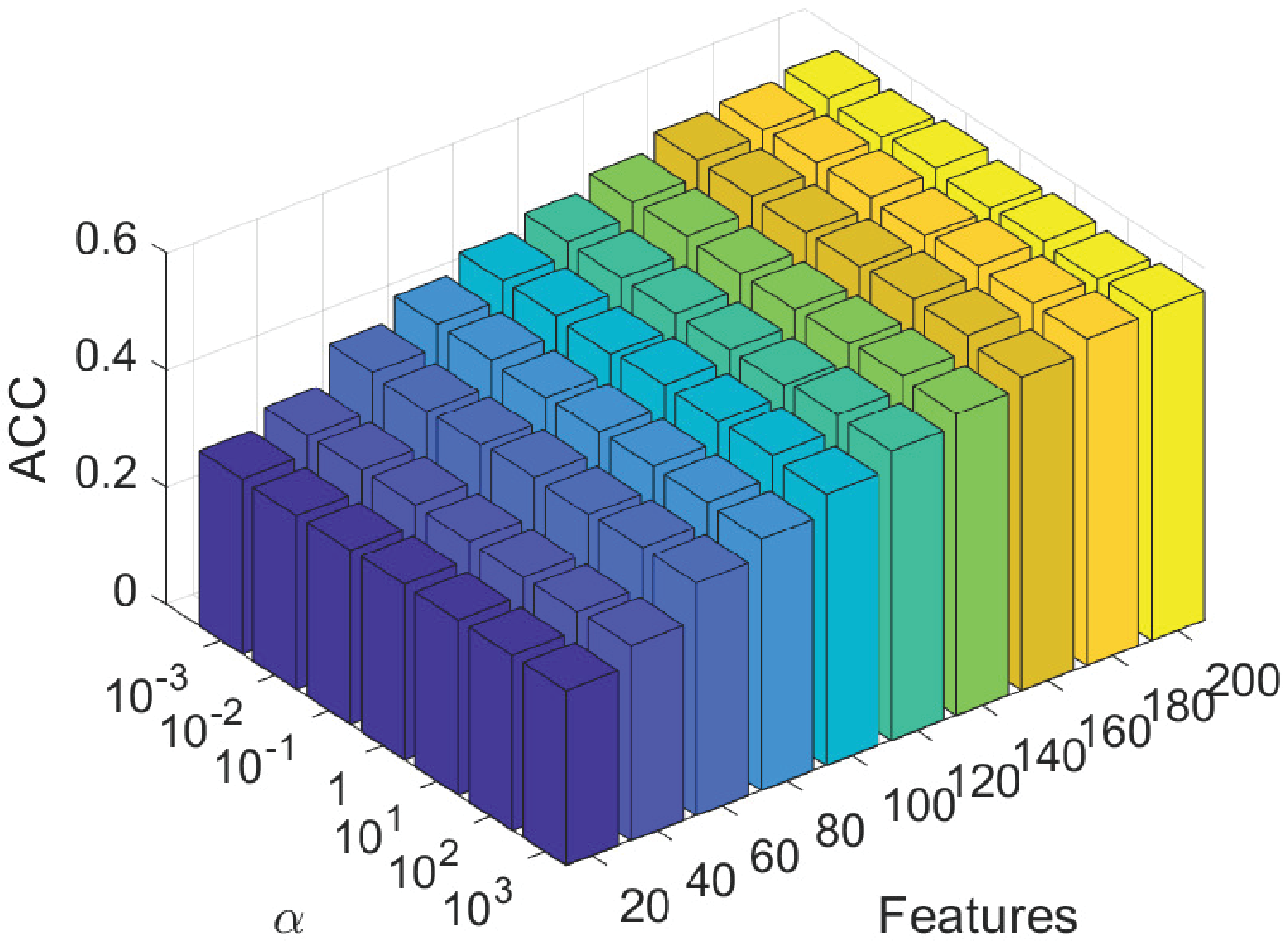}
\end{minipage}
}
%
\subfigure[Colon]{
\begin{minipage}{0.31\linewidth}
\centering
  \includegraphics[width=\textwidth]{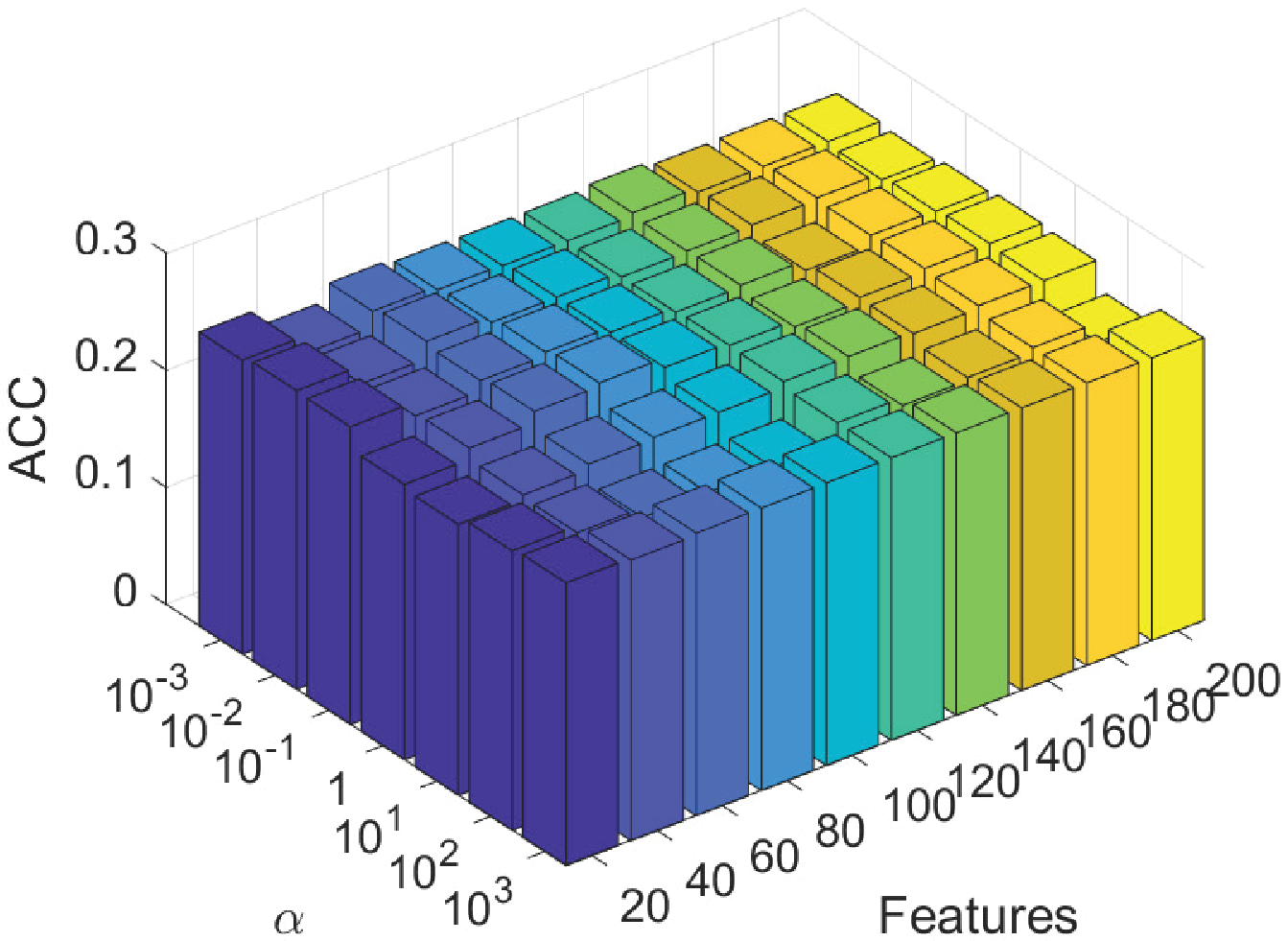}
\end{minipage}
}
\subfigure[warpPIE10P]{
\begin{minipage}{0.31\linewidth}
\centering
  \includegraphics[width=\textwidth]{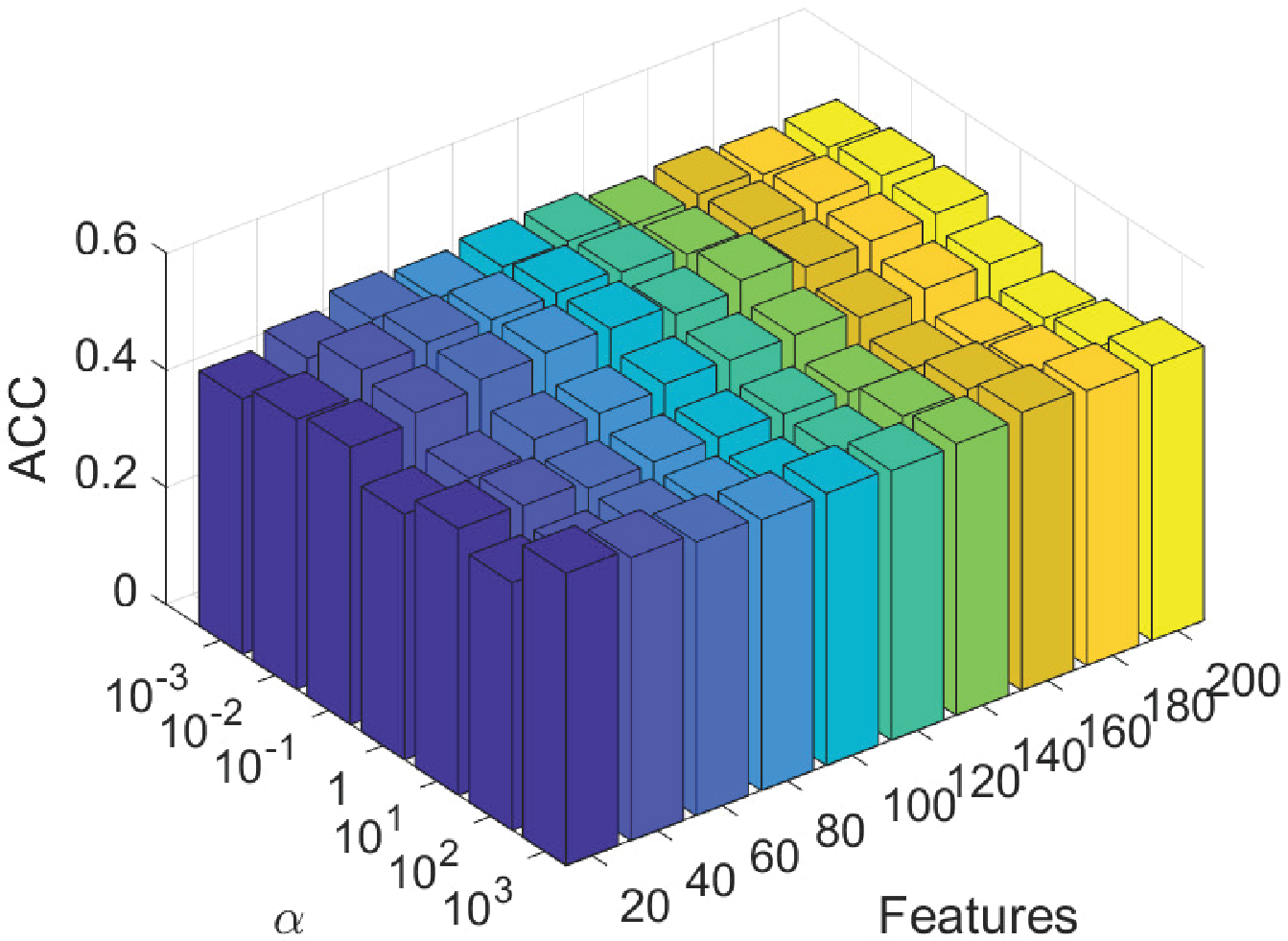}
\end{minipage}
}
\subfigure[GLIOMA]{
\begin{minipage}{0.31\linewidth}
\centering
  \includegraphics[width=\textwidth]{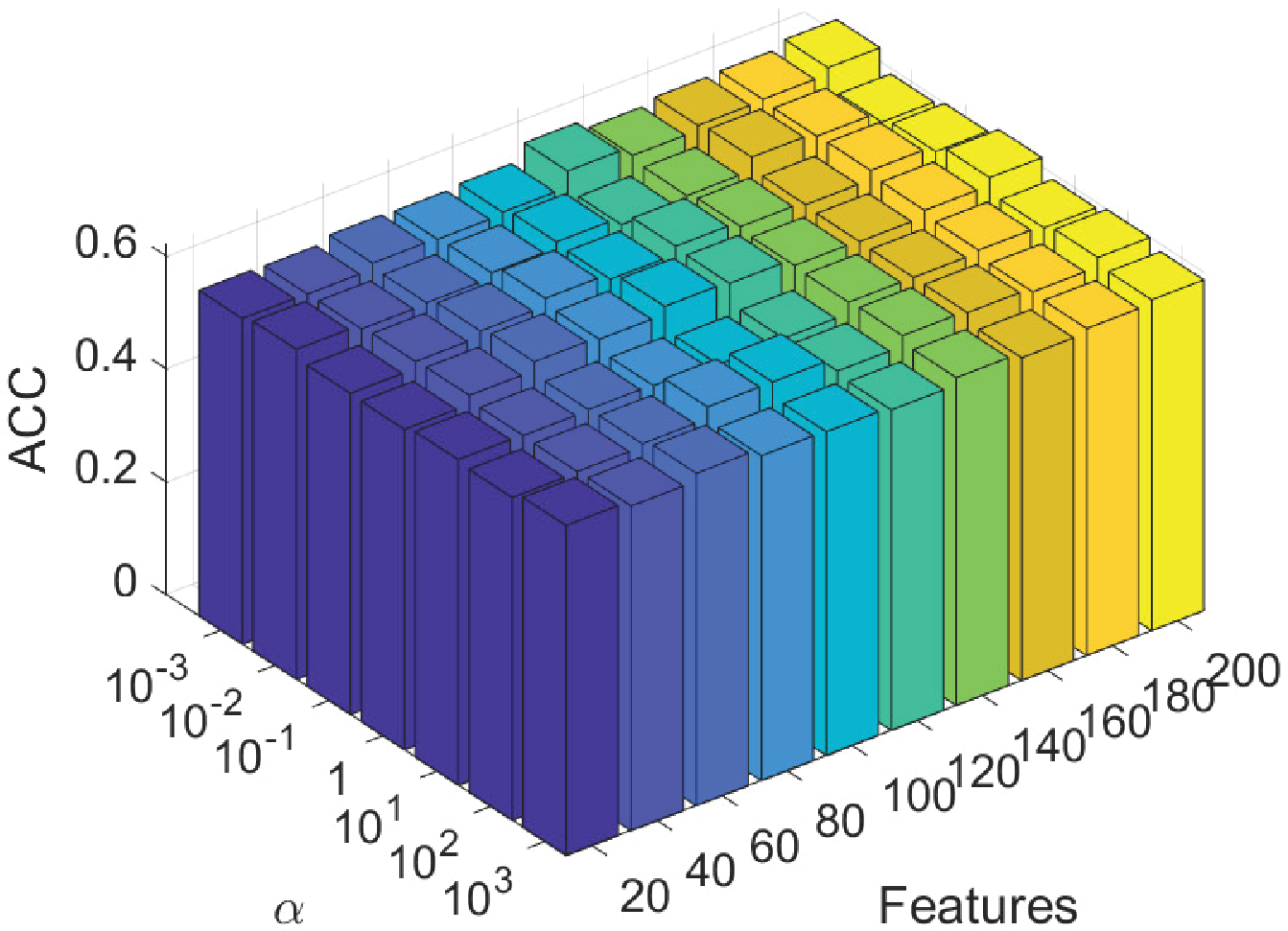}
\end{minipage}
}
\FloatBarrier

\centering
\caption{Parameter sensitivity with respect to $\alpha $ in terms of ACC (${\lambda _1} = 1$, ${\lambda _{\rm{2}}} = 1$, ${\lambda _{\rm{3}}} = 1$)}\label{fig::picture4}
\end{figure*}

\begin{figure*}[htbp]  
\centering
\subfigure[USPS]{
\begin{minipage}{0.31\linewidth}
\centering
  \includegraphics[width=\textwidth]{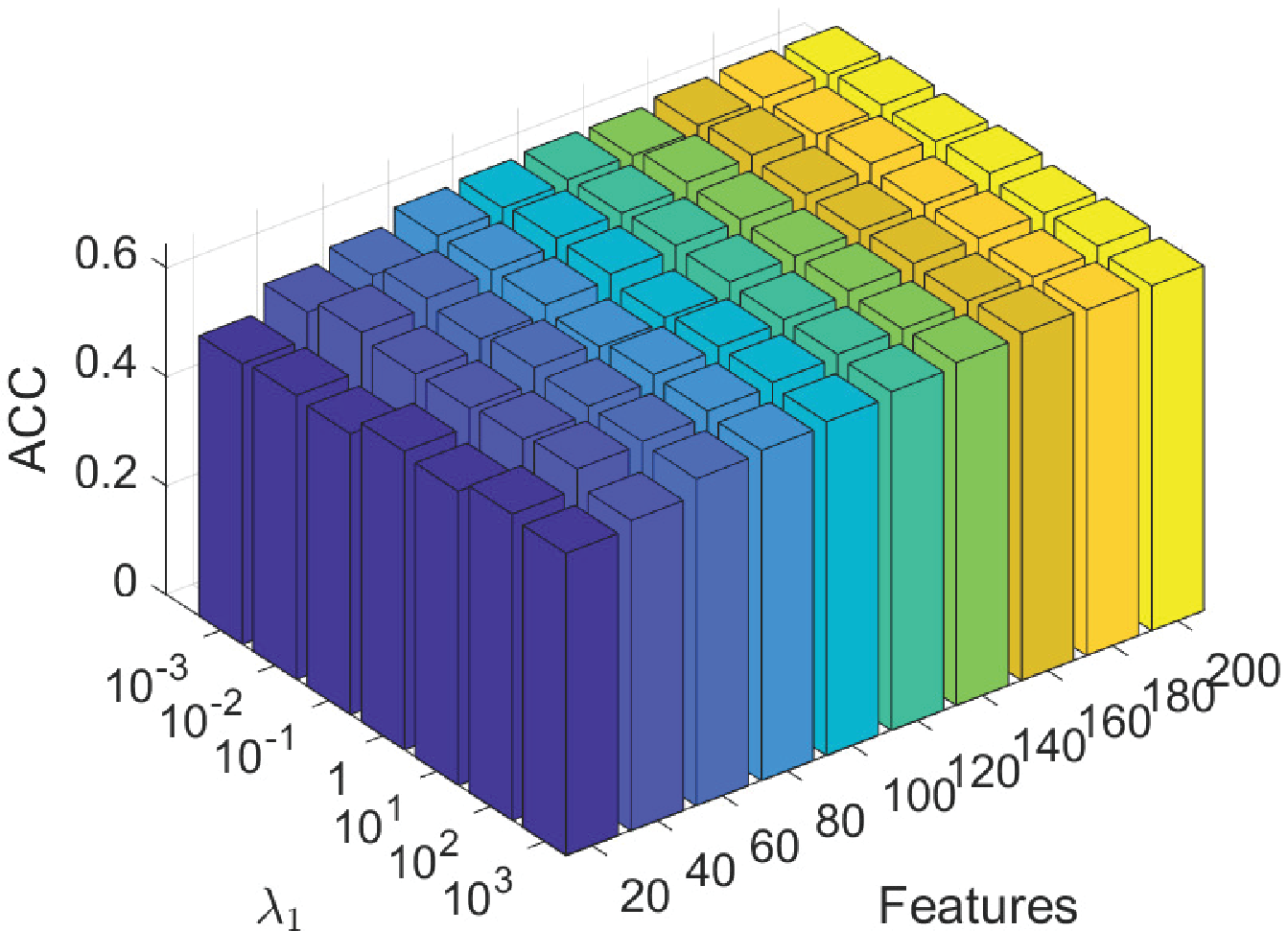}
\end{minipage}
}
\subfigure[Madelon]{
\begin{minipage}{0.31\linewidth}
\centering
  \includegraphics[width=\textwidth]{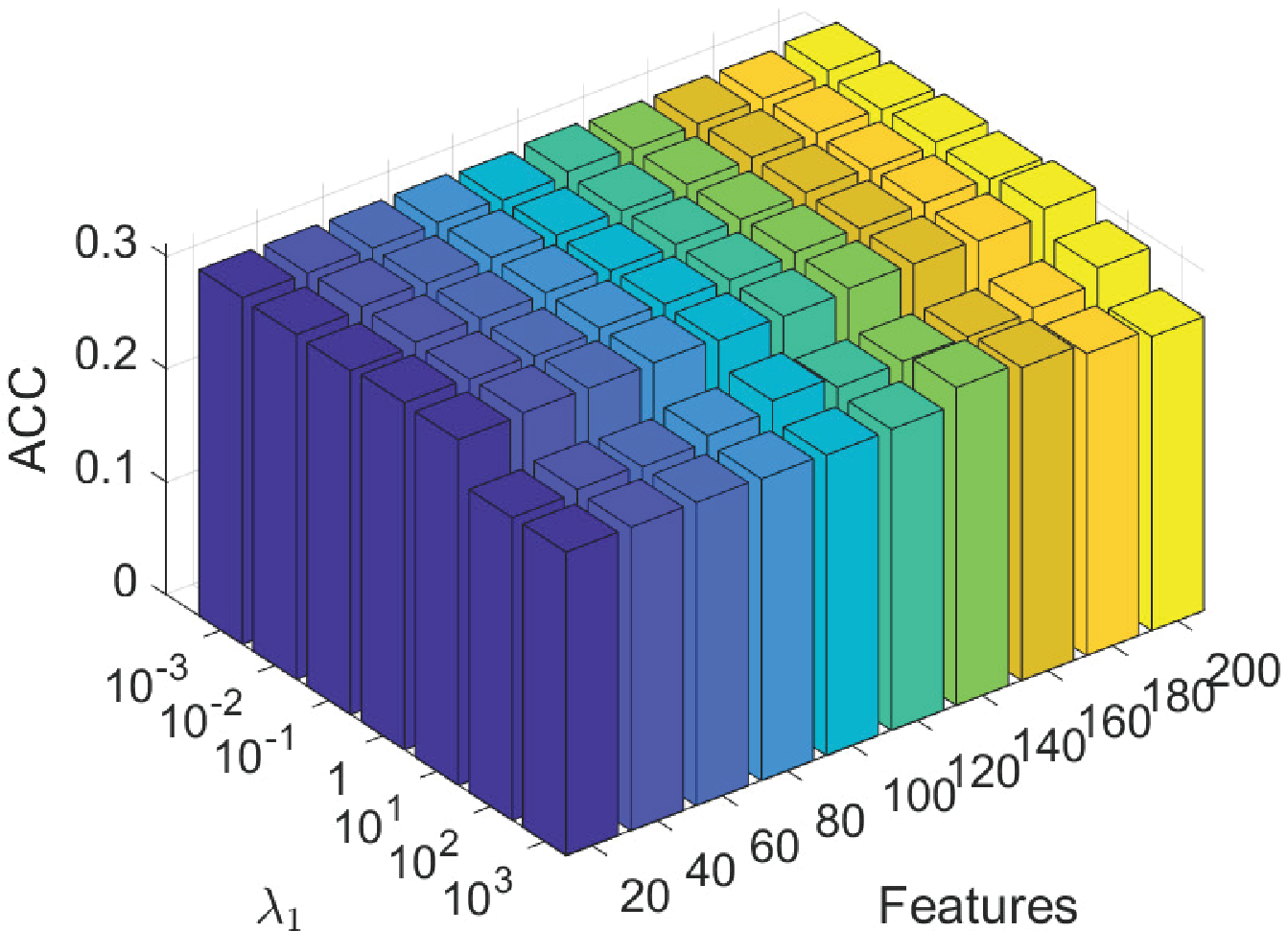}
\end{minipage}
}
\subfigure[Isolet]{
\begin{minipage}{0.31\linewidth}
\centering
  \includegraphics[width=\textwidth]{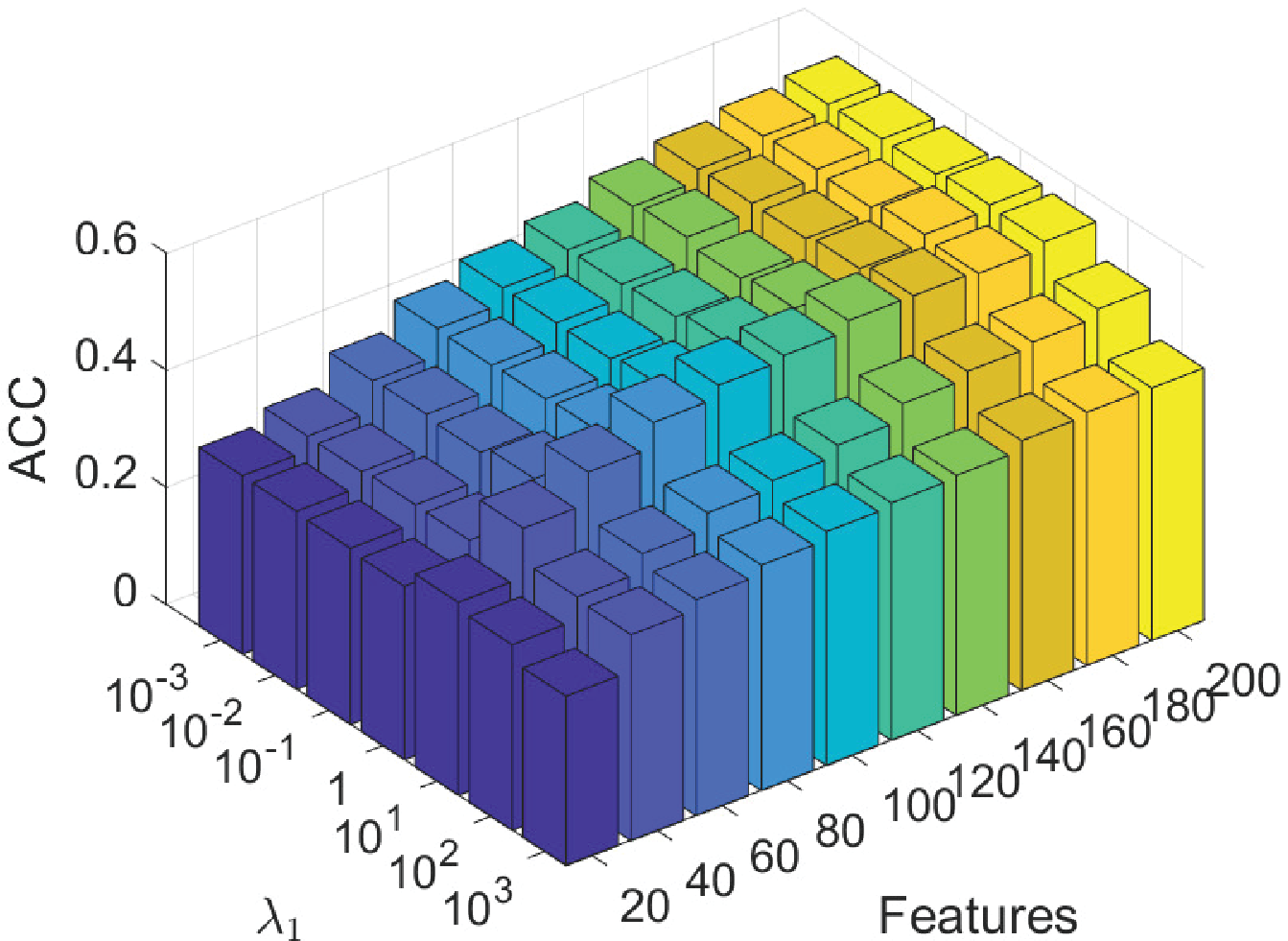}
\end{minipage}
}
\subfigure[Colon]{
\begin{minipage}{0.31\linewidth}
\centering
  \includegraphics[width=\textwidth]{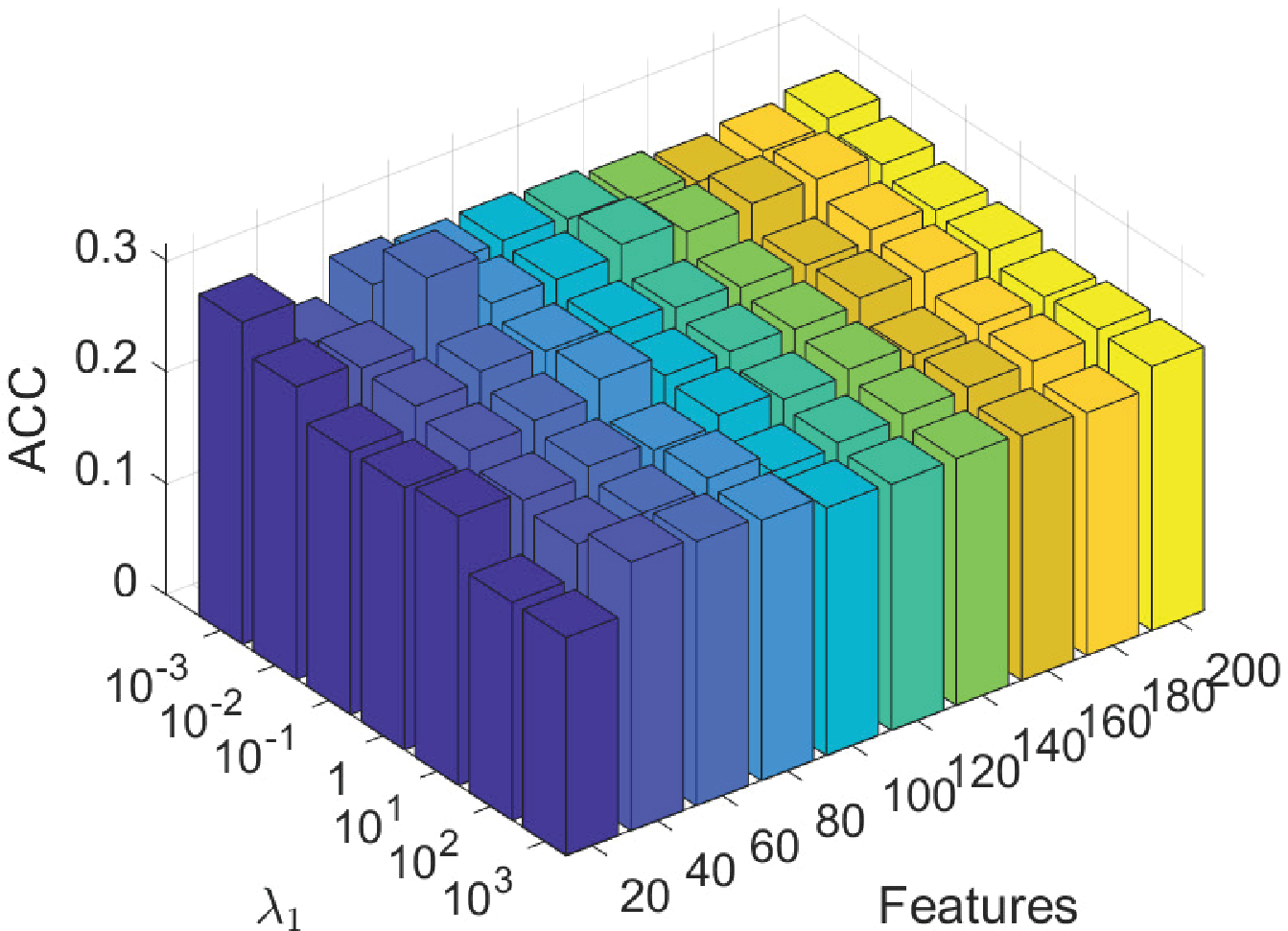}
\end{minipage}
}
\subfigure[warpPIE10P]{
\begin{minipage}{0.31\linewidth}
\centering
  \includegraphics[width=\textwidth]{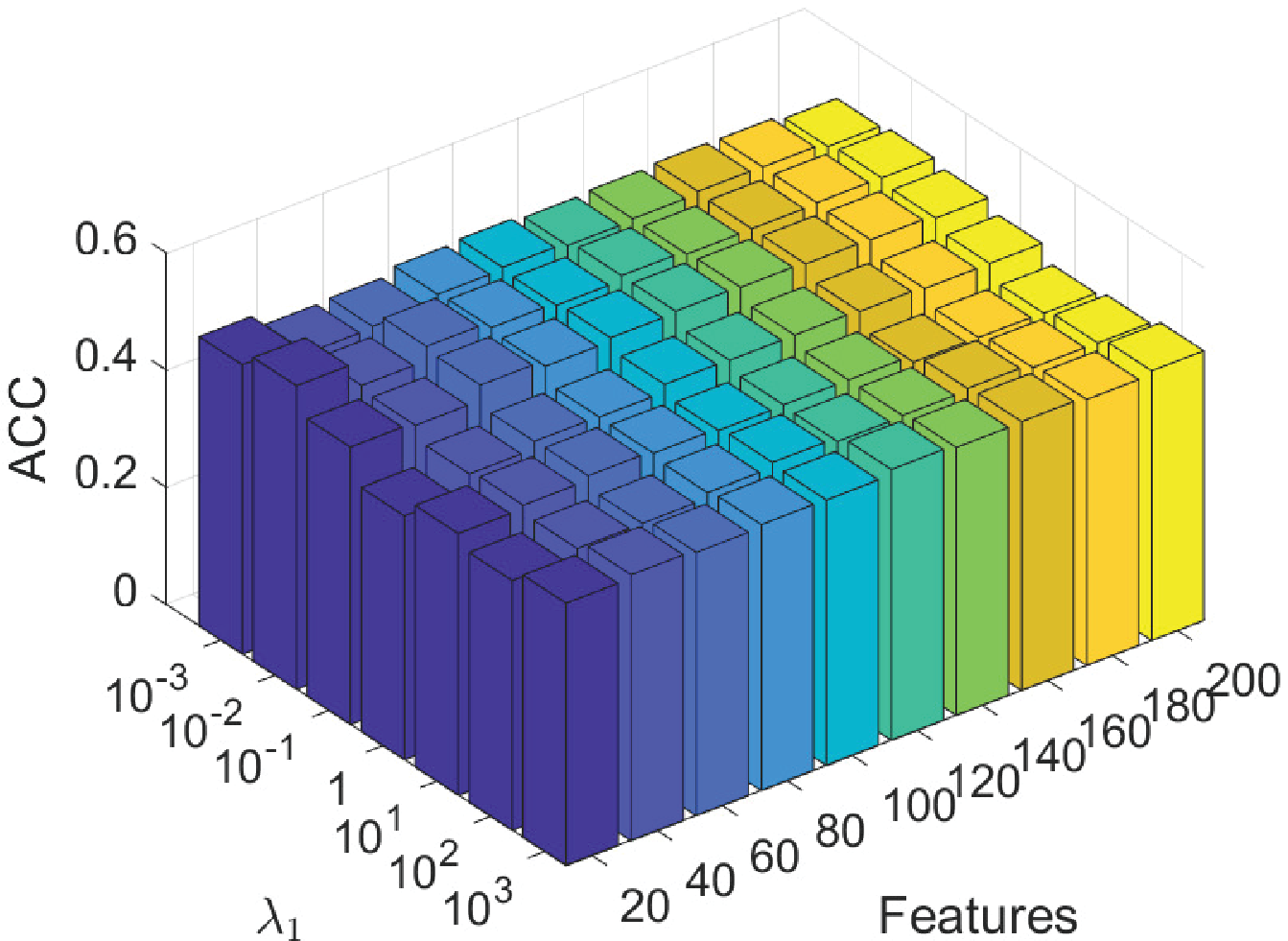}
\end{minipage}
}
\subfigure[GLIOMA]{
\begin{minipage}{0.31\linewidth}
\centering
  \includegraphics[width=\textwidth]{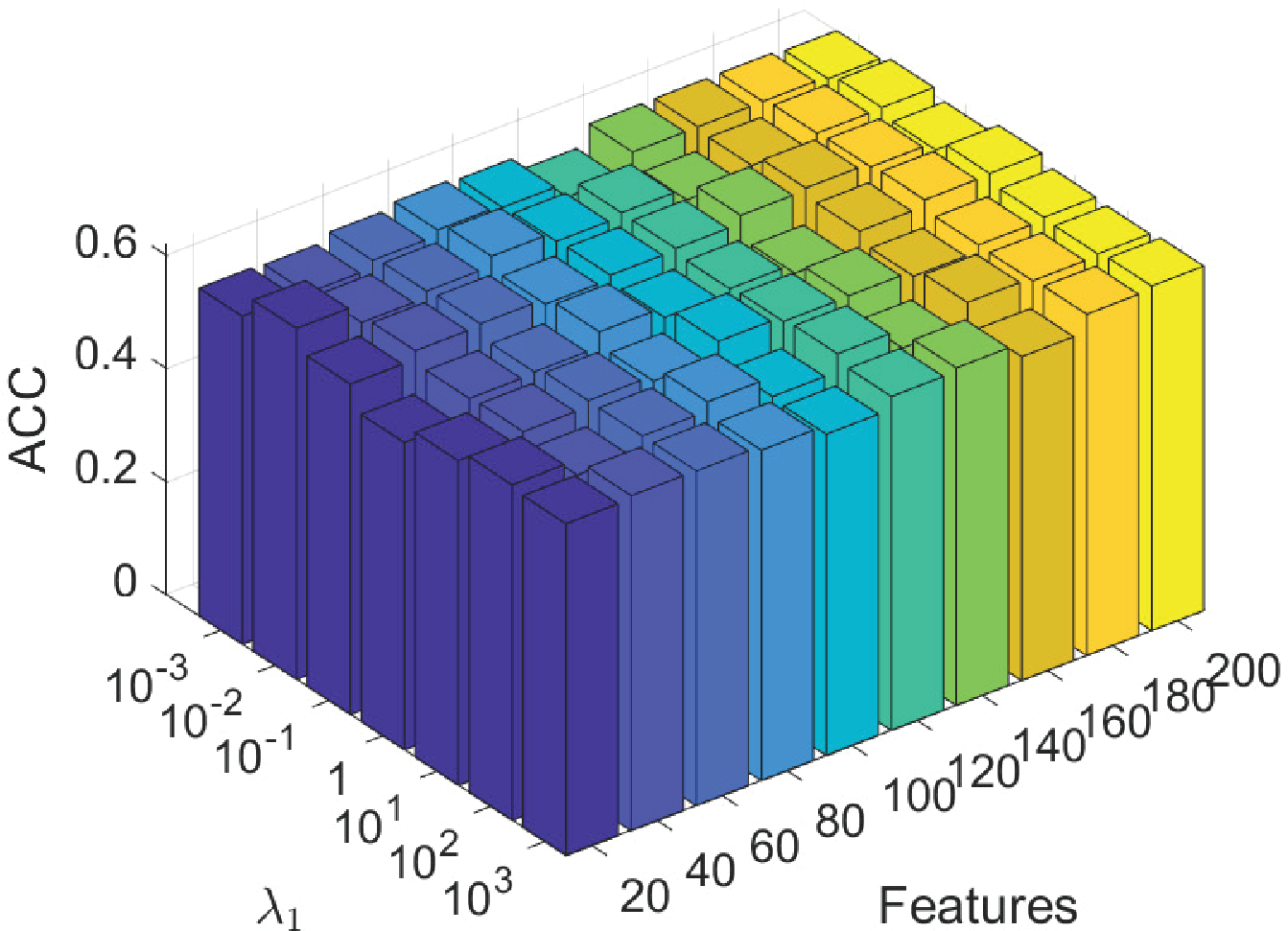}
\end{minipage}
}

\centering
\caption{Parameter sensitivity with respect to ${\lambda _1}$ in terms of ACC ($\alpha  = 1$, ${\lambda _{\rm{2}}} = 1$, ${\lambda _{\rm{3}}} = 1$)}\label{fig::picture5}
\end{figure*}

\begin{figure*}[htbp]  
\centering
\subfigure[USPS]{
\begin{minipage}{0.31\linewidth}
\centering
  \includegraphics[width=\textwidth]{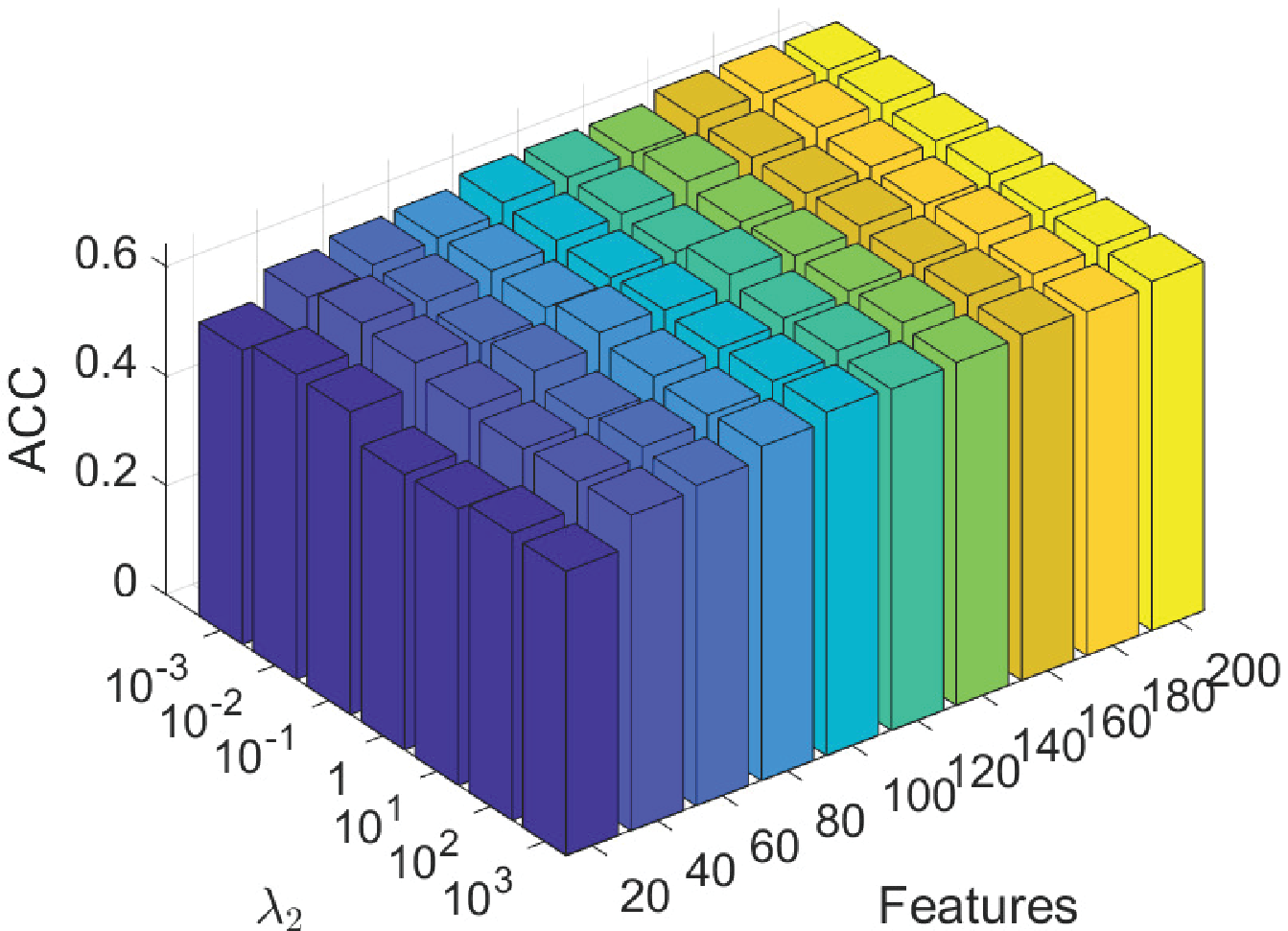}
\end{minipage}
}
\subfigure[Madelon]{
\begin{minipage}{0.31\linewidth}
\centering
  \includegraphics[width=\textwidth]{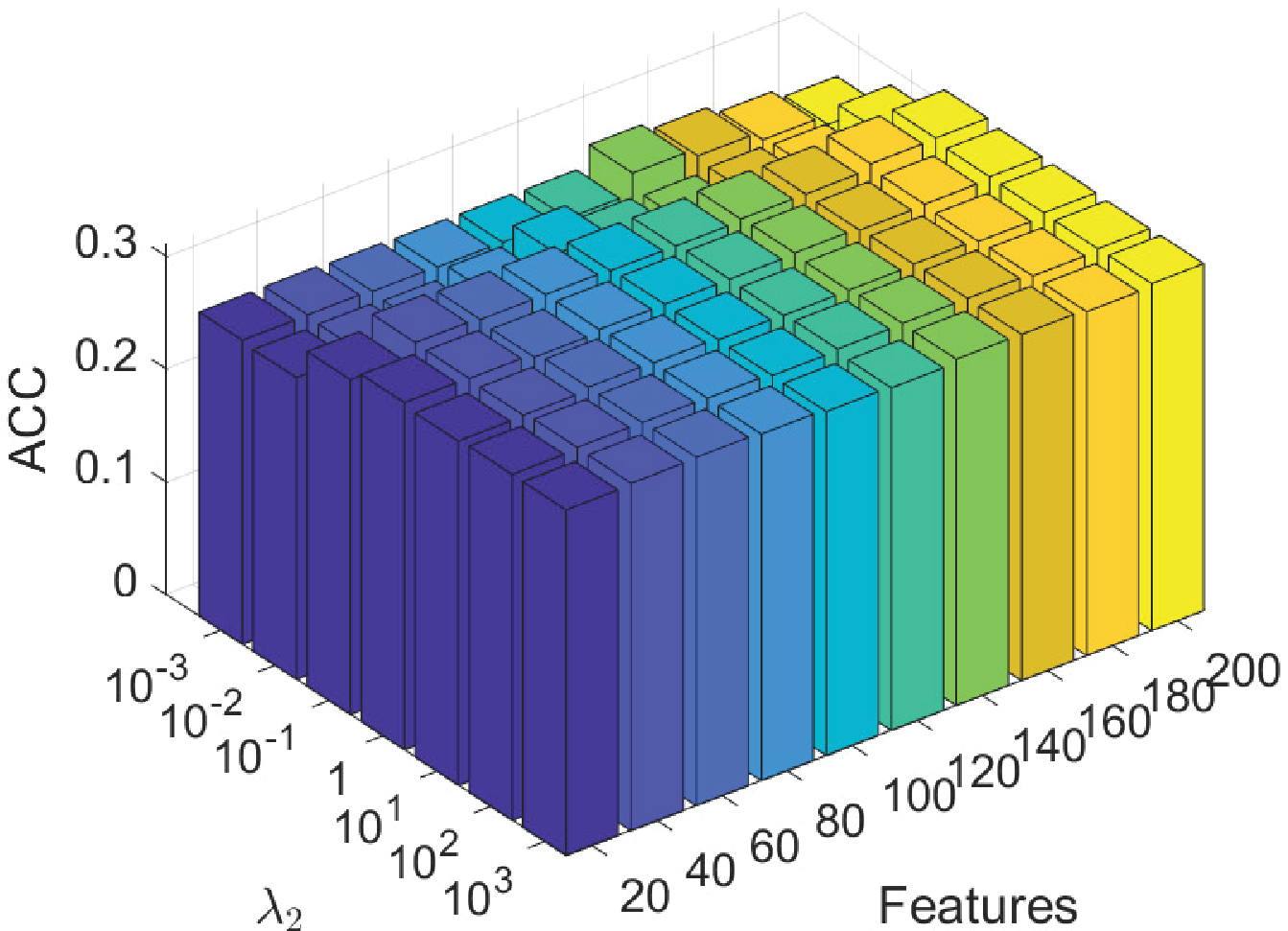}
\end{minipage}
}
\subfigure[Isolet]{
\begin{minipage}{0.31\linewidth}
\centering
  \includegraphics[width=\textwidth]{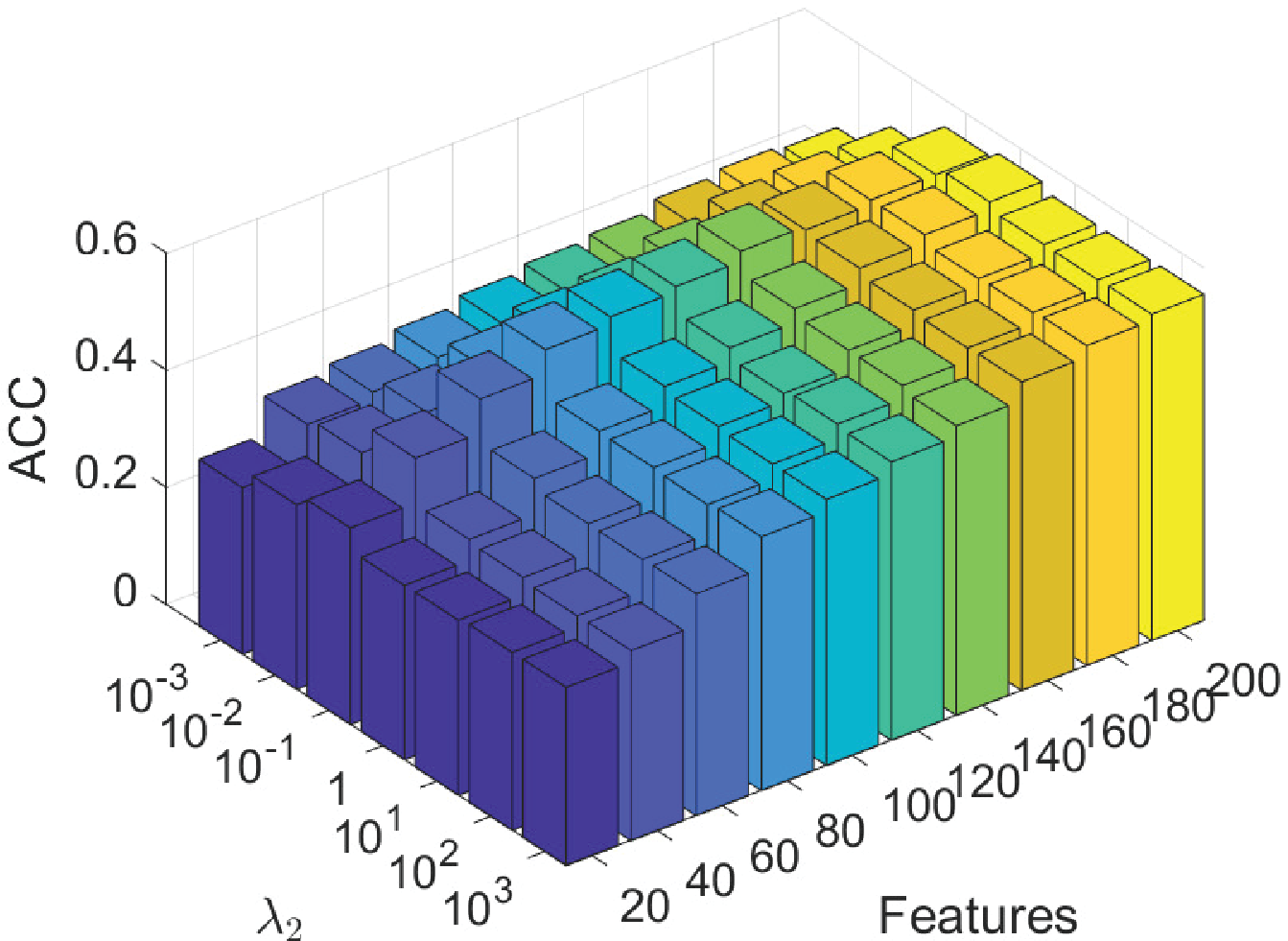}
\end{minipage}
}
\subfigure[Colon]{
\begin{minipage}{0.31\linewidth}
\centering
  \includegraphics[width=\textwidth]{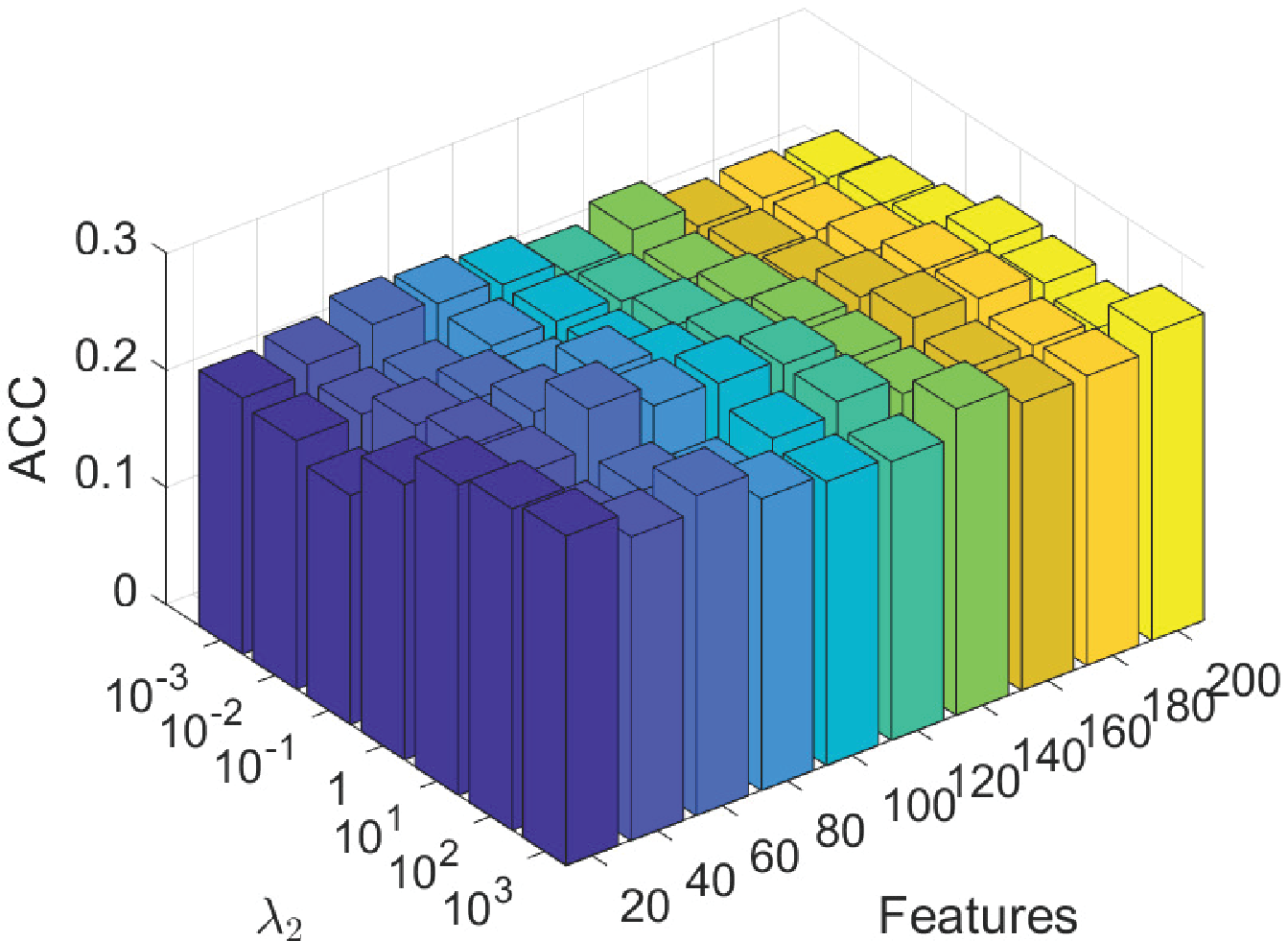}
\end{minipage}
}
\subfigure[warpPIE10P]{
\begin{minipage}{0.31\linewidth}
\centering
  \includegraphics[width=\textwidth]{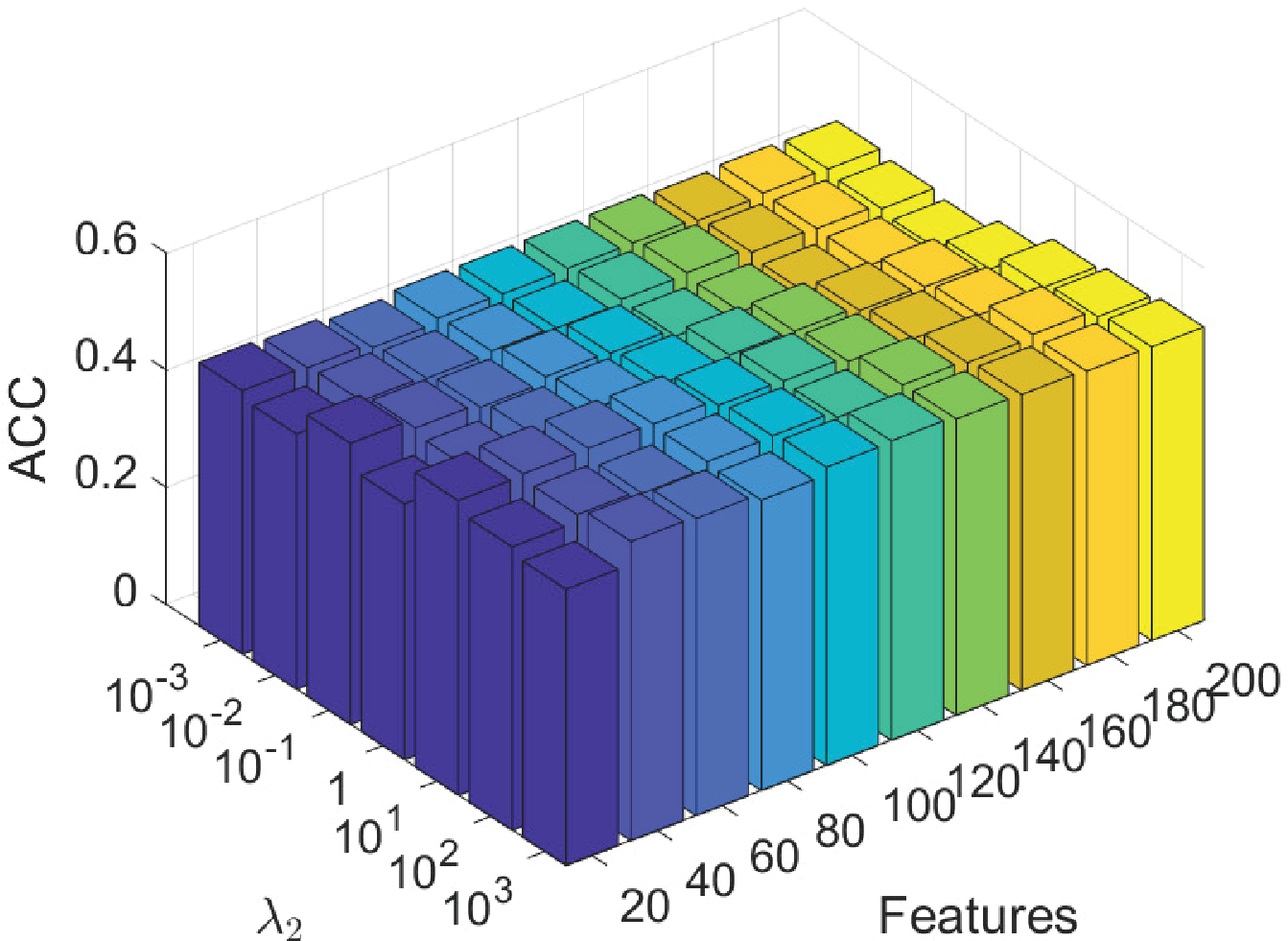}
\end{minipage}
}
\subfigure[GLIOMA]{
\begin{minipage}{0.31\linewidth}
\centering
  \includegraphics[width=\textwidth]{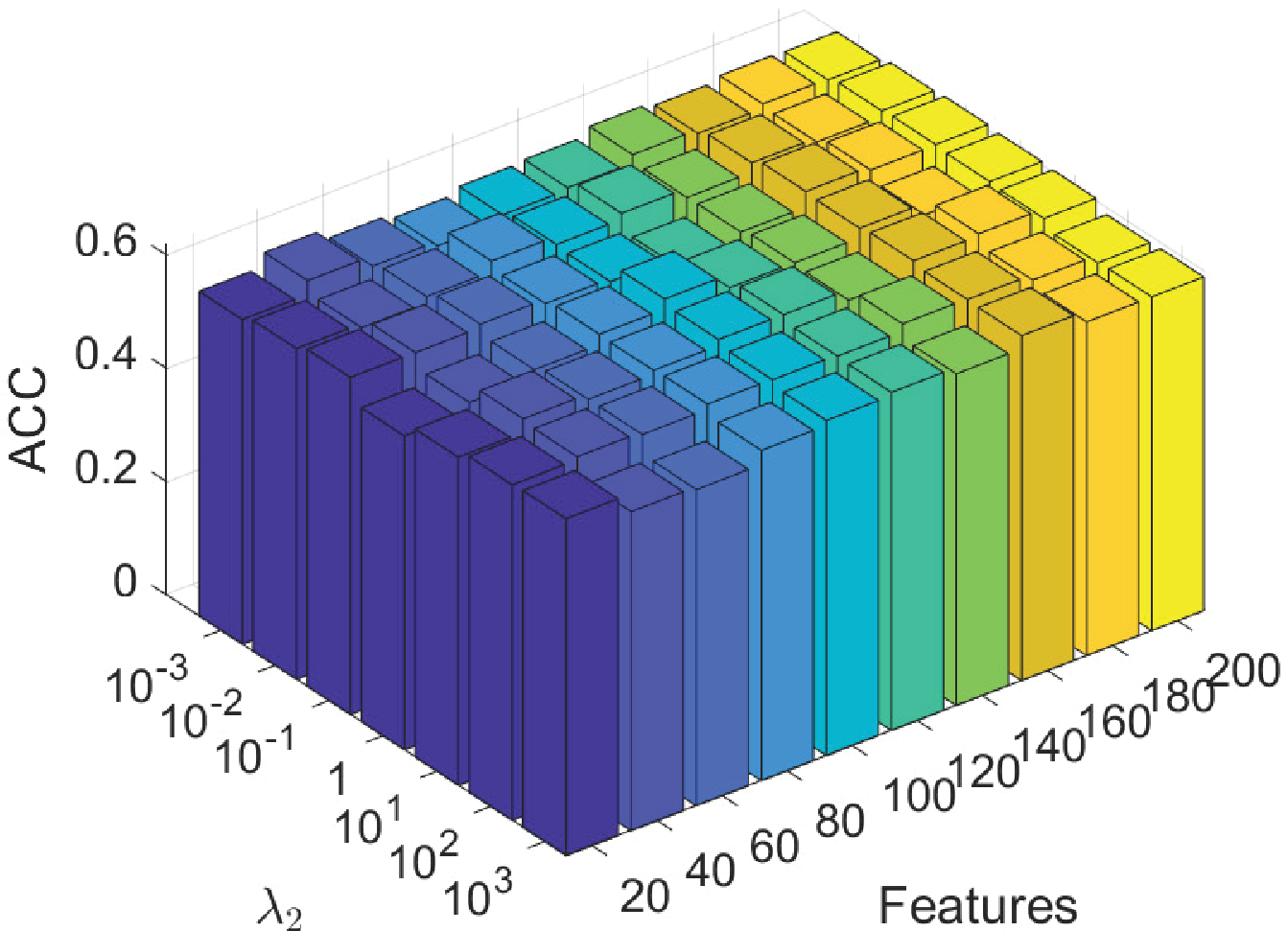}
\end{minipage}
}

\centering
\caption{Parameter sensitivity with respect to ${\lambda _2}$ in terms of ACC ($\alpha  = 1$, ${\lambda _1} = 1$, ${\lambda _{\rm{3}}} = 1$)}\label{fig::picture6}
\end{figure*}

\begin{figure*}[htbp]  
\centering
\subfigure[USPS]{
\begin{minipage}{0.31\linewidth}
\centering
  \includegraphics[width=\textwidth]{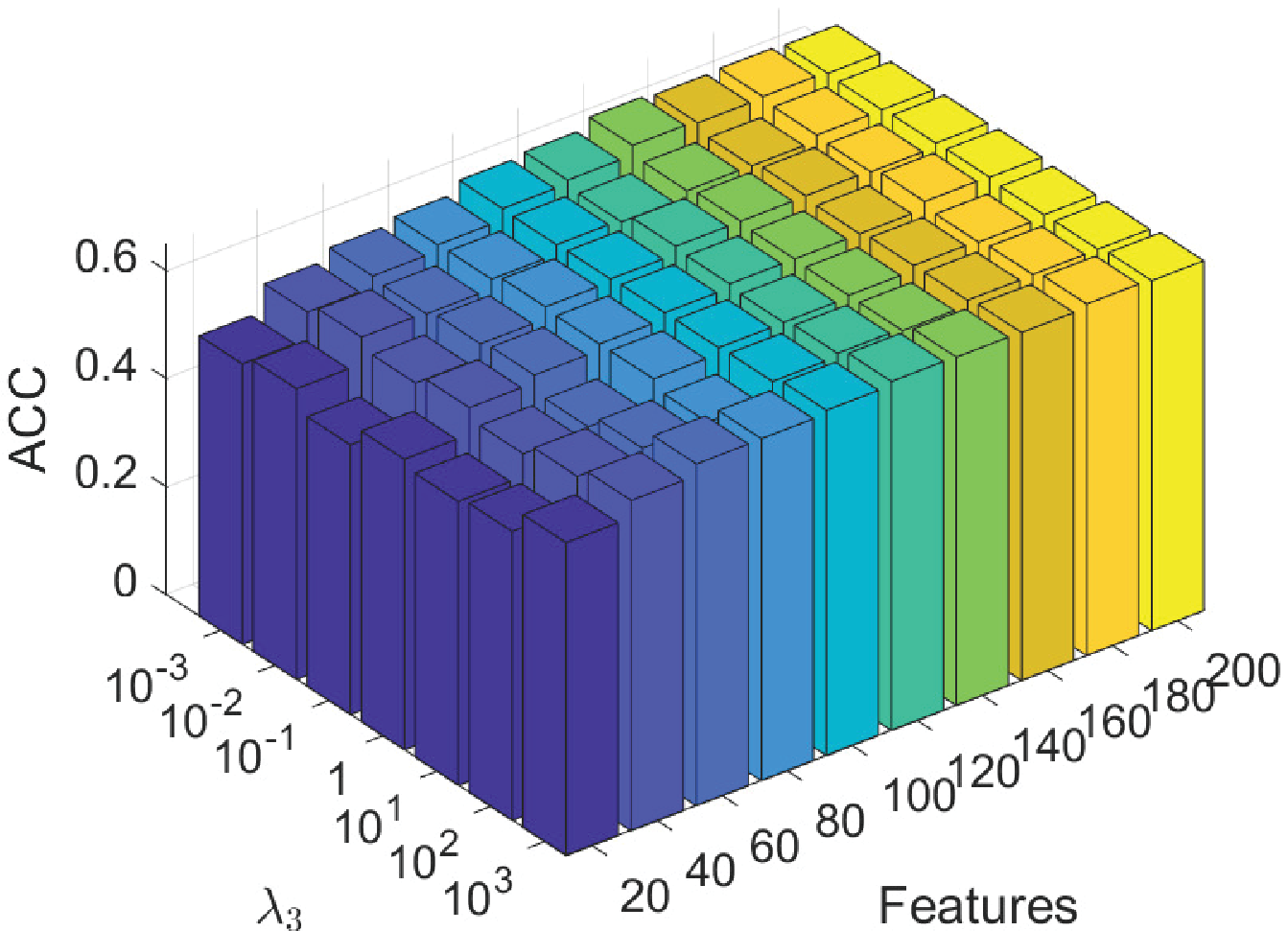}
\end{minipage}
}
\subfigure[Madelon]{
\begin{minipage}{0.31\linewidth}
\centering
  \includegraphics[width=\textwidth]{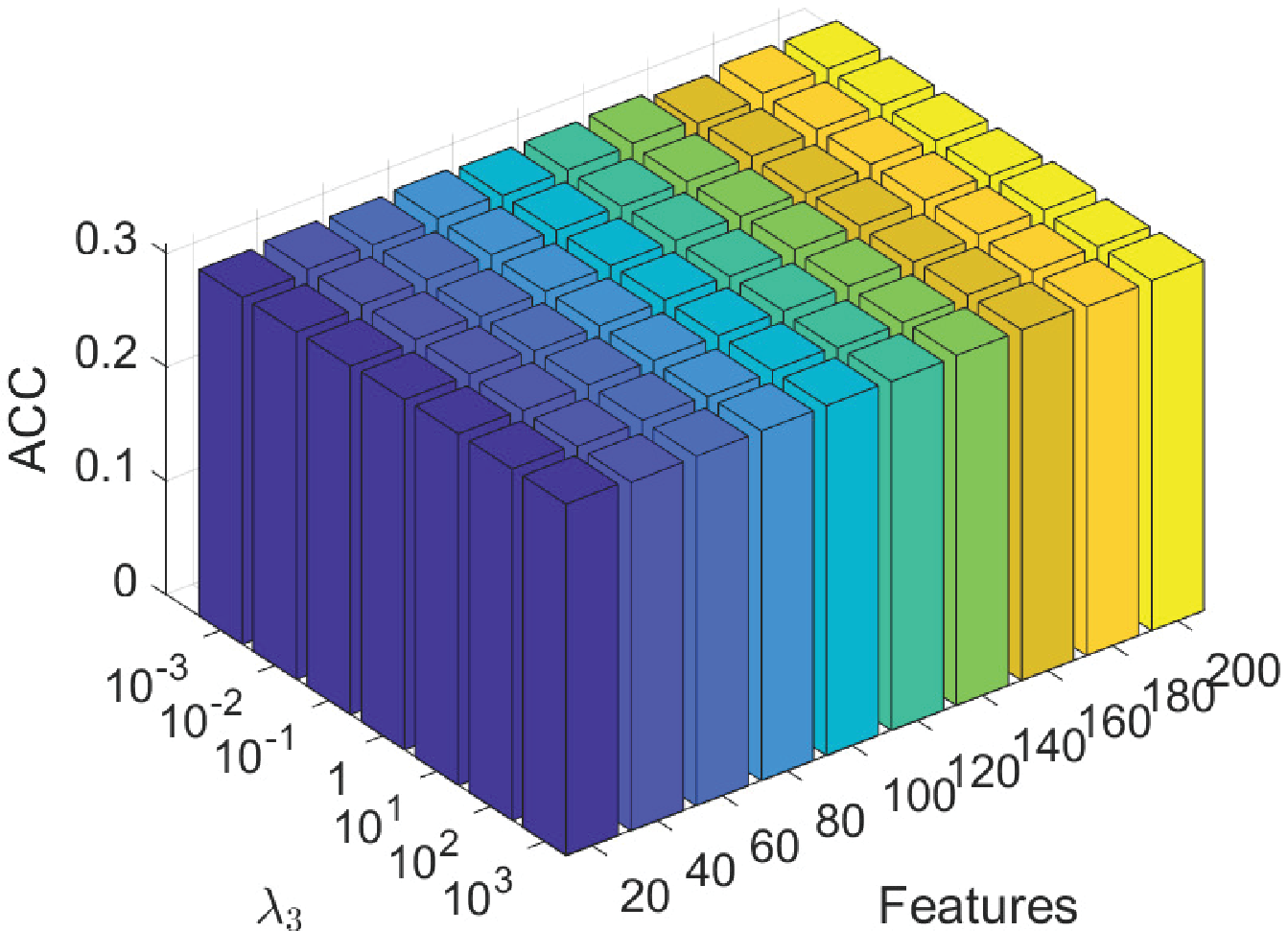}
\end{minipage}
}
\subfigure[Isolet]{
\begin{minipage}{0.31\linewidth}
\centering
  \includegraphics[width=\textwidth]{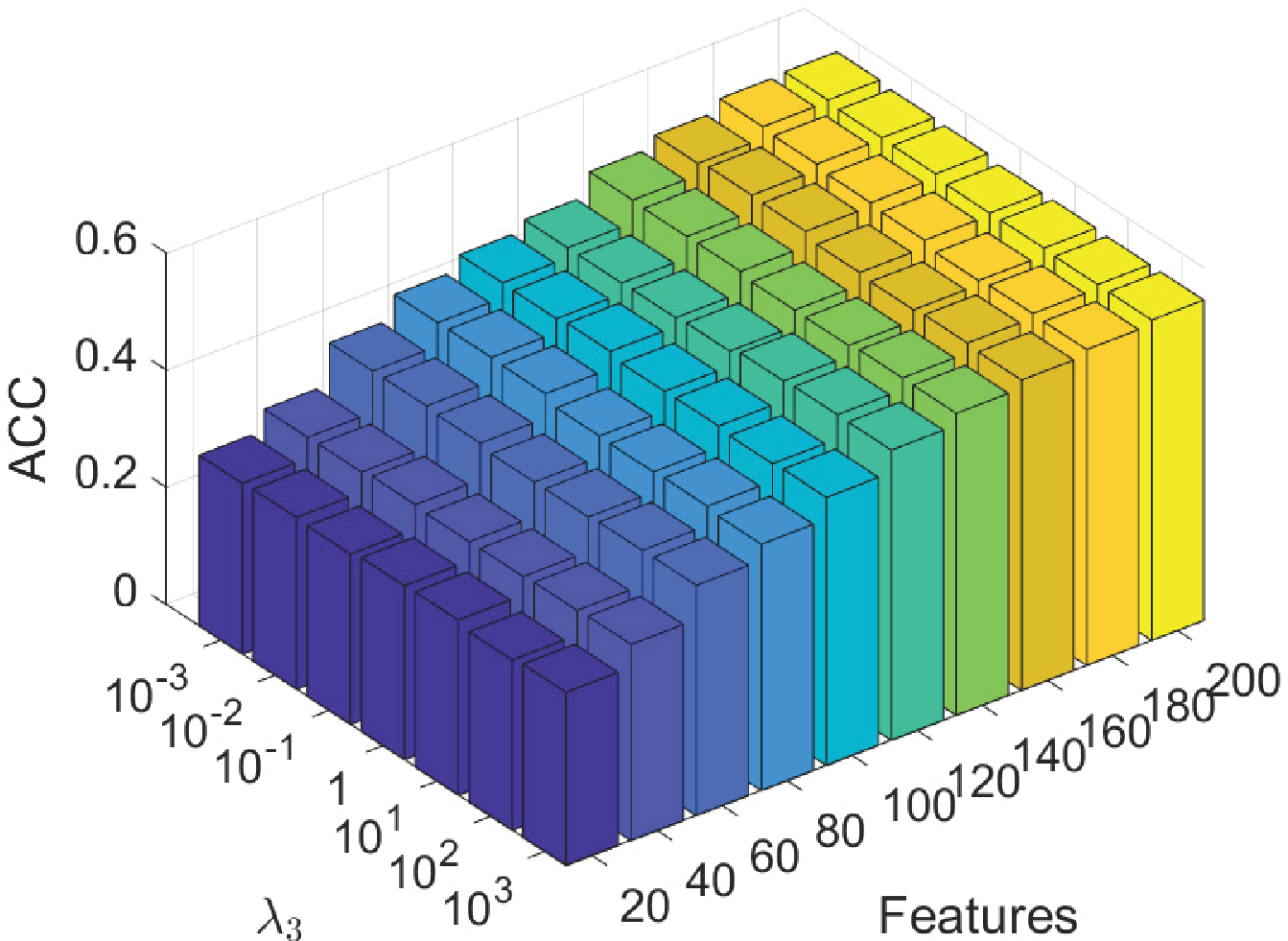}
\end{minipage}
}
\subfigure[Colon]{
\begin{minipage}{0.31\linewidth}
\centering
  \includegraphics[width=\textwidth]{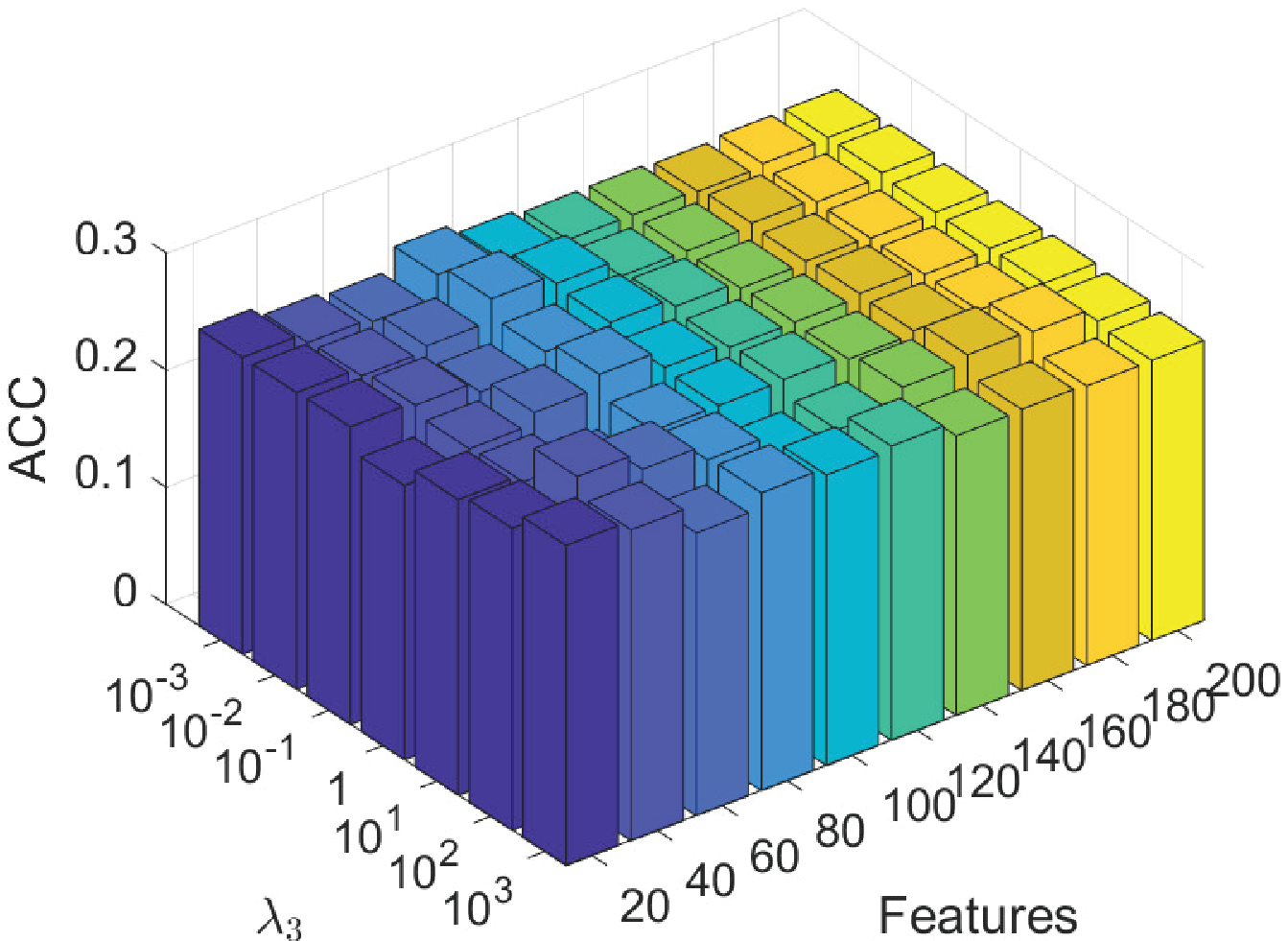}
\end{minipage}
}
\subfigure[warpPIE10P]{
\begin{minipage}{0.31\linewidth}
\centering
  \includegraphics[width=\textwidth]{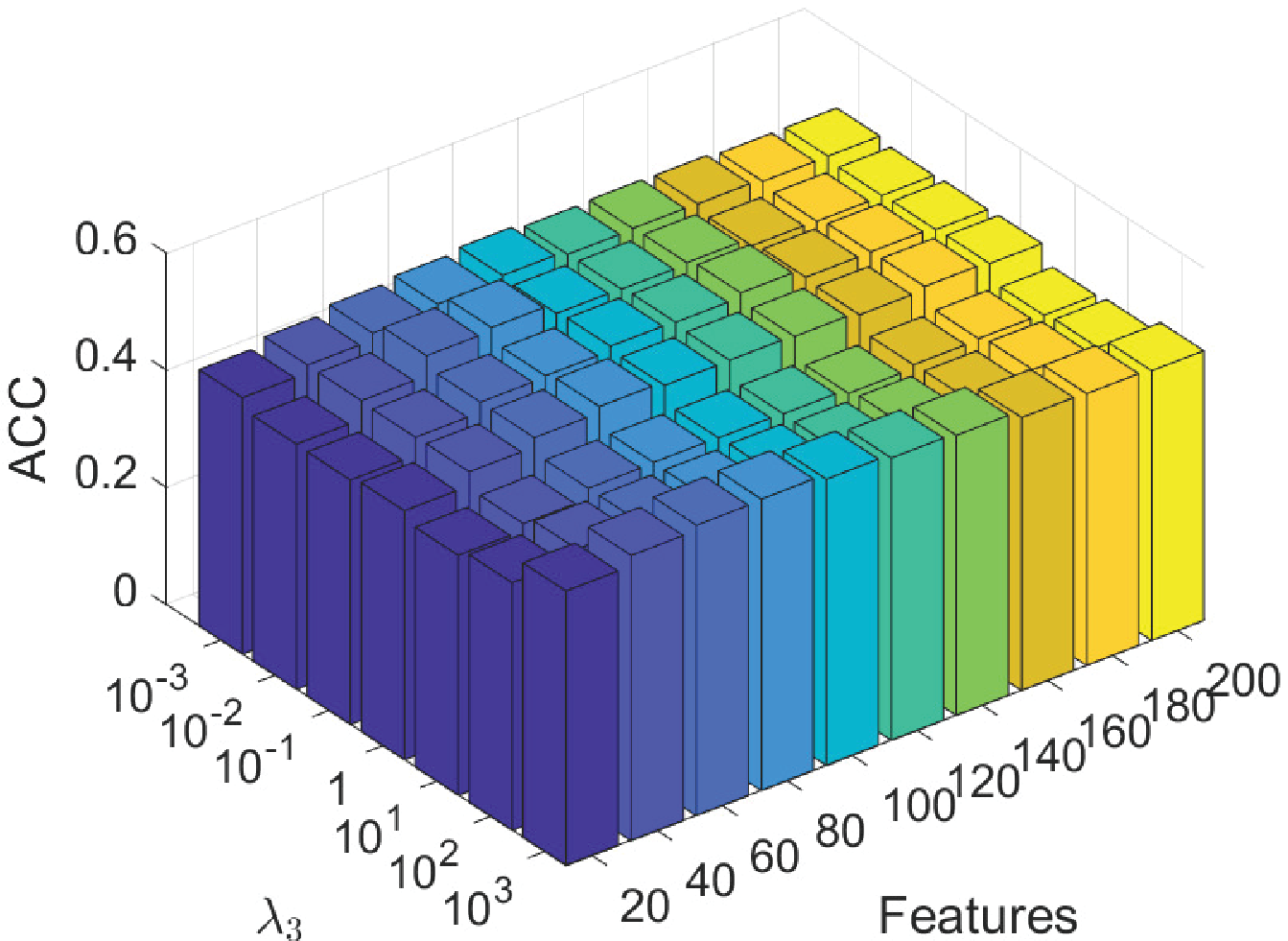}
\end{minipage}
}
\subfigure[GLIOMA]{
\begin{minipage}{0.31\linewidth}
\centering
  \includegraphics[width=\textwidth]{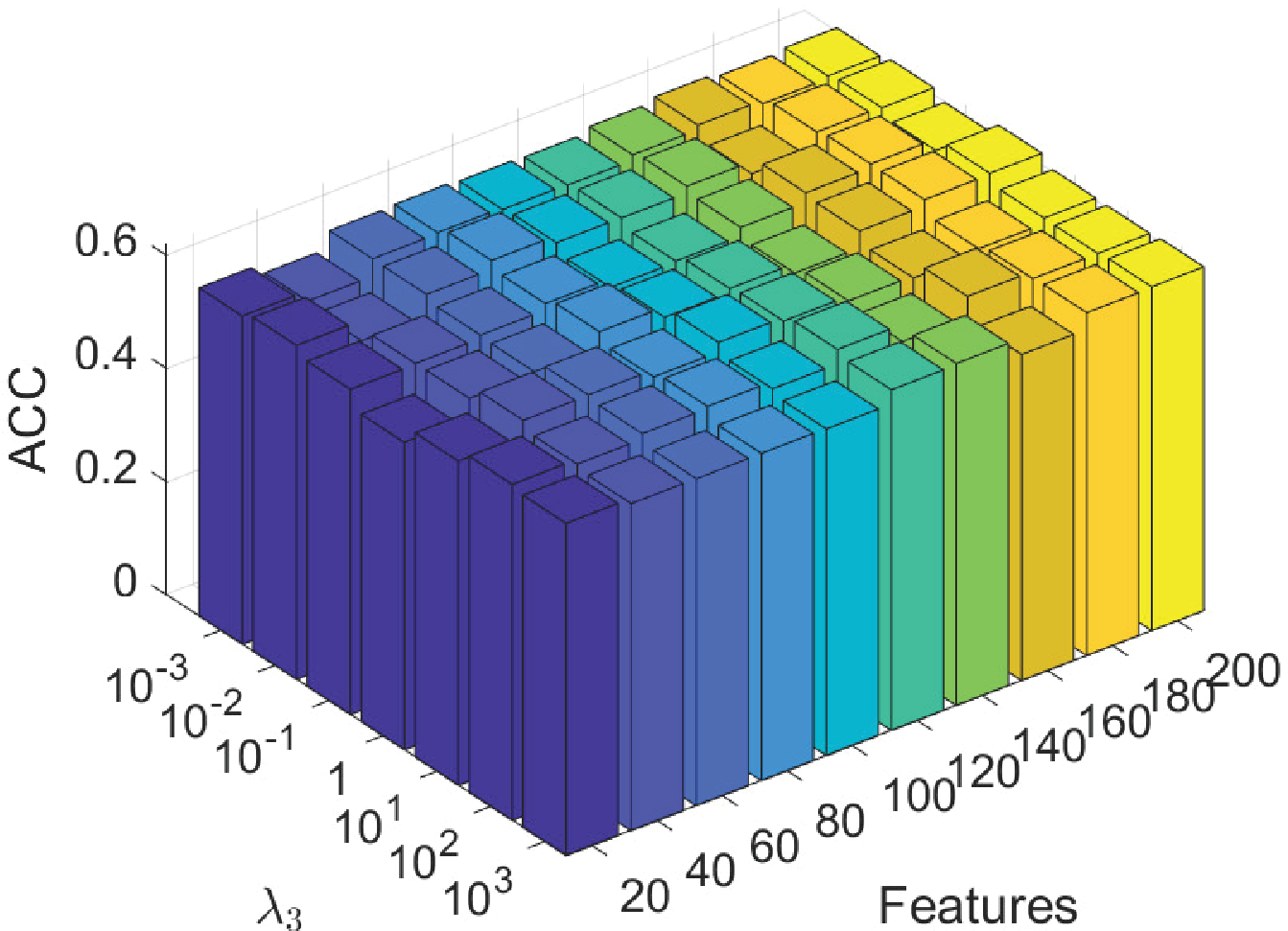}
\end{minipage}
}

\centering
\caption{Parameter sensitivity with respect to ${\lambda _3}$ in terms of ACC ($\alpha  = 1$, ${\lambda _1} = 1$, ${\lambda _{\rm{2}}} = 1$)}\label{fig::picture7}
\end{figure*}

It can be observed that the clustering accuracy is relatively stable with the change of $\alpha $ and ${\lambda _{\rm{3}}}$. While it fluctuates to a certain extent under different values of ${\lambda _1}$ and ${\lambda _{\rm{2}}}$, especially on datasets Isolet, Colon and warpPIE10P. That is to say, SPLR is insensitive to $\alpha $ and ${\lambda _{\rm{3}}}$, but sensitive to ${\lambda _1}$ and ${\lambda _{\rm{2}}}$. In general, as the values of ${\lambda _1}$ and ${\lambda _2}$ increase, the corresponding clustering accuracy goes up first and then goes down no matter how many features are selected. Since a smaller ${\lambda _1}$ means that more redundant features seem to be chosen and a larger ${\lambda _1}$ means that the redundancy between features is minimized as much as possible, the above variation trend verifies that redundant features will indeed degrade the performance of the algorithm but the performance will also decrease if the regularization term is too strong, which may cause overfitting. Likewise, a smaller ${\lambda _2}$ indicates weaker local geometric structure preservation, and a larger ${\lambda _2}$ indicates stronger local geometric structure preservation, which demonstrates the vital part local manifold structure preservation takes in feature selection. For instance, the proposed SPLR can achieve promising results on datasets USPS, Isolet and GLIOMA when ${\lambda _1}$ is set in the range of ${10^{ - 1}}$ to ${10^2}$. And it can achieve promising results on datasets Madelon, Isolet, Colon and warpPIE10P when ${\lambda _2}$ is set in the range of ${10^{ - 1}}$ to ${10^2}$. Therefore, there is no denying the fact that how to determine values of parameters in SPLR is of critical importance.
\section{Conclusions}
\label{sec:Conclusions}

In this paper, we propose a novel algorithm called unsupervised feature selection via self-paced learning and low-redundant regularization (SPLR). To avoid the negative influence of outliers and retain the global reconstruction information of data, self-paced learning is incorporated into the framework of subspace learning. Besides, with a view to redundancy reduction, a diversity term is designed, where inner product of features is deployed to judge the relevance. In addition, a regularization term is introduced to keep the local manifold structure of data unchanged under the assumption that if two samples are close in the original space, they should also be close when embedded into the subspace. What’s more, by virtue of the superiority of ${l_{2,{1 \mathord{\left/
 {\vphantom {1 2}} \right.
 \kern-\nulldelimiterspace} 2}}}$-norm over ${l_{2,1}}$-norm in sparsity, ${l_{2,{1 \mathord{\left/
 {\vphantom {1 2}} \right.
 \kern-\nulldelimiterspace} 2}}}$-norm is utilized to constrain the projection matrix. The optimization problem is addressed by an effective iterative algorithm. Experiments on nine benchmark datasets are conducted not only to confirm the excellent performance of SPLR in comparison with seven state-of-the-art algorithms, but also to validate the convergence of SPLR empirically.

In the future work, owing to the fact that random search, which considers that different parameters play a different role in the performance of the algorithm, has been testified to be more effective than gird search, a novel random search strategy for hyperparameter optimization can be used to obtain the optimal results to further verify the effectiveness of the algorithm. Moreover, since a hypergraph can capture the high order relationships between features rather than the simple pairwise ones, it is possible to transfer the original approach based on the traditional simple graph to the one based on the hypergraph for the sake of admirable performance.
%

\section*{Acknowledgment}
This work is supported by the National Natural Science Foundation of China (Nos. 61976182, 62076171, 61876157, 61976245), Key program for International S\&T Cooperation of Sichuan Province (2019YFH0097), and Sichuan Key R\&D project (2020YFG0035).
\bibliographystyle{elsarticle-num} 

\end{document}